\documentclass{article}

\PassOptionsToPackage{numbers, compress}{natbib}
\usepackage[preprint]{neurips_2023}

\usepackage[utf8]{inputenc} 
\usepackage[T1]{fontenc}    
\usepackage{hyperref}       
\usepackage{url}            
\usepackage{booktabs}       
\usepackage{amsfonts}       
\usepackage{nicefrac}       
\usepackage{microtype}      
\usepackage{xcolor}         

\title{Greedy Pruning with Group Lasso Provably Generalizes for Matrix Sensing}

\usepackage{amssymb,amsfonts,amsmath,amstext,mathrsfs}
\usepackage{amsthm}
\usepackage{graphics,latexsym,psfrag,wrapfig,comment,paralist}
\usepackage{xspace}
\usepackage{bm,multirow}
\usepackage{booktabs}
\usepackage{mathdots}
\usepackage{bbold}
\usepackage{centernot}
\usepackage{prettyref}
\usepackage{listings}
\usepackage{tikz}
\usetikzlibrary{shapes}
\usetikzlibrary{positioning, arrows, matrix, calc, patterns, fit}
\usepackage{caption}
\usepackage{array}
\usepackage{mdwmath}
\usepackage{multirow}
\usepackage{mdwtab}
\usepackage{eqparbox}
\usepackage{multicol}
\usepackage{amsfonts}
\usepackage{tikz}
\usepackage{multirow,bigstrut,threeparttable}
\usepackage{array}
\usepackage{bbm}
\usepackage{epstopdf}
\usepackage{mdwmath}
\usepackage{mdwtab}
\usepackage{eqparbox}
\usepackage{tikz}
\usepackage{latexsym}
\usepackage{amssymb}
\usepackage{bm}
\usepackage{graphicx}
\usepackage{subfigure}
\usepackage{mathrsfs}
\usepackage{psfrag}
\usepackage{setspace}
\usepackage{pgfplots}
\pgfplotsset{compat=1.18}
\usepackage{appendix}
\usepackage{thmtools} 
\usepackage{thm-restate}
\usepackage{algorithm}
\usepackage[noend]{algpseudocode}
\usepackage{caption}
\usepackage{mathtools}
\usepackage{cancel}
\usepackage{titletoc}

\newtheorem{theorem}{Theorem}
\newtheorem{lemma}[theorem]{Lemma}

\newtheorem{remark}[theorem]{Remark}

\newtheorem{definition}[theorem]{Definition}

\newtheorem*{insight*}{\textbf{Observation}}

\newtheorem*{proposition*}{\textbf{Proposition}}

\newtheorem*{lemmai*}{\textbf{Lemma (informal)}}

\newtheorem{theorem*}{\indent \em Theorem}[section]

\newcommand{\tr}{\textsf{Tr}}

\newcommand{\Lsm}{\ell_2^\beta}

\newtheorem{assumption}{Assumption}

\newcommand{\cL}{\mathcal{L}}
\newcommand{\cR}{\mathcal{R}}

\newcommand{\vect}{\textsf{vec}}

\newcommand{\poly}[1]{\text{poly} (#1)}

\newcommand{\prob}[1]{\mathbb{P} \left( #1 \right)}

\usepackage{cleveref}
\crefname{assumption}{Assumption}{Assumptions}
\crefname{definition}{Definition}{Defs.}
\crefname{lemma}{Lemma}{Lemmas}
\crefname{corollary}{Corollary}{Corollaries}

\usepackage{thmtools}

\author{
Nived Rajaraman
\thanks{Dept. of Electrical Engineering and Computer Sciences, UC Berkeley}\\
University of California, Berkeley\\
\texttt{nived.rajaraman@eecs.berkeley.edu}  
\And Devvrit \thanks{Dept. of Computer Science, UT Austin}\\
UT Austin\\
\texttt{devvrit.03@gmail.com} \And
Aryan Mokhtari\thanks{Dept. of Electrical and Computer Engineering, UT Austin}\\
UT Austin\\ \texttt{mokhtari@austin.utexas.edu}
\And Kannan Ramchandran $^*$\\
University of California, Berkeley\\ \texttt{kannanr@eecs.berkeley.edu}
\vspace{3mm}
}

\begin{document}

\maketitle

\begin{abstract}%
Pruning schemes have been widely used in practice to reduce the complexity of trained models with a massive number of parameters. In fact, several practical studies have shown that if a pruned model is fine-tuned with some gradient-based updates it generalizes well to new samples. Although the above pipeline, which we refer to as pruning + fine-tuning, has been extremely successful in lowering the complexity of trained models, there is very little known about the theory behind this success. In this paper, we address this issue by investigating the pruning + fine-tuning framework on the overparameterized matrix sensing problem with the ground truth  $U_\star \in \mathbb{R}^{d \times r}$ and the overparameterized model $U \in \mathbb{R}^{d \times k}$ with $k \gg r$. We study the approximate local minima of the mean square error, augmented with a smooth version of a group Lasso regularizer, $\sum_{i=1}^k \| U e_i \|_2$. In particular, we provably show that pruning all the columns below a certain explicit $\ell_2$-norm threshold results in a solution $U_{\text{prune}}$ which has the minimum number of columns $r$, yet close to the ground truth in training loss. Moreover, in the subsequent fine-tuning phase, gradient descent initialized at $U_{\text{prune}}$ converges at a linear rate to its limit. While our analysis provides insights into the role of regularization in pruning, we also show that running gradient descent in the absence of regularization results in models which {are not suitable for greedy pruning}, i.e., many columns could have their $\ell_2$ norm comparable to that of the maximum. To the best of our knowledge, our results provide the first rigorous insights on why greedy pruning + fine-tuning leads to smaller models which also generalize well.
\end{abstract}

\section{Introduction}

Training overparameterized models with a massive number of parameters has become the norm in almost all machine learning applications. While these massive models are successful in achieving low training error and in some cases good generalization performance, they are hard to store or communicate. Moreover, inference with such large models is computationally prohibitive. To address these issues, a large effort has gone into compressing these overparameterized models via different approaches, such as quantization schemes \citep{krishnamoorthi_18,guo_18}, unstructured \citep{han2015} and structured \citep{he_2018,li_compression,kusupati_20} pruning mechanisms, and distillation techniques using student-teacher models \citep{https://doi.org/10.48550/arxiv.1503.02531,https://doi.org/10.48550/arxiv.1802.05668}. Among these approaches, \textit{greedy pruning}, in which we greedily eliminate the parameters of the trained model based on some measure (e.g., the norms of the weight vectors associated with individual neurons) has received widespread attention \citep{han2015,deep_compression,li_16,lebedev_15,wen_16,klivans}. This is mostly due to the fact that several practical studies have illustrated that training an overparameterized model followed by greedy pruning and fine-tuning leads to better generalization performance, compared to an overparameterized model trained without pruning \citep{prune_generalization}. Furthermore, a phenomenon that has been observed by several practical studies, is that different forms of regularization during training, such as $\ell_0$ or $\ell_1$ regularization   \citep{l0_reg,liu_slimming,ye_pruning} or $\ell_2$ regularization including group Lasso \citep{SCARDAPANE201781} lead to models that are better suited for pruning, and leading to better generalization post fine-tuning. 

While the greedy pruning framework has shown impressive results, there is little to no theory backing why this pipeline works well in practice, nor understanding of the role of regularization in helping generate models which are suitable for greedy pruning. In this work, we address the following questions:

\begin{quote}
\vspace{-2mm}
  {\textit{Does the greedy pruning + fine-tuning pipeline provably lead to a simple model with good generalization guarantees? What is the role of regularization in pruning?}}
\end{quote}
\vspace{-2mm}
In this paper, we use the symmetric matrix sensing problem \citep{li_17} as a test-ground for the analysis of greedy pruning framework, a model very closely related to shallow neural networks with quadratic activation functions \citep{quadratic_networks,li_17}. In this setting, the underlying problem for the  population loss (infinite samples) is defined as $\| UU^T - U_\star U_\star^T \|_F^2$, where $U_\star \in \mathbb{R}^{d \times r}$ is an unknown ground-truth rank-$r$ matrix, with $r$ also being unknown, and   $U \in \mathbb{R}^{d \times k}$ is the overparameterized learning model  with $k \gg r$.
As we discuss in the Appendix~\ref{sec:quad}, the columns of $U$ can be thought of as the weight vectors associated with individual neurons in a $2$-layer shallow neural network with quadratic activation functions, a connection first observed in \citep{nn-quadratic-ms}. Thus, the data generating model has $r$ neurons, while the learner trains an overparameterized model with $k$ neurons.

While the statistical and computational complexity of the overparameterized matrix sensing problem has been studied extensively, we use it as a model for understanding the efficacy of greedy pruning. In particular, we aim to answer the following questions:  Does there exist a simple pruning criteria for which we can \textit{provably} show that the pruned model generalizes well after fine-tuning \textit{while having the minimal necessary number of parameters}? What is the role of regularization during training in promoting models which are compatible with greedy pruning?
Finally, what generalization guarantees can we establish for the pruned model post fine-tuning?

\textbf{Contributions.} Our main contribution is to show that the discussed pruning pipeline not only recovers the correct ground-truth $U_\star U_\star^T$ approximately, but also automatically adapts to the correct number of columns $r$. In particular, we show that training an overparameterized model on the empirical mean squared error with an added group Lasso based regularizer to promote column sparsity, followed by a simple norm-based pruning strategy results in a model $U_{\text{prune}}$ having exactly the minimum number of columns, $r$. At the same time, we show that $\| U_{\text{prune}} U_{\text{prune}}^T - U_\star U_\star^T \|_F^2$ is small, but non-zero. Hence, the pruned model can subsequently be fine-tuned using a small number of gradient steps, and in this regime, $\| U_t U_t^T - U_\star U_\star^T \|_F^2$ shows linear convergence to its limit. Moreover, the pruned model can be shown to admit finite sample generalization bounds which are also statistically optimal for a range of parameters. In particular, we show that to obtain a model $ U_{\text{out}}$ that has exactly $r$ columns and its population error is at most  $\| U_{\text{out}} U_{\text{out}}^T - U_\star U_\star^T \|_F \leq \varepsilon$, our framework requires $O(dk^2r^5 +\frac{rd}{\varepsilon^2})$ samples, which is statistically optimal for sufficiently small $\varepsilon$. We should also add that our framework does not require any computationally prohibitive pre- or post-processing (such as SVD decomposition) for achieving this result.

While there are several works \citep{wainwright2015,li_17,constantine21} establishing that gradient descent in the ``exactly-parameterized'' setting requires $O(rd/\varepsilon^2)$ samples to achieve a generalization error of $\varepsilon$, and converges linearly to this limit, the picture is different in the overparameterized setting. In  \cite{constantine21}, the authors showed that in the overparametrized setting, vanilla gradient descent requires $O(kd/\varepsilon^2)$ samples to achieve a generalization error of $\| U_{\text{gd}} U_{\text{gd}}^T - U_\star U_\star^T \|_F \le \varepsilon$, degrading with the overparameterization of the model. Moreover the resulting solution does not have the correct column sparsity. In order to obtain a model which can be stored concisely, $U_{\text{gd}}$ has to be post-processed by computing its SVD, which is computationally expensive in the high dimensional regime.

As our second continuation, we show that use of explicit regularization to promote column sparsity while training is important to learn models suitable for greedy pruning. Specifically, we show that while implicit regularization \citep{li_17,ye_2021} suffices to learn models with the correct rank, these approaches learn solutions with a large number of columns having $\ell_2$-norms comparable to that of the maximum, even if $r=1$. Hence, it is unclear how to sparsify such models based on the norms of their columns.

\section{Setup and Algorithmic Framework}

Given $n$ observation matrices $\{A_i\}_{i=1}^n$, in the matrix sensing framework, the learner is provided measurements $y_i = \langle A_i, U_\star U_\star^T \rangle + \varepsilon_i$ where $\langle \cdot,\cdot \rangle$ indicates the trace inner product and $\varepsilon_i$ is measurement noise assumed to be distributed i.i.d. $\sim \mathcal{N} (0,\sigma^2)$\footnote{All proofs in the paper go through as long as the noise is i.i.d. sub-Gaussian with variance proxy $\sigma^2$. We study the Gaussian case for simplicity of exposition.}. Here $U_\star \in \mathbb{R}^{d \times r}$ is the unknown parameter, and the rank $r \le d$ is unknown. The goal of the matrix sensing problem is to learn a candidate matrix $X$ such that $X \approx U_\star U_\star^T$. For computational reasons, it is common to factorize $X$ as $UU^T$ for $U \in \mathbb{R}^{d \times k}$. In this paper, we study the factored model in the overparameterized setting, where $k \gg r$. The empirical mean squared error is,
\begin{equation} \label{eq:Lemp}
    \cL_{\text{emp}} (U) = \frac{1}{n} \sum_{i=1}^n \left(\langle A_i, UU^T \rangle - y_i \right)^2.
\end{equation}
For the case that $A_i$'s are sampled entry-wise i.i.d.  $\mathcal{N}(0,1/d)$ and as $n \rightarrow \infty$, up to additive constants which we ignore, the population mean square error can be written down as,
\begin{equation} \label{eq:Lpop}
    \cL_{\text{pop}} (U) = \|UU^T - U_\star U_\star^T\|_F^2.
\end{equation}
There is an extensive literature on how to efficiently learn the right product $U_\star U_\star^T$ in both finite sample and population settings \citep{seewong,park2017non}. In particular, there are several works on the efficiency of  gradient-based methods with or without regularization for solving this specific problem \citep{li_17,constantine21}. While these approaches guarantee learning the correct product $UU^T \approx U_\star U_\star^T$, in the overparameterized setting the obtained solutions are not column sparse and storing these models requires $\widetilde{\Theta}(kd)$ bits of memory (ignoring precision). As a result, in order to obtain a compressed solution $U_{\text{out}}\in \mathbb{R}^{d \times r}$ with the correct number of columns, one has to post-process $U$ and do a singular value decomposition (SVD), an operation which is costly and impractical in high-dimensional settings.

The goal of this paper is to overcome this issue and come up with an efficient approach to recover a solution $U_{\text{out}}$ which generalizes well, in that $\cL_{\text{pop}} (U)$ is small, while at the same time having only a few non-zero columns, i.e. is sparse. Specifically, we show that via some proper regularization, it is possible to obtain a model that approximately learns the right product $U_\star U_\star^T$, while having only a few significant columns. As a result, many of its columns can be eliminated by a simple $\ell_2$ norm-based greedy pruning scheme, without significantly impacting the training loss. In fact, post pruning we end up with a model $ U_{\text{prune}} \in \mathbb{R}^{d \times r}$ that has the correct dimensionality and its outer product $ U_{\text{prune}} U_{\text{prune}}^T$ is close to the true product $U_\star U_\star^T$. When the resulting ``exactly parameterized'' model is fine-tuned, the generalization loss $\cL_{\text{pop}} (U)$ can be shown to converge to $0$ at a linear rate.

To formally describe our procedure, we first introduce the regularization scheme that we study in this paper, and then we present the greedy pruning scheme and the fine-tuning procedure.

\paragraph{Regularization.} The use of regularization for matrix sensing (and matrix factorization) to encourage a  low rank solution  \citep{cai_10,meka_09,keshavan_09,ma_09,recht_09,recht_2010,toh_2010}, or to control the norm of the model for stability reasons \citep{ge_16, bhojanapalli_16}, or to improve the landscape of the loss by eliminating spurious local minima \citep{rong} has been well-studied in the literature. While these approaches implicitly or explicitly regularize for the rank of the learned matrix, the solution learned as a result is often not column sparse. Indeed, note that a matrix can be low rank and dense at the same time, if many columns are linear combinations of the others. We propose studying the following regularized matrix sensing problem with a group Lasso based regularizer \citep{grouplasso,SCARDAPANE201781}. In the population case, the loss is defined as
\begin{equation} \label{eq:regloss}
    \cL_{\text{pop}} (U) + \lambda \cR (U),\qquad  \text{where}\quad  \cR (U) = \sum_{i=1}^{k} \| U e_i \|_2.
\end{equation} 
Note that $\lambda>0$ is a properly selected regularization parameter. Imposing $\cR$ as a penalty on the layer weights of a neural network is a special case of a widely used approach commonly known as \textit{Structured Sparsity Learning} (SSL) \citep{wen_16}. The regularizer promotes sparsity across ``groups'' as discussed by \cite{grouplasso}. Here, the groups correspond to the columns of $U$. 

As we prove in \Cref{sec:pop}, the solution obtained by minimizing a smooth version of the regularized loss in \eqref{eq:regloss}, denoted by $U_{\text{out}}$, is approximately column-sparse, i.e., $k-r$ of its columns have small $\ell_2$ norm. As a result, it is suitable for the simple greedy pruning scheme which we also introduce in \Cref{alg:main}. Interestingly, in \Cref{sec:negative}, we will first show a negative result - the model obtained by minimizing the \textit{unregularized} loss, $\cL_{\text{pop}} (U)$ (\cref{eq:Lpop}) using gradient descent updates could fail to learn a solution which is suitable for greedy pruning. This shows the importance of adding some form of regularization during the training phase of the overparameterized model.

\paragraph{Greedy Pruning.} The greedy pruning approach posits training a model (possibly a large neural network) on the empirical loss, followed by pruning the resulting trained network greedily based on some criteria. The resulting model is often fine-tuned via a few gradient descent iterations before outputting. In the literature, various criteria have been proposed for greedy pruning. Magnitude-based approaches prune away the individual weights/neurons based on some measure of their size such as $\ell_1/\ell_2$ norm of the associated vectors \citep{liu_slimming,li_16}.

In this work, we also focus on the idea of greedy pruning and study a mechanism to prune the solution obtained by minimizing the regularized empirical loss. Specifically, once an approximate second-order stationary point of the loss in \Cref{sec:pop}, we only keep its columns whose $\ell_2$ norm are above a threshold (specified in \Cref{alg:main}) and the remaining columns with smaller norm are removed. We further show that post pruning, the obtained model  $U_{\text{prune}}$ continues to have a small empirical loss, i.e.,  small $ \cL_{\text{emp}} (U_{\text{prune}})$, while having exactly $r$ columns.

\paragraph{Post pruning fine-tuning.} As mentioned above, it is common to fine-tune the smaller pruned model with a few gradient updates before the evaluation process \citep{liu2018rethinking}. We show that the pruned model $U_{\text{prune}}$ has the correct rank and is reasonably close to the ground-truth model $U_{\star}$ in terms of its population loss. By running a few gradient updates on the mean square error, it can be ensured that $UU^T$ converges to $U_{\star} U_\star^T$ at a linear rate. \Cref{alg:main} summarizes the framework that we study.

\section{Implicit regularization does not lead to greedy pruning-friendly models}  \label{sec:negative}

In various overparameterized learning problems, it has been shown that first-order methods, starting from a small initialization, implicitly biases the model toward ``simple'' solutions, resulting in models that generalize well \citep{neyshabur_14,https://doi.org/10.48550/arxiv.1710.10345}, a phenomenon known as implicit regularization. In particular, for matrix sensing, \citep{gunasekar2017implicit,li_17,ye_2021} show that in the absence of any regularization, running gradient descent on the population loss \textit{starting from a small initialization} biases $UU^T$ to low rank solutions, and learns the correct outer product $UU^T \approx U_\star U_\star^T$. However, as discussed earlier, low-rank solutions are not necessarily column sparse, nor is it clear how to sparsify them without computing an SVD. It is unclear whether implicit regularization suffices to learn models that are amenable for greedy pruning. 

\begin{figure}[t!]
\begin{subfigure}
\centering
\vspace{-2mm}
    \centering
    \includegraphics[width=0.52\textwidth]{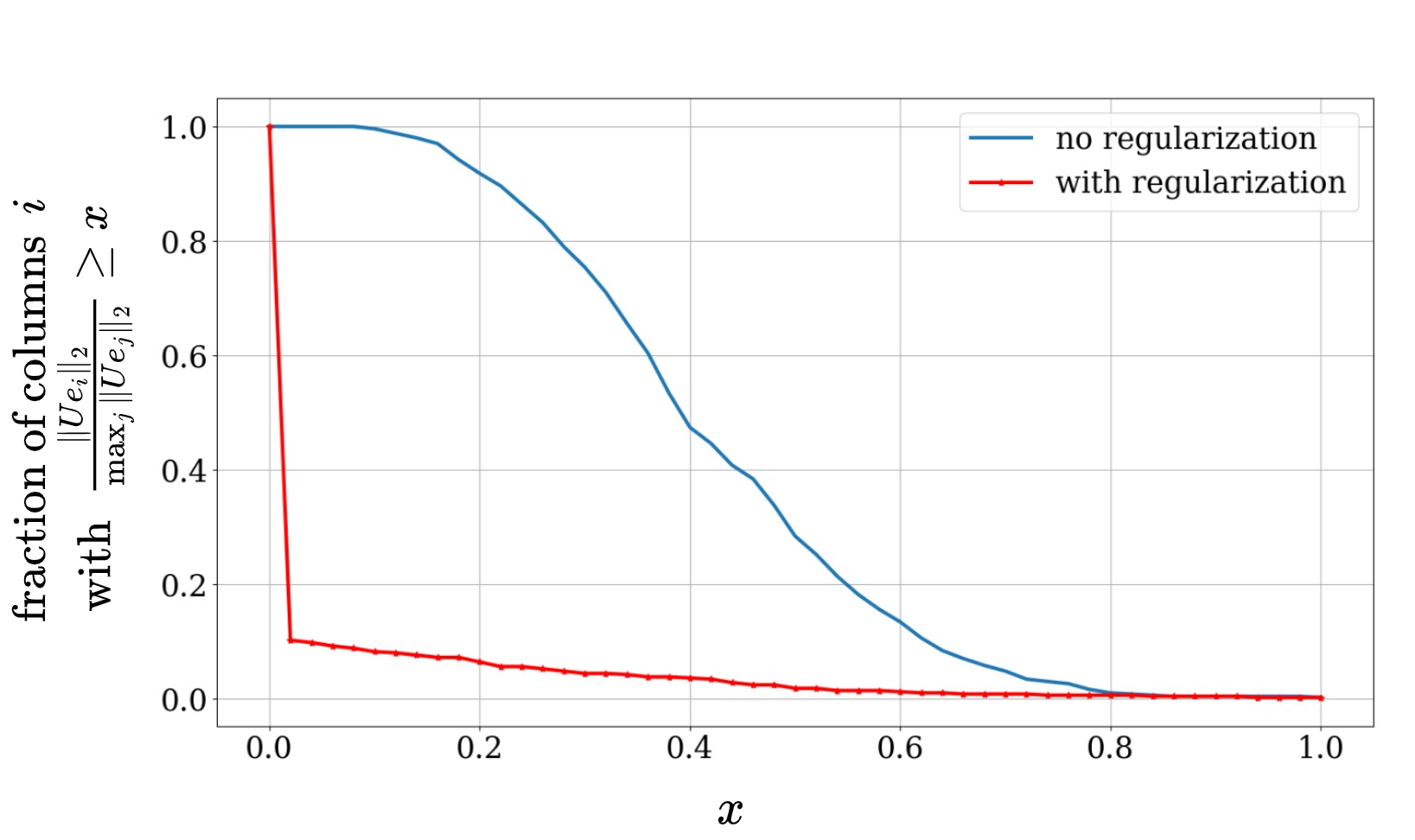}
\end{subfigure}
\begin{subfigure}
    \centering
    \!\!\!\!\!\!\!\!
    \includegraphics[width=0.54\textwidth]{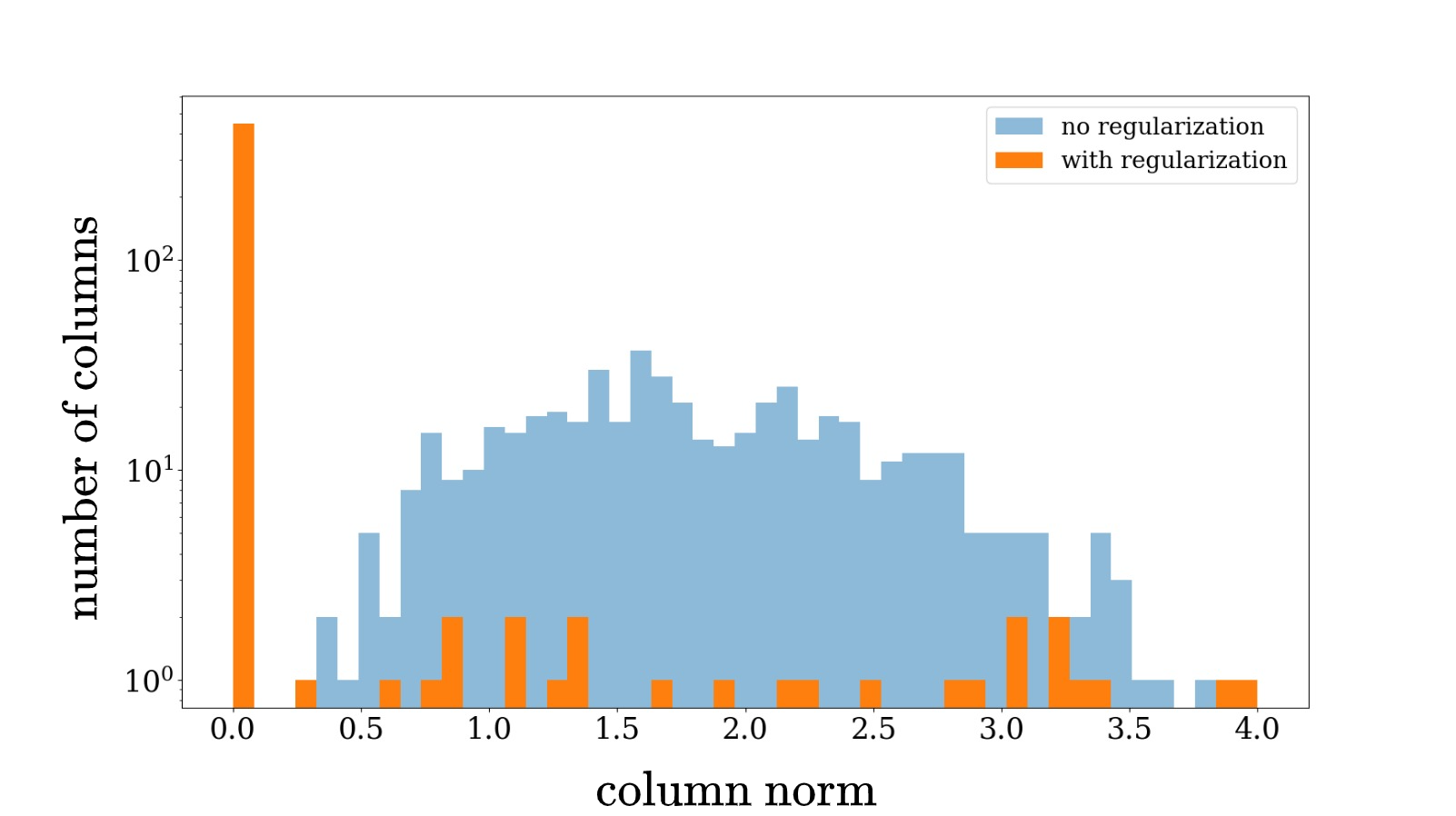}
\end{subfigure}
    \caption{\small We run gradient updates to minimize the loss in \cref{eq:Lpop} until convergence. The model is highly overparameterized with $d=k=500$ and $r=4$. 
    In the left figure, the vertical axis represents the fraction of columns having their norm at least $x$ times the largest norm across columns, with $x \in [0,1]$ on the horizontal axis. In the right figure, for the same experiment we plot a histogram of column norms. Using regularization leads to solutions with most of the columns having small $\ell_2$ norm, and only a few significant columns. Without regularization, a large number of columns have their $\ell_2$ norm comparable to the largest column norm. Thus, training with explicit regularization leads to models that are more suitable for pruning.}
    \label{fig:1}
\end{figure}

In this section, we address this question and show that minimizing the unregularized population loss $\cL_{\text{pop}}$ leads to models which are not suitable for greedy pruning, i.e., have many columns with large $\ell_2$ norm. Specifically, we show that by running gradient flow from a small random initialization, even if the ground truth $U_\star$ is just a single column and $r=1$, the learnt solution $U$ has a large number of columns that are ``active'', i.e. having $\ell_2$-norm comparable to that of the column with maximum norm. Thus, in the absence of the knowledge of $r$, it is unclear how to determine it from just observing the  columns $\ell_2$ norm. We thus claim that such trained models are not compatible with greedy pruning.
 In the following theorem, we formally state our result.

\begin{theorem}
\label{theorem:lb-implicit}
Consider the population loss in \cref{eq:Lpop}, for  $r=1$ and $k\gg 1$ and suppose $\| U_\star \|_{\text{op}} = 1$. Further, suppose the entries of the initial model $U_0$ are i.i.d. samples from $\mathcal{N} (0,\alpha^2)$, where $\alpha \le c_{\ref{theorem:lb-implicit}} /k^3 d \log(kd)$ for some constant $c_{\ref{theorem:lb-implicit}} > 0$. For another absolute constant $c_{\ref{theorem:lb-implicit}}' > 0$, as $t \to \infty$, the iterates of gradient flow  with probability $\ge 1 - O(1/k^{c_{\ref{theorem:lb-implicit}}'})$ converge to a model $U_{gd}$ where $\widetilde{\Omega} (k^{c_{\ref{theorem:lb-implicit}}'})$ active columns satisfy
\begin{equation} \label{eq:123121112}
    \frac{\| U_{gd} e_i \|_2}{\max_{j \in [k]} \| U_{gd} e_j \|_2} \ge 0.99.
\end{equation}
\end{theorem}

\Cref{theorem:lb-implicit} claims that the iterates of gradient descent converge to a solution for which a large number ($\Omega (k^{c_{\ref{theorem:lb-implicit}}'})$) of columns of $U_{\text{gd}}$ have their $\ell_2$ norm at least $0.99$ times the maximum. An inspection of the proof shows that the constant $0.99$ is arbitrary and can be extended to any constant bounded away from $1$. Therefore, the models learnt by vanilla gradient descent cannot be reduced to the right sparsity level by simply pruning away columns with $\ell_2$ norm below any fixed constant threshold.

We provide a general guarantee in \Cref{theorem:lb-implicit-restate} (\Cref{app:1}). It is worth noting that from \citep{li_17} that the learned model $U_{gd} \in \mathbb{R}^{d\times k}$ is guaranteed to achieve a small generalization error, i.e., $U_{gd}U_{gd}^T\approx U_\star U_\star^T$. Hence, the algorithm does learn the correct ground-truth product - what we show is that the learnt model is not ``pruning-friendly'' - it cannot be sparsified to the correct level by on any simple norm based pruning strategies. This is also observed in \Cref{fig:1}: the model trained with implicit regularization has many significant columns and it is unclear how many to prune. We provide a proof sketch below, and refer the reader to \Cref{app:1} for more details.

\paragraph{Proof sketch of \Cref{theorem:lb-implicit}.} Without loss of generality we assume in the proof that $\|U_\star\|_2 =1$. Defining $U(t)$ as the iterate at time $t$, gradient flow on $\cL_{\text{pop}}$ follows the below dynamic,
\begin{align} \label{eq:gradflow-main}
    \frac{dU}{dt} = -\nabla \cL_{\text{pop}} (U) = - (UU^T - U_\star U_\star^T) U,
\end{align}
where for ease of notation we drop the explicit dependence on $t$ in $U(t)$. In our proof, we exactly characterize the limiting point of gradient flow as a function of the initialization in a certain sense, which may be of independent interest. In particular, up to a small multiplicative error, we show,
\begin{align}
    \label{eq:similar-dist-main}
    \forall i \in [k],\ \| U_{\text{gd}} e_i \|_2 \approx \frac{|\langle U_\star, U(0) e_i \rangle|}{\| U_\star^T U(0) \|_2}.
\end{align}
Also, with Gaussian initialization, $\langle U_\star, U(0) e_i \rangle \overset{\text{i.i.d.}}{\sim} \mathcal{N} (0,\alpha^2)$ across different values of $i$. In particular, we show that for some constant $c_{\ref{theorem:lb-implicit}}' > 0$,\ $\widetilde{\Omega}(k^{c_{\ref{theorem:lb-implicit}}'})$ columns will have correlations comparable to the maximum, with $|\langle U_\star, U(0) e_i \rangle| \ge 0.99 \max_{j \in [k]} |\langle U_\star, U(0) e_j \rangle|$. For any of these columns, $i$,
\begin{align}
    \| U_{\text{gd}} e_i \|_2 \approx \frac{|\langle U_\star, U(0) e_i \rangle|}{\| U_\star^T U(0) \|_2} \ge 0.99 \max_{j \in [k]} \frac{|\langle U_\star, U(0) e_j \rangle|}{\| U_\star^T U(0) \|_2} \approx 0.99 \max_{j \in [k]} \| U_{\text{gd}} e_j \|_2,
\end{align}
which completes the proof sketch of \Cref{theorem:lb-implicit}.

\section{Explicit regularization gives pruning-friendly models: population analysis} \label{sec:pop}

In this section, we study the properties of the squared error augmented with group Lasso regularization in the population (infinite sample) setting. We show that second-order stationary points of the regularized loss are suitable for greedy pruning, while at the same time achieving a small but non-zero generalization error. Note that the $\ell_2$ norm is a non-smooth function at the origin, and therefore, the overall regularized loss is non-smooth and non-convex. While there are several notions of approximate stationary points for non-differentiable and non-convex functions, for technical convenience, we replace $\cR$ by a smooth proxy. In particular, for a smoothing parameter $\beta > 0$, define a smooth version of the $\ell_2$ norm, and the corresponding smooth regularizer $\cR_\beta$ as,
\begin{align} \label{eq:reg-smooth}
    \cR_\beta (U) = \sum_{i=1}^{k} \Lsm (Ue_i), \quad \text{ where} \quad \Lsm (v) = \frac{\| v \|_2^2}{\sqrt{\| v \|_2^2 + \beta}}.
\end{align}
Note that the smaller the value of $\beta$ is, the closer is $\ell_2^\beta (v)$ to $\| v \|_2$. Considering this definition, the overall regularized loss we study in the population setting is
\begin{align}
    &f_{\text{pop}} (U) = \cL_{\text{pop}} (U) + \lambda \cR_\beta (U), \label{eq:fregpop}
\end{align}
where $\cL_{\text{pop}} = \| UU^T - U_\star U_\star^T \|_F^2$ as defined in \cref{eq:Lpop}. The above optimization problem is nonconvex due to the structure of $\cL_{\text{pop}} (U)$, and finding its global minimizer can be computationally prohibitive. Fortunately, for our theoretical results, we do not require achieving global optimality and we only require an approximate second-order stationary point of the loss in \cref{eq:fregpop}, which is defined below. 

\begin{definition} \label{def:init-pop}
We say that $U$ is an $(\epsilon, \gamma)$-approximate second-order stationary point of $f$ if,
\begin{enumerate}
\vspace{-1mm}
    \item The gradient norm is bounded above by $\epsilon$, i.e.,  $\| \nabla f(U) \|_2 \le \epsilon$.
    \vspace{-1mm}
    \item The eigenvalues of the Hessian are larger than $-\gamma$, i.e., $\lambda_{\min}(\nabla^2 f(U)) \ge - \gamma$.
\end{enumerate}
\end{definition}

\noindent The full algorithmic procedure that we analyze in this section is summarized in \Cref{alg:main}. 
Once we find an $(\epsilon, \gamma)$-approximate second-order stationary point (SOSP) of \cref{eq:fregpop} for some proper choices of $\epsilon $ and $\gamma$, we apply greedy pruning on the obtained model $U$ and by eliminating all of its columns with $\ell_2$ norm below a specific threshold. As a result, the pruned model $U_{\text{prune}}$  has fewer columns than the trained model $U$. 
In fact, in our following theoretical results, we show that if the parameters are properly selected, $U_{\text{prune}}$ has exactly $r$ columns, which is the same as $U_\star$. Finally, we fine-tune the pruned solution by running gradient descent on the unregularized loss $\cL_{\text{pop}}$ (resp. $\cL_{\text{emp}}$). Next, we state the properties of the pruned model generated by \Cref{alg:main} in the infinite sample setting. 

\begin{algorithm}[t!]
\small
\hspace*{\algorithmicindent} \textbf{Inputs}: Measurements $\{ (A_i,y_i) \text{ where } y_i 
 = \langle A_i, U_\star U_\star^T \rangle + \varepsilon_i \}_{i=1}^n$ (in the population setting $n=\infty$); \\
\hspace*{\algorithmicindent} \textbf{Initialization}: Set parameters:  $\lambda, \beta,\epsilon,\gamma$ and $m_{\text{fine-tune}}$.

\hspace*{\algorithmicindent} \textbf{Greedy pruning phase:}
\begin{algorithmic}[1] 

\State Find an $(\epsilon,\gamma)$-approximate SOSP of $f_{\text{emp}}$ (resp. $f_{\text{pop}}$), $U$, satisfying $\| U \|_{\text{op}} \le 3$.

\State Let $S = \{ i \in [k] : \| U e_i \|_2 \le 2 \sqrt{\beta}$ denote the set of columns with small $\ell_2$ norm.

\noindent Create a new matrix $U_{\text{prune}}$ which only preserves the columns of $U$ in $[k] \setminus S$, deleting the columns in $S$.
\end{algorithmic}
\hspace*{\algorithmicindent} \textbf{Fine-tuning phase:}
\begin{algorithmic}[1]

\State Run $m_{\text{fine-tune}}$ iterations of gradient descent on $\mathcal{L}_{\text{pop}} (U)$ (resp. $\mathcal{L}_{\text{emp}} (U)$) initialized at $U_{\text{prune}}$ to get $U_{\text{out}}$.

\State \Return{$U_{\text{out}}$.}

\end{algorithmic}
\caption{Greedy pruning based on group-Lasso regularization}

\label{alg:main}
\end{algorithm}

\begin{theorem} \label{theorem:main-population}
Consider the population loss with regularization in \cref{eq:fregpop}, where $U_\star$ has rank $r$ and its smallest singular value is denoted by $\sigma_r^\star$.
Moreover, consider $U_{\text{prune}}$ as the output of the pruning phase in \Cref{alg:main} with  parameters $\beta, \lambda, \epsilon, \gamma$ satisfying the conditions\footnote{This style of result exists for LASSO as well (see \cite[Theorem 1]{loh2013regularized}), where the optimal choice of the regularization parameter, $\lambda^\star$, depends on the true sparsity $r$, but a general guarantee can be established as well, which degrades as $\lambda$ deviates from $\lambda^\star$. In practice $\lambda$ is chosen using cross-validation. For simplicity of presentation, we state the result when $r$ and $\sigma_r^\star$ are known up to constants.},
\begin{align} \label{eq:parameter}
    \beta = c_\beta \frac{(\sigma_r^\star)^2}{r},
    \qquad 
    \lambda = c_\lambda \frac{(\sigma_r^\star)^3}{\sqrt{kr}}, \qquad
    \gamma \le c_\gamma \frac{(\sigma^\star_r)^3}{\sqrt{k} r^{5/2}}, 
    \qquad
    \epsilon \le c_\epsilon \frac{(\sigma_r^\star)^{7/2}}{\sqrt{k} r^{5/2}},
\end{align}
for some absolute constants $c_\beta,c_\lambda,c_\epsilon, c_\gamma > 0$. Then, we have
\begin{enumerate}
\vspace{-1mm}
    \item $U_{\text{prune}}$ has exactly $r$ columns.
    \vspace{-1mm}
    \item $\| U_{\text{prune}} U_{\text{prune}}^T - U_\star U_\star^T \|_F \le \frac{1}{2} (\sigma_r^\star)^2$.
\end{enumerate}
\end{theorem}

This result relies on showing that all bounded SOSPs of the regularized loss in \cref{eq:fregpop} are suitable for greedy pruning: removing the columns of $U$ below a certain $\ell_2$-norm threshold results in a solution $U_{\text{prune}}$ having exactly $r$ columns, while at the same time having a small generalization error. Hence, it can serve as a proper warm-start for the fine-tuning phase.

\paragraph{Proof sketch of \Cref{theorem:main-population}.} The key idea is to identify that if we have a matrix $U$ such that $UU^T = U_\star U_\star^T$, and the columns of $U$ are orthogonal to one another, then $U$ has exactly $r $ non-zero columns, where $r$ is the rank of $U_{\star}$. This statement can be shown to hold even when $UU^T \approx U_\star U_\star^T$ and the columns of $U$ are only approximately orthogonal.
The main observation we prove is that a bounded $(\epsilon,\gamma)$-approximate SOSPs of \cref{eq:fregpop} denoted by $U$ satisfies the following condition:
\begin{equation} \label{eq:apxorth}
    \forall i,j : \| U e_i \|_2, \| U e_j \|_2 \ge 2 \sqrt{\beta},\quad \frac{\langle U e_i, U e_j \rangle}{\| U e_i \|_2 \| U e_j \|_2} \approx 0.
\end{equation}
In other words, all the large columns of $U$ have their pairwise angle approximately $90^\circ$. Thus, by pruning away the columns of $U$ that have an $\ell_2$ norm less than $2 \sqrt{\beta}$, the remaining columns of $U$, i.e., the columns of $U_{\text{prune}}$, are now approximately at $90^\circ$ angles to one another. If $\beta$ is chosen to be sufficiently small, after deleting the low-norm columns, the approximation $U_{\text{prune}} U_{\text{prune}}^T \approx UU^T$ holds. By the second order stationarity of $U$, we also have that $UU^T \approx U_\star U_\star^T$. Together, this implies that $U_{\text{prune}} U_{\text{prune}}^T \approx U_\star U_\star^T$ and $U_{\text{prune}} U_{\text{prune}}^T$ is close to a rank $r$ matrix. Since $U_{\text{prune}}$ has orthogonal columns, this also means that it has exactly $r$ columns. Finally, to establish a bound on the approximation error, we simply use the triangle inequality that $\| U_{\text{prune}} U_{\text{prune}}^T - U_\star U_\star^T \|_F \le \| UU^T - U_\star U_\star^T \|_F + \| UU^T - U_{\text{prune}} U_{\text{prune}}^T \|_F$. The former is small by virtue of the fact that $U$ is an approximate second order stationary point of $f_{\text{pop}} \approx \mathcal{L}_{\text{pop}}$ when $\lambda$ is small; the latter term is small by the fact that only the small norm columns of $U$ were pruned away.

Now the only missing part that remains to justify is why the $\cR_\beta$ regularizer promotes orthogonality in the columns of approximate second order stationary points and the expression in \cref{eq:apxorth} holds. This is best understood by looking at the regularized loss for the case $\beta = 0$, which is equivalent to $\| UU^T - U_\star U_\star^T \|_F^2 + \lambda \sum_{i=1}^k \| U e_i \|_2$ and consider any candidate first-order stationary point $U$ of this objective. Let $Z \in \mathbb{R}^{d \times k}$ be a variable constrained to satisfy $Z Z^T = UU^T$. Stationarity of $U$ implies that the choice $Z = U$ must also be a first-order stationary point of the constrained optimization problem,
\begin{align}
    \text{Minimize: } \| ZZ^T - U_\star U_\star^T \|_F^2 + \lambda \sum_{i=1}^k \| Z e_i \|_2, \label{eq:constopt} \qquad \text{Subject to: } ZZ^T = U U^T. 
\end{align}
The first term in the objective is a constant under the constraint and we may remove it altogether. When $U$ is a full-rank stationary point, constraint qualification holds, and it is possible to write down the necessary KKT first-order optimality conditions, which reduce to,
\begin{align} \label{eq:kkt}
    \forall i \in [k], \quad - \lambda \frac{Z e_i}{\| Z e_i \|_2} + (\Lambda + \Lambda^T) Z e_i = 0
\end{align}
where $\Lambda \in \mathbb{R}^{d \times d}$ is the set of optimal dual variables. Since $Z = U$ is a first-order stationary point of the problem in \cref{eq:constopt} and it satisfies \cref{eq:kkt}, the above condition means that the columns $U e_i$ are the eigenvectors of the symmetric matrix $\Lambda + \Lambda^T$. If all the eigenvalues of $\Lambda + \Lambda^T$ were distinct, then this implies that the eigenvectors are orthogonal and $Ue_i \perp U e_j$ for all $i \ne j$.

While this analysis conveys an intuitive picture, there are several challenges in extending this further. It is unclear how to establish that the eigenvalues of $\Lambda + \Lambda^T$ are distinct. Moreover, this analysis only applies for full-rank stationary points and does not say anything about rank deficient stationary points $U$, where constraint qualification does not hold. Furthermore, it is even more unclear how to extend this analysis to approximate stationary points.
Our proof will circumvents each of these challenges by $(a)$ showing that at approximate SOSPs, even if the eigenvalues of $\Lambda + \Lambda^T$ are not distinct, the columns of $U$ are orthogonal, and $(b)$ directly bounding the gradient and Hessian of the regularized loss, rather than studying the KKT conditions to establish guarantees even for approximate stationary points which may be rank deficient. Having established guarantees for the pruning phase of the algorithm in the population setting, we next prove a result in the finite sample setting. $\hfill \blacksquare$

\section{Finite sample analysis} \label{sec:finite_sample}

Next, we extend the results of the previous section to the finite sample setting. Here, we also focus on the smooth version of the regularizer and study the following problem 
\begin{equation}
    f_{\text{emp}} (U) = \cL_{\text{emp}} (U) + \lambda \cR_\beta (U), \label{eq:fregemp}
\end{equation}
where the empirical loss $\cL_{\text{emp}} $ is defined in \cref{eq:Lemp} and the smooth version of the group Lasso regularizer $\cR_\beta (U)$ is defined in \cref{eq:reg-smooth}. In the finite sample setting, we assume that the measurement matrices satisfy the restricted isometry property (RIP) \cite{candes2011tight}, defined below.

\begin{assumption} \label{assump:RIP}
Assume that the measurement matrices $\{ A_1,\cdots,A_n\}$ are $(2k,\delta)$-RIP. In other words, for any $d \times d$ matrix with rank $\le 2k$,
\begin{equation}
    (1 - \delta) \| X \|_F^2 \le \frac{1}{n} \sum_{i=1}^n \langle A_i, X \rangle^2 \le (1 + \delta) \| X \|_F^2.
\end{equation}
This condition is satisfied, for instance, if the entries of the $A_i$'s were sampled $\mathcal{N}(0,1/d)$ (i.e. Gaussian measurements), as long as $n \gtrsim dk/\delta^2$ \cite{candes2011tight}.
\end{assumption}

\begin{theorem} \label{theorem:main_finite}
Consider the empirical loss with smooth regularization in \cref{eq:fregemp}, where $U_\star$ has rank $r$ and unit spectral norm, with its smallest singular value denoted $\sigma_r^\star$, and noise variance $\sigma^2$. Consider $U_{\text{prune}}$ as the output of the pruning phase in \Cref{alg:main} with parameters $\beta, \lambda, \epsilon, \gamma$ chosen as per \cref{eq:parameter}. If the number of samples is at least $n \ge C_{\ref{theorem:main_finite}} \frac{\sigma^2}{(\sigma^\star_r)^4} d k^2 r^5 \log(d/\eta)$ and $\delta \le c_\delta \frac{(\sigma_r^\star)^{3/2}}{\sqrt{k} r^{5/2}}$, where $C_{\ref{theorem:main_finite}} > 0$ is a sufficiently large constant, then with probability at least $ 1 - \eta$, 
\begin{enumerate}
    \item $U_{\text{prune}}$ has exactly $r$ columns.
    \item $U_{\text{prune}}$ satisfies the spectral initialization condition:
    $\| U_{\text{prune}} U_{\text{prune}}^T - U_\star U_\star^T \|_F \le \frac{1}{2} (\sigma_r^\star)^2$.
\end{enumerate}
\end{theorem}
\noindent With Gaussian measurements, the overall sample requirement (including that from the RIP condition) in \Cref{theorem:main_finite} is satisfied when $n \ge \widetilde{\Omega} \left(\frac{1+\sigma^2}{(\sigma_r^\star)^4} d k^2 r^5 \right)$. The high level analysis of this result largely follows that of \Cref{theorem:main-population} in the population setting -- we approximate the finite-sample gradient and Hessian by their population counterparts and show that the approximation error decays with the number of samples as $O(1/\sqrt{n})$. 

\subsection{Fine-tuning phase: Realizing the benefits of pruning}

In the fine-tuning phase, the learner runs a few iterations of gradient descent on $\cL_{\text{emp}}$ initialized at the pruned solution $U_{\text{prune}}$. Since the model is no longer overparameterized after pruning (by \Cref{theorem:main_finite}), there are several works analyzing the generalization performance and iteration complexity of gradient descent. Here we borrow the local convergence result of \cite{wainwright2015} which requires the initial condition that $\|U_0U_0^T-U_\star U_\star^T\|_F \leq  c (\sigma_r^\star)^2$, where $c$ is any constant less than $1$. As shown in part (b) of  Theorem~\ref{theorem:main_finite}, this initial condition is satisfied by $U_{\text{prune}}$.

\begin{theorem}{\cite[Corollary 2]{wainwright2015}} \label{theorem:fine-tuning_finite}
Suppose $\| U_{\star} \|_{\text{op}} = 1$. If we use the output of the greedy pruning in \Cref{alg:main} denoted by $U_{\text{prune}} \in \mathbb{R}^{d \times r}$ as the initial iterate for the fine-tuning phase, then after 
$t \ge m_{\text{fine-tune}} = C_{\ref{theorem:fine-tuning_finite}} \left( \sigma_1^\star/\sigma_r^\star \right)^{10} \cdot \log \left(\sigma_r^\star/\sigma_1^\star \cdot n/d \right)$ iterations, for some sufficiently large absolute   $C_{\ref{theorem:fine-tuning_finite}}>0$, the iterates of factored gradient descent satisfy,
\begin{align}
    \| U_t U_t^T - U_\star U_\star^T \|_F \lesssim \frac{\sigma}{(\sigma_r^\star)^2} \sqrt{\frac{rd}{n}}.
\end{align}
\end{theorem}

\label{sec:fine-tuning_finite_case}
\Cref{theorem:fine-tuning_finite} shows that in the fine-tuning phase, the iterates of gradient descent converges at a linear rate to the generalization error floor of $\sigma/(\sigma_r^\star)^2 \cdot \sqrt{rd/n}$, which is also known to be statistically (minimax) optimal \citep{https://doi.org/10.48550/arxiv.1009.2118,Koltchinskii2010}. This is possible because the learner is able to correctly identify the rank of the model in the pruning phase and operate in the exactly specified setting in the fine-tuning phase. In contrast, one may ask how this guarantee compares with running vanilla factored gradient descent in the overparameterized setting. This corresponds to the case where no pruning is carried out to reduce the size of the model. \cite{constantine21} presented a guarantee for the convergence of factored gradient descent from a warm start in the overparameterized setting. In comparison with the exactly specified case, the generalization error floor $\lim_{t \to \infty} \| U_t U_t^T - U_\star U_\star^T \|_F$ is shown to scale as $O(\sigma/(\sigma_r^\star)^2 \cdot \sqrt{kd/n})$\footnote{The result is often stated in terms of the smallest non-zero eigenvalue of $X_\star = U_\star U_\star^T$, which equals $(\sigma_r^\star)^2$.} which now depends on $k$. Furthermore, linear convergence can no longer be established because of the ill-conditioning of the objective. This is not just an artifact of the proof - experimentally too, the convergence slowdown was noticed in \cite[Figure 1]{constantine21}.

This discussion shows that greedy pruning the model \textit{first}, prior to running gradient descent, in fact generates solutions which generalize better and also converge much faster.

\begin{remark}
Based on \Cref{theorem:main_finite,theorem:fine-tuning_finite}, under Gaussian measurements, given
\begin{align}
    n \ge n_\varepsilon = O \left( \frac{\sigma^2}{(\sigma_r^\star)^4} \frac{rd}{\varepsilon^2} + \frac{1+\sigma^2}{(\sigma_r^\star)^4} dk^2 r^5 \log (d/\eta) \right)
\end{align}
samples, \Cref{alg:main} produces $U_{\text{out}}$ that has exactly $r$ columns and  satisfies $\| U_{\text{out}} U_{\text{out}}^T - U_\star U_\star^T \|_F \le \varepsilon$. Note that the sample complexity depends on the amount of overparameterization, $k$, only in the lower order (independent of $\varepsilon$) term. Note that $n_\varepsilon$ is non-zero even if $\sigma = 0$.
\end{remark}

\begin{figure}[t!]
    \centering
    \includegraphics[width=0.6\textwidth]{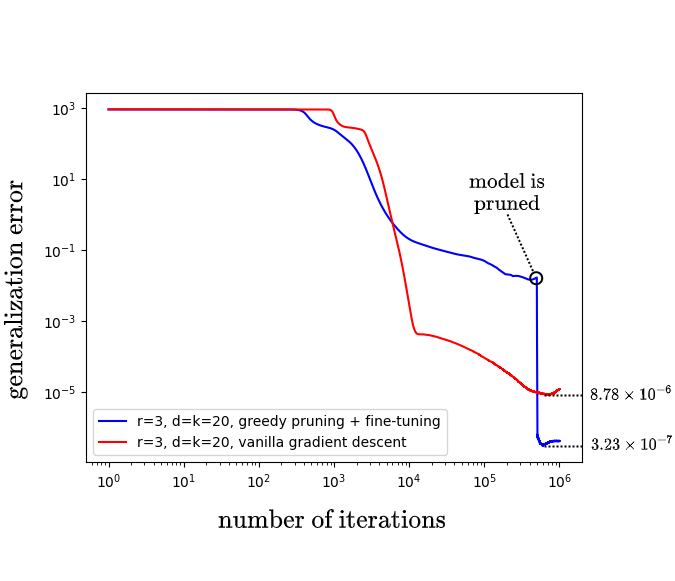}
    \caption{\small \textbf{Comparison of generalization of vanilla gradient descent vs. greedy pruning + fine-tuning.} Vanilla gradient descent is carried out to minimize the unregularized loss in \cref{eq:Lemp} (red). The empirical version of the group Lasso regularized loss, \cref{eq:regloss}, is used in the initial training phase of greedy pruning + fine-tuning (blue). The resulting solution is pruned based on the observed column norms, and fine-tuned using gradient descent on the unregularized loss. In this figure, $d=k=20$ and $r=3$ and the learner has access to $n=100$ Gaussian measurements with no added label noise. The figures plot the generalization error, $\| U_t U_t^T - U_\star U_\star^T \|_F^2$ as a function of the iteration number, $t$. The regularization parameter in \cref{eq:regloss} was tuned minimally for greedy pruning + fine-tuning objective. A sharp improvement in the generalization error is observed after pruning is carried out and the error floor is observed to be an order of magnitude better for the pruned + fine-tuned solution.
    }
    \label{fig:2}
\end{figure}

\section{Implementing \Cref{alg:main}: Smoothness and optimization oracles}
\label{sec:oracle}

In this section we instantiate the optimization oracle in \Cref{alg:main}, which outputs an approximate SOSP with bounded operator norm. First, we establish that the loss $f_{\text{emp}}$ is well behaved on the domain $\{ U : \| U \|_{\text{op}} \le 3 \}$, in that its gradient and Hessian are Lipschitz continuous. These conditions are required by many generic optimization algorithms which return approximate second order stationary points \citep{https://doi.org/10.48550/arxiv.1703.00887,NEURIPS2018_069654d5}. We establish these properties for the population loss for simplicity and leave extensions to the empirical loss for future work.

\begin{theorem} \label{theorem:lipgradhess}
Consider the population loss $f_{\text{pop}}$ in \cref{eq:fregpop}. Assume $\lambda \le \min \{ \beta, \sqrt{\beta} \}$ and $\| U_\star \|_{\text{op}} = 1$. The objective $f_{\text{pop}} (\cdot)$ defined in \cref{eq:fregpop} satisfies for any $U,V \in \mathbb{R}^{d \times k}$ such that $\| U \|_{\text{op}} , \| V \|_{\text{op}} \le 3$,
\begin{enumerate}
    \item[(a)] Lipschitz gradients: $\| \nabla f_{\text{pop}} (U) - \nabla f_{\text{pop}} (V) \|_F \lesssim \| U - V \|_F$,
    \item[(b)] Lipschitz Hessians: $\| \nabla^2 f_{\text{pop}} (U) - \nabla^2 f_{\text{pop}} (V) \|_{\text{op}} \lesssim \| U - V \|_F$.
\end{enumerate}
\end{theorem}

Under the Lipschitz gradients and Lipschitz Hessian condition, a large number of algorithms in the literature show convergence to an approximate SOSP. A  common approach for finding such a point is using the noisy gradient descent algorithm \citep{chi_escaping}.
However, note that we establish these Lipschitzness properties on the bounded domain $\{ U : \| U \|_{\text{op}} \le 3 \}$, and it remains to verify whether these algorithms indeed approach such points. A similar concern is present with \Cref{alg:main}, which requires access to an optimization oracle which finds an $(\epsilon,\delta)$-approximate SOSP of $f_{\text{emp}}$ which are also bounded, in that $\| U \|_{\text{op}} \le 3$. 
The final step is to identify conditions under which these algorithms indeed output stationary points which are bounded, satisfying $\| U \|_{\text{op}} \le 3$. We establish this behavior for a wide family of perturbed gradient based methods.

\paragraph{Perturbed gradient descent:} We consider gradient descent with the following update rule: starting from the initialization $U_0$, for all $t \ge 0$,
\begin{align} \label{eq:gd}
    U_{t+1} \gets U_t - \alpha ( \nabla f_{\text{pop}} (U_t) + P_t)
\end{align}
where $P_t$ is a perturbation term, which for example, could be the explicit noise $(P_t \sim \mathrm{Unif} (\mathbb{B} (r))$ for appropriate $r$) added to escape strict saddle points in \cite{chi_escaping}. Over the course of running the update rule \cref{eq:gd}, we show that $\| U_t \|_{\text{op}}$ remains bounded under mild conditions if the algorithm is initialized within this ball. In combination with \Cref{theorem:lipgradhess}, this shows that the noisy gradient descent approach of \citep{chi_escaping} can be used to find the SOSPs required for \Cref{alg:main}.

\begin{theorem} \label{theorem:bounded}
Consider optimization of the regularized population loss using the update rule in \cref{eq:gd}. Assume that $\| U_\star \|_{\text{op}} = 1$ and suppose the parameters are selected as$\alpha \le 1/8$ and $\lambda \le \sqrt{\beta}$, and we have $\| P_t \|_{\text{op}} \le 1$ almost surely for each $t \ge 0$. Then, assuming that the condition $\| U_0 \|_{\text{op}} \le 3$ is satisfied at initialization, for every $t \ge 1$, we have $\| U_t \|_{\text{op}} \le 3$.
\end{theorem}

Thus, when perturbed gradient descent is carried out to optimize the regularized loss $f_{\text{pop}}$ \cref{eq:fregpop}, the algorithm always returns bounded iterates. In conjunction with \Cref{theorem:lipgradhess}, this implies, for example, that noisy gradient descent \cite{chi_escaping} converges to an approximate second-order stationary point of the regularized population objective.

\section{Conclusion}

In this paper, we studied the efficacy of the greedy pruning + fine-tuning pipeline in learning low-complexity solutions for the matrix sensing problem, as well as for learning shallow neural networks with quadratic activation functions. We showed that training on the mean squared error augmented by a natural group Lasso regularizer results in models which are suitable for greedy pruning. Given sufficiently many samples, after pruning away the columns below a certain $\ell_2$-norm threshold, we arrived at a solution with the correct column sparsity of $r$. Running a few iterations of gradient descent to fine-tune the resulting model, the population loss was shown to converge at a linear rate to an error floor of $O(\sqrt{rd/n})$, which is also statistically optimal. We also presented a negative result showing the importance of regularization while training the model. To the best of our knowledge, our results provide the first theoretical guarantee on the generalization error of the model obtained via the greedy pruning + fine-tuning framework.

\section*{Acknowledgements}
The research of A. Mokhtari is supported in part by NSF Grants 2007668 and 2127697, ARO Grant W911NF2110226, the National AI Institute for Foundations of Machine Learning (IFML), the Machine Learning Lab (MLL), and the Wireless Networking and Communications Group (WNCG) industrial affiliates program at UT Austin. Kannan Ramchandran would like to acknowledge support from NSF CIF-2002821 and ARO fund 051242-00.

{
\bibliographystyle{unsrt}
\bibliography{references}
}

\newpage
\appendix

\section*{\begin{center} Supplementary material \end{center}}
\startcontents[sections]
\printcontents[sections]{l}{0}{\setcounter{tocdepth}{2}}

\addtocontents{toc}{\protect\setcounter{tocdepth}{3}}

\section{Failure of implicit regularization in pruning: Proof of \Cref{theorem:lb-implicit}}
\label{app:1}

In this section, we prove the negative result showing that running gradient descent trained with implicit regularization results in dense networks that are not compatible with greedy pruning.

\noindent For scale reasons, and to simplify the exposition, we will assume that $\| U_\star \|_2 = 1$ through the remainder of this proof. The gradient flow update rule can be written as,
\begin{align} \label{eq:gradflow}
    \frac{dU}{dt} = -\nabla \cL (U) = - (UU^T - U_\star U_\star^T) U
\end{align}
where $U = U(t)$ is the weight matrix at time $t$.
Define the column vector $r(t) = U^T U_\star \in \mathbb{R}^{k}$ capturing the correlation/alignment between the weights on the neurons of $U_t$ with the ground truth $U_\star$, the signal component. Wherever convenient, we will drop the time argument in $r(t)$ and simply refer to it as $r$.

We will show that given that gradient flow is run starting from a small initialization, at convergence, the column norms $\| U e_i \|_2$'s approximately distribute themselves proportional to the alignment of the corresponding column of $U$ with $U_\star$ initially at $t=0$. Namely,
\begin{align} \label{eq:similar-dist}
    \lim_{t \to \infty} \| U (t) e_i \|_2 \propto |\langle r(0),e_i \rangle|
\end{align}
Given a random initialization where the entries of $U (0)$ are i.i.d., $|\langle r(0), e_i \rangle|$ for each $i$ are independent. Moreover, when $U (0)$ follows a Gaussian distribution, we expect no single coordinate $|\langle r(0), e_i \rangle| = |\langle (U(0))^T U_\star, e_i \rangle|$ to be much larger than the others, a statement which we formally prove in \Cref{lemma:unidist}. Combining this with \cref{eq:similar-dist} results in a proof of \Cref{theorem:lb-implicit}.\\

\noindent The rest of this section is devoted to proving a formal version of \cref{eq:similar-dist}, which we state below.

\begin{lemma} \label[lemma]{lemma:ratiopreserved}
Suppose the initialization scale parameter $\alpha \le c \eta^2 / k^3d \log(kd)$ for a sufficiently small absolute constant $c > 0$. With probability $\ge 1 - O(\frac{1}{k})$ over the initialization, for each column $i \in [k]$,
\begin{align} \label{eq:13}
    (1 - 5 \eta) \frac{|\langle r(0), e_i \rangle|}{\| r(0) \|_2} \le \lim_{t \to \infty} \| U(t) e_i \|_2 \le \frac{|\langle r(0), e_i \rangle|}{\| r(0) \|_2} (1 + 4 \sqrt{\eta})
\end{align}
\end{lemma}

As a counterpart to the signal component $r (t)$, the noise component is $E(t) = (I - U_\star U_\star^T ) U$. At a high level, the proof of \Cref{lemma:ratiopreserved} will be to establish that with a small initialization, the noise component satisfies $E \approx 0$ approximately, while the signal component $\| r (t) \|_2$ grows exponentially fast to $1$ with $t$, and $\frac{d\langle r(t), e_i \rangle}{dt} \approx  \langle r(t), e_i \rangle$ until $\| r(t) \|_2$ gets sufficiently close to $1$. This will imply that $\langle r(t), e_i \rangle \approx \langle r(0), e_i \rangle e^t$ until then, and by extension,
\begin{align} \label{eq:ratiopreserved}
    \frac{\langle r(t), e_i \rangle}{\langle r(t), e_j \rangle } \approx \frac{\langle r(0), e_i \rangle}{\langle r(0), e_j \rangle }.
\end{align}
Since the noise component $E \approx 0$, $\| U e_i \|_2 \approx |\langle r(t), e_i \rangle|$, plugging which into \cref{eq:ratiopreserved} results in the ratio form of \Cref{lemma:ratiopreserved}. By the gradient flow equation \cref{eq:gradflow}, it is a short calculation to see that $r$ and $E$ evolve per the following differential equations,

\begin{minipage}[t]{.47\linewidth}
\begin{align}
    \frac{d r}{dt} &= r(1 - \| r \|_2^2) - E^T E r. \label{eq:rtdiff}
\end{align}
\end{minipage}
\begin{minipage}[t]{.47\linewidth}
\begin{align}
    \frac{d E}{dt} &= - E U^T U. \label{eq:Etdiff}
\end{align}
\end{minipage}

\vspace{1.5em}
\noindent In the sequel, we will show that as a function of $t$, $E^T E r$ decays linearly. The resulting differential equation $\frac{dr}{dt} \approx r ( 1 - \| r \|_2^2)$ shows linear convergence of $r$ until the point when $\| r \|_2$ approaches $1$. We place these intuitions formally by discussing a formal proof of \Cref{lemma:ratiopreserved}.

\subsection{Proof outline of \Cref{lemma:ratiopreserved}} Firstly, as a function of the scale parameter $\alpha$, we bound the energy of the signal and noise components at initialization in \Cref{lemma:init}. In \Cref{lemma:Etdec} we first establish that $\| E (t) \|_2$ does not increase with time. In \Cref{lemma:rtub}, we establish an upper bound on the signal norm $\| r(t) \|_2$ as a function of time. This shows that the signal energy cannot grow to be $1$ too rapidly, which is necessary to show that the error term $E^T E r$ has sufficient time to decay to $0$. In \Cref{lemma:linear-slow} we show that the signal norm $\| r(t) \|_2^2$ does not fall below a threshold of $3/4$ after $t$ grows to be sufficiently large. This is essential in proving \Cref{lemma:Er-bound} which shows that the norm of the error term $\| E^T E r\|_2$ in \cref{eq:rtdiff} begins decaying after $\| r(t) \|_2^2$ becomes larger than $1/2$. Finally, in \Cref{lemma:rtlb,lemma:rtlb-old} we use these results to prove a refined bound on the rate at which $\|r (t) \|_2 \to 1$. These results are collectively used in \Cref{lemma:rti-lb,lemma:rti-ub} to prove the upper and lower bounds on $|\langle r(t), e_i \rangle| \approx \| U e_i \|_2$. \\

\subsection{Understanding gradient flow: Proof of \Cref{lemma:ratiopreserved}}

In this section we prove \Cref{lemma:ratiopreserved} formally.  First we establish bounds on the scale of parameters at initialization.
\begin{lemma} \label[lemma]{lemma:init}
With probability $\ge 1 - O(\frac{1}{k})$, at initialization,
\begin{align}
    &\| E(0) \|_F^2 , \| r(0) \|_2^2 \le 2 \alpha^2 k d \\
    & \| r(0) \|_2^2 \ge  \cdot \left[ \frac{k\alpha^2}{10}, 10 k\alpha^2 \right] \\
    &\forall i \in [k], |\langle r(0), e_i \rangle| \ge \frac{\alpha}{k^{2}}
\end{align}
\end{lemma}
\begin{proof}
Recall that $\langle r, e_i \rangle$ for each $i \in [k]$ is distributed $\sim \mathcal{N} (0, \alpha^2)$. By Gaussian anti-concentration \cite{hdp} and union bounding, $\prob{ \min_{i \in [k]} |\langle r(0), e_i \rangle| \ge \frac{\alpha}{k^{2}}} \gtrsim \frac{1}{k}$ simultaneously for all $i \in [k]$. On the other hand, since every entry of $U(0)$ is iid distributed $\sim \mathcal{N} (0,\alpha^2)$, by tail bounds for $\chi^2$-distributed random variables \cite[Lemma 1]{10.1214/aos/1015957395}, $\| E(0) \|_F^2 \le \| U (0) \|_F^2 \le 2\alpha^2 k d$ with probability $\ge 1 - \exp (-kd) \ge 1 - O(\frac{1}{k})$. The same bound applies for $\| r(0) \|_2^2$ since it is also upper bounded by $\| U (0) \|_F^2$. Finally, the upper and lower bounds on $\| r(0) \|_2$ follows by noting that $\| r(0) \|_2^2 = \sum_{i=1}^{k} \langle U_\star, U e_i \rangle^2$, which concentrates around $k \alpha^2$. The result directly follows by concentration of $\chi^2$-random variables.
\end{proof}

\noindent As a corollary of \Cref{lemma:init} when the intialization parameter satisfies the upper bound in \Cref{lemma:ratiopreserved}, $\| E(0) \|_F$ and $ \| r(0) \|_2$ are upper bounded by small constants. We will use this fact several times in proving \Cref{lemma:ratiopreserved}. Next we establish that the error $\| E \|_F$ does not grow with time.

\begin{lemma} \label[lemma]{lemma:Etdec}
At any time $t \ge 0$, $\| E(t) \|_F \le \| E(0) \|_F$.
\end{lemma}
\begin{proof}
The proof follows by showing that the time derivative of $\| E(t) \|_F^2$ is non-positive. Indeed,
\begin{align}
    \frac{d \| E \|_F^2}{dt} = 2\left\langle E, \frac{dE}{dt} \right\rangle 
    = - 2\tr ( E^T E U^T U )  
    = - 2\| U E^T \|_F^2 
    \le 0.
\end{align}    
\end{proof}

\noindent Starting from a small initialization, this means that the error matrix $E$ remains small over the course of gradient flow. In for any coordinate $i \in [k]$, by \cref{eq:rtdiff}, the differential equation governing $\langle r, e_i \rangle = \langle r(t), e_i \rangle$ is,
\begin{align}
    \frac{\langle r, e_i \rangle}{dt}&= \langle r, e_i \rangle (1 - \| r \|_2^2) - e_i^T E^T E r.
\end{align}
By \Cref{lemma:Etdec} we expect the error term $E^T E r$ to be small (not necessarily decaying) in comparison with the first term as long as $\| r(t) \|_2$ is smaller than an absolute constant. In this regime, the differential form is easy to control since we expect $\langle r, e_i \rangle$ to grow linearly and across the different coordinates $i \in [k]$, we expect its value to be proportional to that at initialization, $\langle r(0), e_i \rangle$.

However, when $\| r \|_2$ eventually approaches $1$, the relative contribution of the signal term $\langle r, e_i \rangle (1 - \| r \|_2^2)$ and the error term $E^T E r$ become important. This is the main technical challenge in proving \Cref{lemma:ratiopreserved}. We will show that even as $t \to \infty$, the error term cannot change $\langle r, e_i \rangle$ by more than a constant factor. For a sufficiently large constant $C > 0$, define,
\begin{align}
    T_0 \triangleq - \frac{1}{2} \frac{\log ( \| r (0) \|_2^2) - 1}{1 - \| E(0) \|_F^2}.
\end{align}
As we will later show, $T_0$ controls the amount of time it takes the signal norm $\| r(t) \|_2$ to grow to an absolute constant. Firstly, we show that when the initialization scale $\alpha$ is sufficiently small, $T_0$ is approximately $- \log \| r(0) \|_2$.

\begin{lemma} \label[lemma]{lemma:T0bound}
When the initialization scale $\alpha \le c/k^3 d \log(kd)$ for a sufficiently small absolute constant $c > 0$,
    \begin{align} \label{eq:T0_bound_main}
        -\log (\| r (0) \|_2) + \frac{1}{2} \le T_0 \le -\log (\| r (0) \|_2) + \frac{3}{2}.
\end{align}
\end{lemma}
\begin{proof}
This follows from the fact that,
\begin{align}
    T_0 &= - \log (\| r (0) \|_2) - \frac{\| E(0) \|_F^2}{1 - \| E(0) \|_F^2} \log(\| r(0) \|_2) + \frac{1}{2(1 - \| E(0) \|_F^2)} \label{eq:62-1-1}
\end{align}
The lower bound on $T_0$ follows readily, noting that $\| r(0) \|_2, \| E(0) \|_F < 1$ from the initialization bounds in \Cref{lemma:init} when $\alpha \le c/k^3 d \log(kd)$. For the upper bound on the other hand, note that the last term of \cref{eq:62-1-1} is upper bounded by,
\begin{align}
    \label{eq:1E0}
    \frac{1}{2(1 - \| E(0) \|_F^2)} \le 1
\end{align}
And the middle term of \cref{eq:62-1-1} is upper bounded by,
\begin{align} \label{eq:E0r0_bound}
    -\frac{\| E(0) \|_F^2}{1 - \| E(0) \|_F^2} \log(\| r(0) \|_2) \le \frac{\| E(0) \|_F^2}{1 - \| E(0) \|_F^2} \frac{1}{\| r(0 ) \|_2} \overset{(i)}{\le} \frac{2 \alpha^2 kd}{(1 - 2 \alpha^2 kd)} \frac{10}{\sqrt{k} \alpha} \le \frac{1}{2},
\end{align}
where the last inequality assumes that $\alpha \le c_{\ref{eq:E0r0_bound}} / \sqrt{k} d$ for a small constant $c_{\ref{eq:E0r0_bound}} > 0$, and inequality $(i)$ uses the bounds on $\| E(0) \|_F$ and $\| r(0) \|_2$ proved in \Cref{lemma:init}. Combining \cref{eq:1E0,eq:E0r0_bound} with \cref{eq:62-1-1} results in the proof of \Cref{lemma:T0bound}.
\end{proof}

\noindent The next result we establish is an upper bound on the signal norm $\| r (t) \|_2$. The proof follows by upper bounding the rate of change of $\| r (t) \|_2$ and integrating the resulting bound.

\begin{lemma} \label[lemma]{lemma:rtub}
At any time $t \ge 0$, the signal norm is upper bounded by,
\begin{align}
    \| r(t) \|_2^2 \le \frac{e^{2(t-T_0)+1+\eta)}}{1 + e^{2(t - T_0)+1+\eta)}} \le 1.
\end{align}
\end{lemma}
\begin{proof}
From the differential equation governing $r$ we can infer that,
\begin{align} \label{eq:rtub}
    \frac{d \| r \|_2^2}{dt} = 2 \| r \|_2^2 (1 - \| r \|_2^2) - 2 \| E r \|_2^2 \le 2 \| r \|_2^2 (1 - \| r \|_2^2)
\end{align}
By a theorem of \cite{Petrovitch1901} on differential inequalities, the trajectory of $\| r (t) \|_2^2$ as a function of $t$ starting from some reference point $\| r(0) \|_2^2$ is pointwise lower bounded by the trajectory of $\| r (t) \|_2^2$ obtained when the inequality is set to be an equality. In particular, by integrating the differential equation this results in the lower bound,
\begin{align}
    \log \frac{\| r(t) \|_2^2}{1 - \| r(t) \|_2^2} - \log \frac{\| r(0) \|_2^2}{1 - \| r(0) \|_2^2} \le 2 t
\end{align}
Observe that $- \log \frac{\| r(0) \|_2^2}{1 - \| r (0) \|_2^2} \ge - \log \| r(0) \|_2^2 - \eta \ge 2 T_0 - 1 - \eta$ where the first inequality uses the upper bound on $\| r (0) \|_2^2 \le \eta$ when $\alpha \le c / k^3 d \log(kd)$ and the second by \Cref{lemma:T0bound}. Therefore by rearranging the terms we have,
\begin{align}
    \| r(t) \|_2^2 \le \frac{e^{2t-2T_0+1+\eta}}{1 + e^{2t - 2T_0+1+\eta}}.
\end{align}
\end{proof}

As it turns out, the error term $E^T E r (t)$ in \cref{eq:rtdiff} can be shown to decrease linearly once $\| r (t) \|_2^2$ hits the critical threshold of $1/2$. In the next lemma we show that the signal norm $\| r(t) \|_2^2$ never drops below the threshold of $3/4$ for all $t \ge T_0 + 2$. Thus, after a sufficiently large amount of time, we expect the differential equation for $r$ to behave as $\frac{dr}{dt} \approx r (1 - \| r \|_2^2)$.

\begin{lemma} \label[lemma]{lemma:linear-slow}
For $t \ge T_0$, 
\begin{align}
    \| r(t) \|_2^2 \ge \frac{1 - \| E(0) \|_F^2}{1 + e^{- 2 \left( 1 - \| E(0) \|_F^2 \right) (t - T_0) + 1}}
\end{align}
As an implication, for any $t \ge T_0 + 2$, under the small initialization $\alpha \le c/k^3 d \log(kd)$, $ \| r (t) \|_2^2 \ge 3/4$.
\end{lemma}
\begin{proof}
From the differential equation governing $r$, we may infer that,
\begin{align}
    \frac{d \| r \|_2^2}{dt} = 2 \left\langle r, \frac{dr}{dt} \right\rangle = 2 \| r \|_2^2 (1 - \| r \|_2^2) - 2 \| E r \|_2^2. \label{eq:r2-rate}
\end{align}
Note that $\| E r \|_2 \le \| E \|_F \| r \|_2 \le \| E(0) \|_F \| r \|_2$ and therefore,
\begin{align}
    2 \| r \|_2^2 (1 - \| E(0) \|_F^2 - \| r \|_2^2)\le \frac{d \| r \|_2^2}{dt}
\end{align}
Rearranging and integrating both sides,
\begin{align}
    2 (1 - \| E(0) \|_F^2) t &\le \log \left( \frac{\| r(t) \|_2^2}{1 - \| E (0) \|_F^2 - \| r(t) \|_2^2} \right) - \log \left( \frac{\| r(0) \|_2^2}{1 - \| E (0) \|_F^2 - \| r(0) \|_2^2} \right) \\
    &\le \log \left( \frac{\| r(t) \|_2^2}{1 - \| E (0) \|_F^2 - \| r(t) \|_2^2} \right) + 2 (1 - \| E(0) \|_F^2) T_0 + 1
\end{align}
where the last inequality uses the choice of the initialization scale $\alpha \le c\eta^2 /k^3 d \log (kd)$, and the bounds on $\| E(0) \|_F^2, \| r(0) \|_2^2$ in \Cref{lemma:init}. Therefore,
\begin{align}
    &(1 - \| E(0) \|_F^2 - \| r (t) \|_2^2) e^{ 2 \left(1 - \| E(0) \|_F^2 \right) (t - T_0) - 1} \le \| r (t) \|_2^2 \\
    &\implies \frac{1 - \| E(0) \|_F^2}{1 + e^{ - 2 \left(1 - \| E(0) \|_F^2 \right) (t - T_0) + 1}} \le \| r (t) \|_2^2
\end{align}
\end{proof}

Having established that $\| r(t) \|_2^2$ does not decay below the threshold of $3/4$ beyond time $T_0 + 2$, we establish that this condition is sufficient for the error term $E^T Er$ to begin decaying to $0$. Below we instead bound $\| E r \|_F$, and a bound on $\| E^T E r \|_F$ is obtained by upper bounding it as $\| E (t) \|_F \| Er (t) \|_F \le \| E (0) \|_F \| E r \|_F$ by \Cref{lemma:Etdec}.

\begin{lemma} \label[lemma]{lemma:Er-bound}
At any time $t \ge 0$, the error term,
\begin{align}
    \| Er (t) \|_2 &\le \frac{\| E r (0) \|_2 e^{t}}{1 + e^{2 \left( 1 - \| E (0) \|_F^2 \right) (t-T_0) - 1}}
\end{align}
\end{lemma}
\begin{proof}
By explicit computation,
\begin{align}
    \frac{d Er}{dt} &= - E U^T U r + E r ( 1 - \| r \|_2^2) - E E^T E r \\
    &= - E (E^T E + r r^T) r + E r ( 1 - \| r \|_2^2) - E E^T E r \\
    &= - 2 E (E^T E) r  + E r ( 1 - 2 \| r \|_2^2) \label{eq:Erdiff}
\end{align}
By taking an inner product with $E r$, we get that,
\begin{align}
    \frac{d \| Er \|_2^2}{dt} = -4 \| (E^T E) r \|_2^2 + 2 \| Er \|_2^2 (1 - 2 \| r \|_2^2).
\end{align}
Using the lower bound on $\| r \|_2^2$ in \Cref{lemma:linear-slow},
\begin{align}
    \frac{d \| Er \|_2^2}{dt} \le 2 \| Er \|_2^2 \left( 1 - \frac{ 2(1 - \| E(0) \|_F^2)}{1 + e^{-2 \left( 1 - \| E (0) \|_F^2 \right) (t-T_0) + 1}} \right)
\end{align}
Rearranging and integrating both sides from time $0$ to $t$ using the fact that $\int \frac{dx}{1+e^{-x}} = \log (1 + e^x)$,
\begin{align}
    \log \| Er (t) \|_2^2 - \log \| Er (0) \|_2^2 &\le 2 \left( t - \log \left( 1 + e^{2 \left( 1 - \| E (0) \|_F^2 \right) (t-T_0) - 1} \right) \right) \\
    \implies \| Er (t) \|_2 &\le \frac{\| E r (0) \|_2 e^{t}}{1 + e^{2 \left( 1 - \| E (0) \|_F^2 \right) (t-T_0) - 1}},
\end{align}
where the last inequality follows by exponentiating both sides and rearranging.
\end{proof}

Using the fact that $\| E r \|_F$ rapidly decays to $0$ after time $T_0$, we can in fact establish a more refined lower bound on the rate at which $\| r(t) \|_2$ approaches $1$. The decay of the error term is not captured in the prior lower bound on $\| r(t) \|_2$ in \Cref{lemma:linear-slow}. The error decay established in \Cref{lemma:Er-bound} is essential to proving such a result, especially as $\| r(t) \|_2$ approaches $1$. In this regime, the error term $\| E r \|_2^2$ in \cref{eq:r2-rate} becomes comparable with the leading order term $\| r \|_2^2 ( 1 - \| r \|_2^2)$.

\begin{lemma} \label[lemma]{lemma:rtlb}
At any time $t \le 3T_0/2$, for some absolute constant $C_{\ref{eq:991244}} > 0$,
\begin{align} \label{eq:991244}
    \| r(t) \|_2^2 \ge \frac{e^{2t}}{\| r (0) \|_2^{-2} + e^{2t} - 1} - C_{\ref{eq:991244}} \frac{\| E (0) \|_F^2 T_0}{\|r(0)\|_2}.
\end{align}
\end{lemma}
\begin{proof}
From \Cref{lemma:Er-bound}, note that,
\begin{align}
    \| Er (t) \|_2^2 &\le \frac{\| E (0) \|_F^2 \| r(0) \|_2^2 e^{2t}}{ \left( 1 + e^{2 \left( 1 - \| E (0) \|_F^2 \right) (t-T_0) - 1}\right)^2} \\
    &\lesssim \frac{\| E r (0) \|_F^2 e^{2t}}{ \left( 1 + e^{2 (t-T_0)}\right)^2}, \label{eq:r2-rate-1}
\end{align}
where the last inequality uses the fact that when $t \le 4 T_0$, $(1 - \| E(0) \|_F^2) (t - T_0) \ge t - T_0 - 4$ (see the analysis in \cref{eq:62-1-1,eq:E0r0_bound}). This inequality also uses \Cref{lemma:T0bound} to bound $\| r(0) \|_2^2$. From \Cref{lemma:Er-bound}, therefore, denoting $x = \| r \|_2^2$, for some absolute constant $C_{\ref{eq:r2-rate-1}} > 0$,
\begin{align}
    \frac{d x}{dt} &= 2 x(1-x) - C_{\ref{eq:r2-rate-1}} \frac{e^{2t} \| E r (0) \|_F^2}{\left( 1 + e^{2 (t-T_0)} \right)^2}\\
    \implies e^{-2t} \frac{dx}{dt} &= 2 e^{-2t} x - 2e^{-2t} x^2 - C_{\ref{eq:r2-rate-1}} \frac{\| E r (0) \|_F^2}{\left( 1 + e^{2 (t-T_0)} \right)^2} \\
    \implies \frac{d}{dt} \left\{ e^{-2t} x \right\} &= -2 e^{-2t} x^2 - C_{\ref{eq:r2-rate-1}} \frac{\| E r (0) \|_F^2}{\left( 1 + e^{2 (t-T_0)} \right)^2} \\
    \implies \frac{dy}{dt} &= -2 y^2 e^{2t} - C_{\ref{eq:r2-rate-1}} \frac{\| E r (0) \|_F^2}{\left( 1 + e^{2 (t-T_0)} \right)^2}, \label{eq:dyt-bound}
\end{align}
where $y = e^{-2t} x$. Define a new variable $\tilde{y} (t)$ with $\tilde{y} (0) = y(0)$ and satisfying the following differential equation corresponding to just the first term on the RHS of \cref{eq:dyt-bound},
\begin{align}
    \frac{1}{\tilde{y}^2} \frac{d\tilde{y}}{dt} = -2 e^{2t}
    \iff \tilde{y}(t) = \frac{1}{1/y(0) + e^{2t}-1}
\end{align}
On the other hand, since $\frac{dy}{dt} \le -2y^2 e^{2t}$, rearranging and integrating both sides, we get $y(t) \le 1/(y(0) + e^{2t} - 1) = \tilde{y}(t)$. Plugging this back into \cref{eq:dyt-bound},
\begin{align}
    \frac{dy}{dt} \ge - 2 (\tilde{y})^2 e^{2t} - C_{\ref{eq:r2-rate-1}} \frac{\| E r (0) \|_F^2}{\left( 1 + e^{2 (t-T_0)} \right)^2}.
\end{align}
However, by definition, $\frac{d \tilde{y}}{dt} = - 2 (\tilde{y})^2 e^{2t}$ and therefore, integrating both sides, for some absolute constant $C_{\ref{eq:r2-rate-1}}' > 0$, we get,
\begin{align}
    y(t) - y(0) \ge \tilde{y} (t) - \tilde{y} (0) - C_{\ref{eq:r2-rate-1}}' \| E r (0) \|_F^2 T_0
\end{align}
Plugging in $\tilde{y} (t)$ and subsequently $y(t)$ and $x(t)$ results in the equations,
\begin{align}
    \| r (t) \|_2^2 \ge \frac{e^{2t}}{\| r(0) \|_2^{-2} + e^{2t} - 1} - C_{\ref{eq:r2-rate-1}}' \| E r (0) \|_F^2 T_0 e^{2t}
\end{align}
When $t \le 3T_0/2$, by using the upper bound on $T_0$ from \Cref{lemma:T0bound}, the second term on the RHS itself is upper bounded by $C_{\ref{eq:r2-rate-1}}'' \| E (0) \|_F^2 T_0 /\|r(0)\|_2$ for another absolute constant $C_{\ref{eq:r2-rate-1}}'' > 0$. This completes the proof.
\end{proof}

While the refined convergence lemma in \Cref{lemma:rtlb} applies for the case when $t \le 3T_0/2$, we also prove a lemma for the case when $t \ge 3T_0/2$.

\begin{lemma} \label[lemma]{lemma:rtlb-old}
At any $t \ge 3T_0/2$,
\begin{align}
    e^{3t/4} (1 - \| r (t) \|_2^2) \lesssim 1 + e^{11T_0/4} \| E r (0) \|_2^2 + e^{3T_0/4}.
\end{align}
\end{lemma}
\begin{proof}
Consider any time $t \ge 3T_0/2$ which is greater than $ T_0+2$ by the small initialization and \Cref{lemma:init}. By \Cref{lemma:linear-slow}, $\| r \|_2^2 \ge 3/4$. From \cref{eq:r2-rate} and \Cref{lemma:Er-bound},
\begin{align}
\forall t \ge T_0,\ \frac{d\| r \|_2^2}{dt} &\ge \frac{3}{4} (1 - \| r \|_2^2) - 2\frac{\| E r (0) \|_2^2 e^{2t}}{ \left( 1 + e^{2 \left( 1 - \| E (0) \|_F^2 \right) (t-T_0) - 1}\right)^2}
\end{align}
Multiplying both sides by $e^{3t/4}$ and rearranging,
\begin{align}
     & e^{3t/4} \frac{d(1 - \| r \|_2^2)}{dt} + \frac{3}{4} e^{3t/4} (1 - \| r \|_2^2) \le 2\frac{\| E r (0) \|_2^2 e^{2t} e^{3t/4}}{ \left( 1 + e^{2 \left( 1 - \| E (0) \|_F^2 \right) (t-T_0) - 1}\right)^2} \\
     \implies & \frac{d}{dt} \left( e^{3t/4} (1 - \| r \|_2^2) \right) \le 2 \int \frac{\| E r (0) \|_2^2 e^{2t} e^{3t/4}}{ \left( 1 + e^{2 \left( 1 - \| E (0) \|_F^2 \right) (t-T_0) - 1}\right)^2} dt. \label{eq:ubEr0}
\end{align}
Integrating from $T_0$ to $t$, we can upper bound \cref{eq:ubEr0} as,
\begin{align}
    e^{3t/4} (1 - \| r (t) \|_2^2)- e^{3T_0/4} (1 - \| r (T_0) \|_2^2) \lesssim \int_{T_0}^t \frac{\| E r (0) \|_2^2 e^{2t} e^{3t/4}}{ e^{4 \left( 1 - \| E (0) \|_F^2 \right) (t-T_0)}} dt,
\end{align}
By the small initialization, $\| E (0) \|_F^2 \le 1/8$ and the denominator in the integral is lower bounded by $e^{7/2 (t - T_0)}$. Therefore,
\begin{align}
    e^{3t/4} (1 - \| r (t) \|_2^2)- e^{3T_0/4} (1 - \| r (T_0) \|_2^2) &\lesssim e^{7T_0/2} \int_{T_0}^t \| E r (0) \|_2^2 e^{-3t/4} dt \\
    &\lesssim e^{7T_0/2} \| E r (0) \|_2^2 e^{-3T_0/4} \\
    &= 2 e^{11T_0/4} \| E r (0) \|_2^2 \label{eq:b2}
\end{align}
Therefore,
\begin{align}
    e^{3t/4} (1 - \| r (t) \|_2^2) \lesssim 1 + e^{11T_0/4} \| E r (0) \|_2^2 + e^{3T_0/4}.
\end{align}
\end{proof}

\noindent Finally we are ready to prove the lower and upper bounds on the limiting value of column norms of $U_t$ as $t \to \infty$. \Cref{lemma:rti-lb} establishes the lower bound, while \Cref{lemma:rti-ub} establishes the upper bound.

\begin{lemma} \label[lemma]{lemma:rti-lb}
When the initialization parameter $\alpha \le c/k^3 d \log(kd)$, then with probability $\ge 1 - O(1/k)$ over the initialization, for any $t \ge 4T_0$,
\begin{align}
    \| U (t) e_i \|_2 \ge (1 - \eta) \frac{|\langle r (0), e_i \rangle|}{\| r(0) \|_2}.
\end{align}    
\end{lemma}
\begin{proof}
From the differential equation governing $r$, \cref{eq:rtdiff}, for any coordinate $i \in [k]$,
\begin{align}
\frac{d \langle r, e_i \rangle^2}{dt} &= 2\langle r, e_i \rangle^2 (1 - \| r \|_2^2) - 2 \langle r, e_i \rangle e_i E^T E r \label{eq:reider-2} \\
&\ge 2\langle r, e_i \rangle^2 (1 - \| r \|_2^2) - 2 |\langle r, e_i \rangle| \| E \|_F \| E r \|_2
\end{align}
From \Cref{lemma:rtub} and \Cref{lemma:Er-bound}, and since $\| E (t) \|_F \le \| E (0) \|_F$, at any time $t \ge 0$,
\begin{align}
&\frac{d \langle r, e_i \rangle^2}{dt} \ge \frac{2\langle r, e_i \rangle^2}{e^{2(t-T_0)+1+\eta)}+ 1} - \frac{2 |\langle r, e_i \rangle| \| E(0) \|_F \| Er (0) \|_F e^t}{1 + e^{2 ( 1- \| E(0) \|_F^2) (t - T_0) - 1}} \nonumber\\
\implies &\frac{d |\langle r, e_i \rangle|}{dt} - \frac{|\langle r, e_i \rangle|}{e^{2(t-T_0)+1+\eta}+ 1} \ge - \frac{\| E(0) \|_F^2 e^t}{1 + e^{2 ( 1- \| E(0) \|_F^2) (t - T_0) - 1}} \label{eq:reider}
\end{align}
Define $q (t) = e^{- 2(t - T_0) + 1 + \eta}$ (which we abbreviate simply as $q$) and multiply both sides by $\sqrt{1+q}$,
\begin{align}
\sqrt{1+q} \frac{d |\langle r, e_i \rangle|}{dt} - |\langle r, e_i \rangle| \frac{q}{\sqrt{1+q}} &\ge - \frac{2 \| E(0) \|_F^2 e^t \sqrt{1+q}}{1 + e^{2 ( 1- \| E(0) \|_F^2) (t - T_0) - 1}}
\end{align}
Noting that $\frac{d}{dt} \sqrt{1 + q} = -\frac{q}{\sqrt{1+q}}$, we get that,
\begin{align}
\implies \frac{d}{dt} \left( |\langle r, e_i \rangle| \sqrt{1+q} \right) &\ge  - 2 \| E(0) \|_F^2 e^{- (t - T_0) + T_0 - 2} \cdot \frac{\sqrt{1 + e^{2(t - T_0 + 1)}}}{1 + e^{2 ( 1- \| E(0) \|_F^2) (t - T_0) - 1}}.
\end{align}
We further lower bound the RHS for the case when $t \le 4 T_0$. Later we will give a different proof to show that when $t \ge 4T_0$, $|\langle r, e_i \rangle|$ does not change significantly then onward.\\

For $t \le 4 T_0$, $(1 - \| E (0) \|_F^2) (t - T_0) \ge t-T_0 - 4$. This follows from the same analysis as \cref{eq:62-1-1,eq:E0r0_bound}. In this regime, we therefore have that,
\begin{align}
    \frac{\sqrt{1 + e^{2(t - T_0) + 1 + \eta}}}{1 + e^{2 ( 1- \| E(0) \|_F^2) (t - T_0) - 1}} \lesssim \frac{1}{\sqrt{1 + e^{2 (t-T_0)+1+\eta}}} \label{eq:403002}
\end{align}
On the other hand,
\begin{align}
    e^{-(t - T_0) + T_0 - 2} \lesssim e^{T_0} \sqrt{1 + e^{-2(t-T_0) + 1 + \eta}} \label{eq:403003}
\end{align}
Multiplying \cref{eq:403002,eq:403003} results in the lower bound,
\begin{align}
    \frac{d}{dt} \left( |\langle r, e_i \rangle| \sqrt{1+q} \right) &\gtrsim - \| E (0) \|_F^2 e^{T_0}
\end{align}
Integrating both sides from $t = 0$ to $4 T_0$, and since $T_0 \ge 1$ by the small initialization bound,
\begin{align} \label{eq:00192555}
    |\langle r (4T_0),e_i| \ge | \langle r(0), e_i \rangle| \sqrt{\frac{1 + e^{2T_0-1 - \eta}}{1 + e^{-4T_0}}} - \frac{c_{\ref{eq:00192555}} \| E (0) \|_F^2 T_0 e^{T_0}}{\sqrt{1 + e^{-4T_0}}}
\end{align}
where $c_{\ref{eq:00192555}} > 0$ is an absolute constant. Since $e^{T_0} \ge - \log(\| r (0) \|_2) + 1/2$ from \Cref{lemma:T0bound}, and since $\| r(0) \|_2^2 \le \eta$ by the small initialization, the term multiplying $|\langle r(0), e_i \rangle|$ on the RHS is upper bounded by $\sqrt{\frac{1 + e^{2T_0-1 - \eta}}{1 + e^{-T_0}}} \ge (1 - 2\eta) e^{T_0-1/2-\eta/2} \ge (1 - 3 \eta) e^{T_0 - 1/2}$. On the other hand, noting again by \Cref{lemma:T0bound} that $T_0 \le - \log(\| r (0) \|_2) + 3/2$, and since $\| r (0) \|_2^2 \gtrsim k \alpha^2$ and $\| E(0) \|_F^2 \lesssim \alpha^2 kd$ and furthermore $|\langle r(0), e_i \rangle| \ge \alpha/k^2$ from \Cref{lemma:init}, as long as $\alpha \le c\eta/ k^3 d \log (kd)$ for a sufficiently small constant $c > 0$,
\begin{align}
    |\langle r(0), e_i \rangle| \ge \frac{c_{\ref{eq:00192555}} \| E (0) \|_F^2 T_0}{\eta \sqrt{1 + e^{-4T_0}}}. \label{eq:62}
\end{align}
And by implication, \cref{eq:00192555} gives,
\begin{align} \label{eq:63}
    |\langle r (4T_0),e_i| \ge (1 - 4 \eta) | \langle r(0), e_i \rangle| e^{T_0-1/2},
\end{align}
where the last inequality follows from the lower bound on $T_0$ in \Cref{lemma:T0bound}. \\

\noindent Finally we show that, from time $4T_0$ onward, $\langle r(t), e_i \rangle$ does not change much. From \cref{eq:reider},
\begin{align}
    \frac{d |\langle r, e_i \rangle|}{dt} &\ge - \frac{\| E(0) \|_F^2 e^t}{1 + e^{2 ( 1- \| E(0) \|_F^2) (t - T_0) - 1}} \gtrsim -\frac{\| E(0) \|_F^2 e^t}{e^{(9/5) (t - T_0)}}.
\end{align}
where the last inequality uses the fact that $\| E (0) \|_F^2 \le 1/10$ by \Cref{lemma:init} and the upper bound on the initialization scale $\alpha \le c\eta^2 /k^3 d \log (kd)$. Integrating both sides from $4T_0$ to $t$,
\begin{align}
    |\langle r (t), e_i \rangle| - |\langle r (4T_0), e_i \rangle| &\gtrsim - \| E (0) \|_F^2 e^{9T_0/5} \int_{4 T_0}^\infty e^{-4t/5} dt \\
    &\gtrsim - \| E (0) \|_F^2 e^{-7T_0/5} \\
    &\overset{(i)}{\gtrsim} - |\langle r(0), e_i \rangle| e^{- 7 T_0/5}
\end{align}
where $(i)$ follows from \cref{eq:62}. Finally, combining with \cref{eq:63},
\begin{align} \label{eq:134}
    |\langle r(t), e_i \rangle| \ge (1 - 4 \eta) |\langle r(0), e_i \rangle| \left( e^{T_0-1/2} - c_{\ref{eq:134}} e^{-7T_0/5} \right)
\end{align}
for some absolute constant $c_{\ref{eq:134}} > 0$. By noting that $T_0 \ge -\log \| r (0) \|_2 + 1/2$ from \Cref{lemma:T0bound} and the fact that $\| r (0) \|_2 \lesssim \alpha^2 k$, by choosing $\alpha \le c_{\ref{eq:134}}'/\eta \sqrt{k}$ for some absolute constant $c_{\ref{eq:134}}' > 0$ results in the inequality,
\begin{align}
    \forall t \ge T_0,\ |\langle r(t), e_i \rangle| \ge (1 - 5 \eta) |\langle r(0), e_i \rangle| e^{T_0 - 1/2} \ge (1 - 5 \eta) \frac{|\langle r(0), e_i \rangle|}{\| r(0) \|_2}.
\end{align}
Since $\| U e_i \|_2 \ge |\langle r(t), e_i \rangle|$, this completes the proof.
\end{proof}

\begin{lemma} \label[lemma]{lemma:rti-ub}
Suppose $\alpha \le c \eta^2 / \eta k^{3} d \log (kd)$ for a sufficiently small absolute constant $c > 0$. Then with probability $\ge 1 - O(1/k)$ over the initialization, for any $t \ge 3T_0/2$,
\begin{align}
    \| U(t) e_i \|_2 \le (1 + 4\sqrt{\eta}) \frac{|\langle r (0), e_i \rangle|}{\| r(0) \|_2}.
\end{align}    
\end{lemma}
\begin{proof}
Following the proof of the lower bound in \cref{lemma:rti-lb}, from the differential equation governing $r$, for any coordinate $i \in [k]$,
\begin{align}
\frac{d \langle r, e_i \rangle^2}{dt} &= 2\langle r, e_i \rangle^2 (1 - \| r \|_2^2) - 2 \langle r, e_i \rangle e_i E^T E r \\
&\le 2\langle r, e_i \rangle^2 (1 - \| r \|_2^2) + 2 |\langle r, e_i \rangle| \| E \|_F \| E r \|_2 \\
\implies \frac{d |\langle r, e_i \rangle|}{dt} &\le |\langle r, e_i \rangle| (1 - \| r \|_2^2) + \frac{\| E (0) \|_F \| E r (0) \|_2 e^t}{1 + e^{2 (1 - \| E(0)\|_F^2) (t-T_0) - 1}} \label{eq:temp1}
\end{align}
where in the last inequality, we bound $\| E \|_F \le \| E (0) \|_F$ and apply \Cref{lemma:Er-bound} to upper bound the error term $\| E r (t) \|_2$. Akin to \cref{lemma:rti-lb}, we carry out the analysis of $|\langle r(t) , e_i \rangle|$ in two parts, we first analyze its growth from time $0$ to $3T_0/2$. From time $3T_0/2$ to $t$ we show that $|\langle r (t),e_i \rangle|$ does not change significantly.

\noindent In particular, at any time $t \le 3T_0/2$, since $(1 - \| E(0) \|_F^2) (t - T_0) \ge t-T_0-4$ (see the analysis in \cref{eq:62-1-1,eq:E0r0_bound}), for some absolute constant $C_{\ref{eq:00192111}}$,
\begin{align} \label{eq:00192111}
\frac{d |\langle r, e_i \rangle|}{dt} &\le |\langle r, e_i \rangle| (1 - \| r \|_2^2) + C_{\ref{eq:00192111}} \frac{ \| E (0) \|_F \| E r(0) \|_2 e^{t}}{1 + e^{2 (t-T_0)}} \\
&\le |\langle r, e_i \rangle| (1 - \| r \|_2^2) + C_{\ref{eq:00192111}} \| E (0) \|_F^2 \| r(0) \|_2 e^{T_0},
\end{align}
where the last inequality uses the fact that $e^t / (1 + e^{2(t-T_0)})$ is maximized at $t = T_0$. Plugging in the lower bound on $\| r \|_2^2$ in \Cref{lemma:rtlb},
\begin{align}
    \frac{d |\langle r, e_i \rangle|}{dt} &\le |\langle r, e_i \rangle| \left( \frac{\| r(0) \|_2^{-2} - 1}{\| r (0) \|_2^{-2} + e^{2t} - 1} + C_{\ref{eq:991244}} \frac{\| E (0) \|_F^2 T_0}{\|r(0)\|_2}\right) + C_{\ref{eq:00192111}} \| E (0) \|_F^2 \| r(0) \|_2 e^{T_0} \nonumber\\
    &\le |\langle r, e_i \rangle| \left( \frac{p}{p + e^{2t}} \right) + C_{\ref{eq:00192111-1}} \| E (0) \|_F^2 T_0. \label{eq:00192111-1}
\end{align}
where $C_{\ref{eq:00192111-1}} >0$ is a sufficiently large absolute constant, and $p = \| r(0) \|_2^{-2} - 1$. Multiplying both sides by $y = \sqrt{p e^{-2t}+1}$ and noting that $\frac{dy}{dt} = - \left( \frac{p}{p + e^{2t}} \right) y$, we get,
\begin{align}
    \frac{d}{dt} \left( y |\langle r, e_i \rangle| \right) &= C_{\ref{eq:00192111-1}} \| E (0) \|_F^2 T_0.
\end{align}
Integrating both sides from $0$ to $3T_0/2$,
\begin{align}
    &|\langle r (3T_0/2), e_i \rangle| \sqrt{p e^{-3T_0} + 1} - |\langle r (0), e_i \rangle| \sqrt{1+p}  = C_{\ref{eq:00192111-1}}' \| E(0) \|_F^2 T_0^2 \\
    \implies & |\langle r (3T_0/2), e_i \rangle|  \le C_{\ref{eq:00192111-1}}' \| E(0) \|_F^2 T_0^2 + \frac{|\langle r(0), e_i \rangle|}{\| r(0) \|_2} \label{eq:102}
\end{align}
By \Cref{lemma:init}, at initialization, $\| r (0) \|_2^2 \lesssim \alpha \sqrt{k}$ and $|\langle r(0), e_i \rangle| \ge \alpha/k^2$. Therefore, by the small initialization condition, $C_{\ref{eq:00192111-1}}' \| E (0) \|_F^2 T_0^2 \le \eta |\langle r(0), e_i \rangle| / \| r(0) \|_2$ and,
\begin{align} \label{eq:r3T0/2}
    |\langle r (3T_0/2), e_i \rangle| \le (1 + \eta) \frac{|\langle r(0), e_i \rangle|}{\| r(0) \|_2}.
\end{align}
Next we show that from $3T_0/2$ to $t$, $|\langle r(t), e_i \rangle|$ does not change significantly. By the same analysis as \cref{eq:temp1},
\begin{align}
    \frac{d |\langle r, e_i \rangle|}{dt} \le |\langle r, e_i \rangle| (1 - \| r \|_2^2) + \frac{2 \| E(0) \|_F^2 e^{t}}{1 + e^{2 \left( 1 - \| E (0) \|_F^2 \right) (t-T_0) - 1}}
\end{align}
Plugging in the upper bound on $1 - \| r (t) \|_2^2$ from \Cref{lemma:rtlb-old} and simplifying,
\begin{align}
    \frac{d |\langle r, e_i \rangle|}{dt} &\lesssim |\langle r, e_i \rangle| e^{-3t/4} \left( e^{3T_0/4} + e^{11T_0/4} \| E r (0) \|_2^2 \right) + \frac{2 \| E(0) \|_F^2 e^{t}}{e^{2 \left( 1 - \| E (0) \|_F^2 \| r (0) \|_2 \right) (t-T_0)}}, \\
    &\overset{(i)}{\le} |\langle r, e_i \rangle| e^{-3t/4} \left( e^{3T_0/4} + e^{11T_0/4} \| E r (0) \|_2^2 \right) + \frac{2 \| E(0) \|_F^2 \| r (0) \|_2 e^{t}}{e^{7/4 (t-T_0)}} \\
    &= |\langle r, e_i \rangle| e^{-3t/4} \left( e^{3T_0/4} + e^{11T_0/4} \| E r (0) \|_2^2 \right) + 2 \| E(0) \|_F^2 \| r (0) \|_2 e^{-3t/4} e^{7T_0/4},\\
    &= \left( |\langle r, e_i \rangle| \left( 1 + e^{2T_0} \| E r (0) \|_2^2 \right) + 2 \| E(0) \|_F^2 \| r (0) \|_2 e^{T_0} \right) e^{-3(t-T_0)/4}, \label{eq:0044562}
\end{align}
where $(i)$ invokes the fact that $\| E(0) \|_F^2 \le 1/8$ from \Cref{lemma:init}, by the bound on the initialization scale $\alpha \le c\eta/k^3 d \log (kd)$. Simplifying the terms in \cref{eq:0044562},
\begin{align}
    e^{2T_0} \| E r(0) \|_2^2 \le e^{2T_0} \| E (0) \|_2^2 \| r(0) \|_2^2 \overset{(i)}{\lesssim} \| E (0) \|_2^2 \le 1, \label{eq:87}
\end{align}
where $(i)$ is by \Cref{lemma:T0bound} and the last equation follows from the scaling of $\alpha$ and \Cref{lemma:init}. Likewise, $\| E(0) \|_F^2 \| r(0) \|_2 e^{T_0} \lesssim \| E(0) \|_F^2$, and plugging this and \cref{eq:87} into \cref{eq:0044562},
\begin{align}
    \frac{d |\langle r, e_i \rangle|}{dt} &\lesssim \left( |\langle r, e_i \rangle| + 2 \| E(0) \|_F^2 \right) e^{-3(t-T_0)/4}
\end{align}
Integrating both sides,
\begin{align}
    \log \frac{|\langle r (t), e_i \rangle| + 2 \| E(0) \|_F^2}{|\langle r (3T_0/2), e_i \rangle| + 2 \| E(0) \|_F^2} \lesssim \int_{3T_0/2}^\infty e^{-3(t-T_0)/4} dt \lesssim e^{-3 T_0/8}.
\end{align}
Therefore, for some absolute constant $C_{\ref{eq:116}} > 0$,
\begin{align}
    |\langle r (t), e_i \rangle| &\le \left( |\langle r (3T_0/2), e_i \rangle| + 2 \| E(0) \|_F^2 \right) (1 + C_{\ref{eq:116}} e^{-3T_0/8}) \label{eq:116} \\
    &\overset{(i)}{\le} \left( |\langle r (3T_0/2), e_i \rangle| + 2 \| E(0) \|_F^2 \right) (1 + \sqrt{\eta}) \\
    \implies \| U (t) e_i \|_2 &\overset{(ii)}{\le} \| E(0) \|_F + \left( |\langle r (3T_0/2), e_i \rangle| + 2 \| E(0) \|_F^2 \right) (1 + \sqrt{\eta}), \label{eq:777776}
\end{align}
where $(i)$ lower bounds $T_0$ using \Cref{lemma:T0bound} and uses the fact that at initialization $\| r(0) \|_2 \gtrsim \alpha \sqrt{k}$ and therefore when $\alpha \le c \eta^{4/3}$ for a sufficiently small constant $c>0$, $C_{\ref{eq:116}} e^{-3T_0/8} \le \sqrt{\eta}$. On the other hand, $(ii)$ uses triangle inequality to bound $\| U(t) e_i \|_2 \le |\langle r (t), e_i \rangle| + \| E(t) e_i \|_2 \le |\langle r (t), e_i \rangle| + \| E(t) \|_F \le |\langle r (t), e_i \rangle| + \| E(0) \|_F$ by \Cref{lemma:Etdec}. Using the bound on $|\langle r(3T_0/2), e_i \rangle|$ in \cref{eq:r3T0/2},
\begin{align}
    \| U (t) e_i \|_2 &\le \| E(0) \|_F + \left( \frac{| \langle r(0), e_i\rangle|}{\| r(0) \|_2} + 2 \| E(0) \|_F^2 \right) (1 + \sqrt{\eta}) \\
    &\le \frac{| \langle r(0), e_i\rangle|}{\| r(0) \|_2} (1 + 4 \sqrt{\eta}).
\end{align}
The last inequality uses the fact that at initialization, $|\langle r (0),e_i \rangle| \ge \alpha/k^2$ and $\| r(0) \|_2 \lesssim \alpha \sqrt{k}$. Therefore, $|\langle r (0),e_i \rangle| / \| r(0) \|_2 \ge 1/k^{5/2}$ and $\| E (0) \|_F \le \alpha \sqrt{kd}$. Therefore, by the small initialization condition on $\alpha$, $\| E(0) \|_F, \| E(0) \|_F^2 \le \eta |\langle r (0),e_i \rangle| / \| r(0) \|_2$.
\end{proof}

\noindent Combining the statements of \Cref{lemma:rti-lb} and \Cref{lemma:rti-ub} shows that, with probability $\ge 1 - O(1/k)$ over the initialization, at any time $t \ge 4 T_0$,
\begin{align}
    \frac{|\langle r (0), e_i \rangle|}{\| r(0) \|_2} (1 - 5 \eta) \le \| U (t) e_i \|_2 \le \frac{|\langle r (0), e_i \rangle|}{\| r(0) \|_2} (1 + 4 \sqrt{\eta}).
\end{align}
This completes the proof of \Cref{lemma:ratiopreserved}.

\subsection{Behavior at initialization: many columns are ``active''}

\Cref{lemma:ratiopreserved} establishes the limiting behavior of gradient flow as a function of the initialization, specifically that the norm of a column grows to a value proportional to the alignment of the column with $U_\star$ at initialization, $|\langle r(0), e_i \rangle|$. In this section, we show that at initialization, $\langle r(0), e_i \rangle$ is significant for many $i \in [d]$ compared to the maximum among them.

At time $t=0$, each coordinate of $r(0)$ is distributed as the inner product of a Gaussian vector $\sim \mathcal{N} (0, \alpha^2 I)$ with a fixed vector $U_\star$. By Gaussian concentration, we therefore expect each coordinate of $r(0)$ to concentrate around $\alpha$, and no one coordinate of $r(0)$ to be significantly larger than the others. The next lemma makes this claim precise.

\begin{lemma}{Many columns are ``active" at initialization}
\label[lemma]{lemma:unidist}
For any $\eta > 0$, let $S_{\max} (\eta)$ denote the set \\
$\left\{ i \in [k] : \frac{|\langle r(0), e_i \rangle|}{\max_{j \in [k]} |\langle r(0), e_j \rangle|} \ge 1 - \sqrt{\eta} \right\}$. For any $\eta \le 1 - \frac{1}{\log(k)}$,
\begin{align} \label{eq:unidist}
    \prob{  | S_{\max} (\eta)| \le \frac{k^\eta}{2\sqrt{2(1-\eta) \log(k)}} } \le \frac{c_{\ref{eq:unidist}} \sqrt{\log(k)}}{k^\eta}.
\end{align}
for a sufficiently small absolute constant $c_{\ref{eq:unidist}} > 0$. In other words, with high probability, $\widetilde{\Omega} (k^\eta)$ of the column indices $i \in [k]$ are significantly correlated with $U_\star$ compared to the maximum among all columns.
\end{lemma}
\begin{proof}
By rotation invariance of Gaussians and since $\| U_\star \|_2 = 1$, for each $i$, $\langle r(0), e_i \rangle \overset{\text{i.i.d.}}{\sim} \mathcal{N} (0, \alpha^2)$. By Gaussian anti-concentration, for each $i \in [k]$, from \cite[Proposition 2.1.2]{hdp}, for $t > 0$,
\begin{align} \label{eq:04}
    \prob{ |\langle r(0), e_i \rangle| \ge t \alpha} \ge \left( \frac{1}{t} - \frac{1}{t^3} \right) \cdot \frac{1}{\sqrt{2 \pi}} e^{-t^2/2} \triangleq p(t).
\end{align}
Therefore roughly we expect $k p(t)$ of the columns $i \in [k]$ to satisfy the condition $\{ |\langle r(0), e_i \rangle| \ge t \alpha \}$. In particular, by the independence across $i$, the number of successes follows a binomial distribution with number of trials $k$ and probability of success $p(t)$. Denote $S (t) = \{ i \in [k] : |\langle r(0), e_i \rangle| \ge t\alpha \}$. By binomial concentration \citep{hdp},
\begin{align} \label{eq:05}
    \prob { |S (t)| \le \frac{k p(t)}{2}} \le e^{- \frac{k p(t)}{4}}.
\end{align}
On the other hand, for i.i.d standard normal random variables, $X_1,\cdots,X_k \sim \mathcal{N} (0,1)$, the supremum satisfies the following concentration inequality 
\begin{align}
    \prob{\max_{1 \le i \le k} X_i \ge \sqrt{2 \log (k)} + \eta} \le e^{-\eta^2/2}
\end{align}
Therefore,
\begin{align} \label{eq:06}
    \prob{ \max_{i \in [k]} |\langle r(0), e_i \rangle| \le \alpha \sqrt{2\log(k)} + \alpha \sqrt{2\eta \log(k)}} \ge 1 - k^{-\eta}.
\end{align}
For each $t > 0$, define,
\begin{align}
    S_{\max} (t) = \left\{ i \in [k] : \frac{|\langle r(0), e_i \rangle|}{\max_{i \in [k]} |\langle r(0), e_i \rangle|} \ge \frac{t}{\sqrt{2\log (k)} (1 + \sqrt{\eta})} \right\},
\end{align}
and combining \cref{eq:04,eq:05,eq:06} results in the inequality,
\begin{align} \label{eq:0012919}
    \prob{ |S_{\max} (t)| \le \frac{k p(t)}{2} } \le e^{- \frac{k p(t)}{4}} + k^{-1/8} \le \frac{4}{k p(t)} + k^{-\eta}.
\end{align}
With the choice of $t = \sqrt{2(1-\eta)\log(k)}$ and for any $\eta \le 1 - \frac{1}{\log(k)}$, we obtain that  $p(t) \ge k^{-(1 - \eta)} / 4\sqrt{(1 - \eta) \log(k)}$. Plugging into \cref{eq:0012919} and using the  results in the bound,
\begin{align}
    \prob{ |S_{\max} (\sqrt{2 ( 1-\eta)\log(k)})| \le \frac{k^\eta}{4 \sqrt{ ( 1 - \eta) \log (k)}}} \le \frac{1 + 16 \sqrt{2 (1 - \eta) \log(k)}}{k^\eta}
\end{align}
Using the definition of $S_{\max} (\sqrt{2 ( 1-\eta) \log(k)})$ and the fact that $\frac{\sqrt{1 - \eta}}{1 + \sqrt{\eta}} \ge 1 - \sqrt{\eta}$ completes the proof.
\end{proof}

\subsection{Putting it all together: Proof of \Cref{theorem:lb-implicit}}
Below we state and prove a stronger version of \Cref{theorem:lb-implicit} which follows by combining \Cref{lemma:ratiopreserved,lemma:unidist}.

\begin{theorem}
\label{theorem:lb-implicit-restate}
Consider the population loss defined in \eqref{eq:Lpop}, for the case that $r=1$ and $k\gg 1$. Moreover, assume that the entries of the initial model $U_0$ are i.i.d. samples from $\mathcal{N} (0,\alpha^2)$, where $\alpha \le \eta^2/k^3 d \log (kd)$. Then, for any constant $\eta \in (0,1)$, the iterates generated by gradient descent converge to a model $U_{gd}$ where $\widetilde{\Omega} (k^{\eta})$ columns of which are active and satisfy, 
\begin{align}
    \frac{\| U_{gd} e_i \|_2}{\max_{j \in [k]} \| U_{gd} e_j \|_2} \ge 1 - C_{\ref{eq:123121112}} \sqrt{\eta},
\end{align}
 with probability $\ge 1 - \poly {1/d,1/k^{\eta}}$, where $C_{\ref{eq:123121112}} > 0$ is a sufficiently large constant.
\end{theorem}

In particular, if for some $i$,
\begin{align}
    \frac{|\langle r(0), e_i \rangle|}{\max_{j \in [k]} |\langle r(0), e_j \rangle|} \ge 1 - \sqrt{\eta}
\end{align}
Then, by \Cref{lemma:ratiopreserved},
\begin{align} \label{eq:0012955419}
    \lim_{t \to \infty} \frac{\| U (t) e_i \|_2}{\max_{j \in [k]} \| U (t) e_j \|_2} \ge 1 - C_{\ref{eq:0012955419}} \sqrt{\eta}.
\end{align}
where $C_{\ref{eq:0012955419}} > 0$ is a sufficiently large constant. Invoking \Cref{lemma:unidist} completes the proof of \Cref{theorem:lb-implicit-restate}.

\section{Population analysis of the regularized loss: Proof of \Cref{theorem:main-population}} \label{app:2}

\noindent In this section, we study the second-order stationary points of loss $f_{\text{pop}} (U) = \cL_{\text{pop}} (U) + \lambda \cR_\beta (U)$ (\eqref{eq:fregpop}) which is the regularized loss in the population (infinite sample) setting. The main result we prove in this section is about approximate second-order stationary points of the regularized loss $f_{\text{pop}} (U)$. When $\lambda$ and $\beta$ are chosen appropriately, we show that such points $(i)$ are ``pruning friendly'' in that greedily pruning the columns of $U$ based on their $L_2$ norm results in a solution $U_{\text{prune}}$ having exactly $r$ columns, and $(ii)$ the resulting solution $U_{\text{prune}}$ satisfies $\| U_{\text{prune}} U_{\text{prune}}^T \|_F^2 \le c (\sigma_r^\star)^2$ where $c > 0$ is a small constant, and serves essentially as a ``spectral initialization'' in the sense of \cite{constantine21} for the subsequent fine-tuning phase.

The proof of \Cref{theorem:main-population} relies on two main observations: $(i)$ showing that at approximate second-order stationary points $U U^T \approx U_\star U_\star^T$, where the error is small when $\lambda$ is small, $(ii)$ the regularizer ensures that the columns of $U$ that are not too small in $\ell_2$-norm are all approximately orthogonal to one another, in that the angle between pairs of vectors is $\approx 90^\circ$. Since the columns are orthogonal, the rank of $U$ equals the number of non-zero columns. However, $U$ is close to a rank $r$ matrix since $UU^T \approx U_\star U_\star^T$. This will imply that $U$ also has approximately $r$ non-zero columns. Moreover, pruning away the columns of $U$ at the correct threshold will result in a model having exactly $r$ columns remaining, while at the same time not affecting the population loss significantly.

The intuition behind $UU^T \approx U_\star U_\star^T$ at second order stationary points is straightforward - prior work \cite{rong} characterizes the behavior of such points for matrix sensing in the absence of any regularization, showing that $UU^T = U_\star U_\star^T$ at second-order stationary points, and establishing a strict saddle condition for the population loss $\cL_{\text{pop}}$. In the presence of regularization, as long as $\lambda$ is small, we do not expect the behavior to change significantly. As we will discuss in more detail in \Cref{app:oracle}, the regularizer $\cR_\beta$ satisfies gradient and Hessian-Lipschitzness as long as $\beta$ is not too small. This will suffice in showing that the locations of first and second order stationary points do not change significantly when $\lambda$ is small.

As introduced earlier in \cref{eq:apxorth}, the main result we prove in this section is that the ``bounded'' (in operator norm) $(\epsilon,\gamma)$-approximate second-order stationary points of $f_{\text{pop}}$ returned by the optimization oracle $\mathcal{O}$, in \Cref{alg:main} satisfies the following condition,
\begin{equation} \nonumber
    \forall i,j : \| U e_i \|_2, \| U e_j \|_2 \ge 2 \sqrt{\beta},\quad \frac{\langle U e_i, U e_j \rangle}{\| U e_i \|_2 \| U e_j \|_2} \approx 0.
\end{equation}
In other words, all the large columns of $U$ have their pairwise angle approximately $90^\circ$. By pruning away the columns of $U$ that have an $\ell_2$ norm less than $2 \sqrt{\beta}$, the remaining columns of $U$, i.e., which are a superset of the columns of $U_{\text{prune}}$, are now approximately at $90^\circ$ angles to one another. But since $UU^T \approx U_\star U_\star^T$, we know that there can be at most $r$ significant columns of $U$. Therefore, after pruning the resulting model has at most $r$ columns.

Note that the above discussion does not absolve the risk of over-pruning the model. If the pruning threshold is not chosen carefully, it might be possible to end up with a model having fewer than $r$ columns. Such a model cannot generalize to a vanishing error even in the population setting. As we will show, it is possible to establish that error of the pruned model, $\| U_{\text{prune}} U_{\text{prune}}^T - U_\star U_\star^T \|_F$ is also small, which comes to the rescue. We can show that it is at most $\frac{1}{2} (\sigma_r^\star)^2$. On the other hand if $U_{\text{prune}}$ indeed had fewer than $r$ columns, the approximation error must be at least $(\sigma_r^\star)^2$. This shows that the greedy pruning strategy we employ is not too aggressive.

\subsection{Proof outline of \Cref{theorem:main-population}}

In \Cref{lemma:grad-lb-matrix-sensing} we begin with a lower bound on the gradient norm of the reulgarized loss. The lower bound  that at exact stationary points, every pair of columns $i,j \in [d]$ will either have $\| Ue_i \|_2 = \| U e_j \|_2$, or will be orthogonal to each other. Next we use this calculation to show that at approximate second order stationary points of the regularized loss, every pair of columns which are not too small in $\ell_2$ norm are approximately orthogonal to each other. Next, in \Cref{lemma:negativeeig} we show that at approximate second order stationary points of the loss, we also expect $\| UU^T - U_\star U_\star^T \|_F$ to be small when $\lambda$ is small. In \Cref{lemma:ZFbound,lemma:proj-similar} we show that first-order stationary points are $U$ are aligned with the correct subspace induced by the columns of $U_\star$. Finally, we combine these results in \Cref{theorem:prune-prebound} to show bounds on the generalization loss of the pruned solution as a function of the problem parameters. By instantiating the problem parameters properly, we prove \Cref{theorem:main-population} in \Cref{app:B2}.

Note that the gradient and Hessian of $\cR_\beta (U)$ is calculated in \Cref{sec:regularizer-auxiliary-calc} and are in terms of diagonal matrices $D(U)$ and $G(U)$ defined in \cref{eq:DU-def,eq:GU-def}.

\begin{lemma}[A lower bound on the gradient norm] \label[lemma]{lemma:grad-lb-matrix-sensing}
Consider any loss function $\cL$ of the form $f(UU^T)$ where $f : \mathbb{R}^{d \times d} \to \mathbb{R}$ is a differentiable function. For any candidate matrix $U \in \mathbb{R}^{d \times k}$,
    \begin{align}
        \| \nabla (\cL + \lambda \cR_\beta) (U) \|_F^2 \ge \lambda^2 \max_{i \ne j \in [k]} (D (U)_{ii} - D (U)_{jj})^2 \frac{\langle U e_i, U e_j \rangle^2}{\| U e_i \|_2^2 + \| U e_j \|_2^2}
    \end{align}
\end{lemma}

\begin{proof}
Note that $\| \nabla (\cL + \lambda \cR_\beta) (U) \|_F^2 \ge \langle Z , \nabla (\cL + \lambda \cR_\beta) (U) \rangle^2$ for any candidate $Z \in \mathbb{R}^{d \times k} : \| Z \|_F \le 1$. The rest of this proof will be dedicated to finding such a $Z$.\\

\noindent Note from \Cref{lemma:genL-grad-hessian} that $\langle \cL (U), Z \rangle = \langle (\nabla f) (UU^T), UZ^T + ZU^T \rangle$. Suppose the perturbation $Z$ is chosen as $U W$ where $W$ is a skew-symmetric matrix. Then, $U Z^T + Z U^T = U (W + W^T) U^T = 0$. Therefore,
\begin{align}
    \langle \nabla \cL (U), Z \rangle = 0
\end{align}
On the other hand, from \Cref{lemma:R-grad-hessian},
\begin{align}
    \langle \nabla \cR_\beta (U), Z \rangle &= \tr ( D(U) U^T Z ) = \tr ( D(U) U^T U W ).
\end{align}
where $D(U)$ is defined in \cref{eq:DU-def}.
Define $(i, j)$ as an arbitrary pair of distinct indices in $[k]$. Suppose $W = e_{i} e_{j}^T - e_{j} e_{i}^T$. Then,
\begin{align}
    \langle \nabla \cR_\beta (U), Z \rangle^2
    &= ( e_{i}^T (D (U) U^T U - U^T U D(U)) e_{j} )^2 \\
    &= (D (U)_{i i} - D (U)_{j j})^2 \langle U e_{i}, U e_{j} \rangle^2.
\end{align}
Dividing throughout by $\| Z \|_F^2 = \| UW \|_F^2 = \| U e_{i} \|_F^2 + \| U e_{j} \|_F^2$ and noting that $i \ne j$ are arbitrary completes the proof.
\end{proof}

\begin{remark}
The interpretation of the lower bound on the gradient norm in \Cref{lemma:grad-lb-matrix-sensing} is best understood by looking at its behavior at first-order stationary points, where the gradient is $0$. At a first-order stationary point, for any $i,j \in [k]$ such that $i \ne j$, either $D(U)_{ii} = D(U)_{jj}$ or $U e_i \perp U e_j$. The former condition is true iff $\| U e_i \|_2 = \| U e_j \|_2$.
\end{remark}

Next we prove a result which establishes near-orthogonality of the columns of $U$ at approximate second-order stationary points.

\begin{lemma} \label[lemma]{lemma:apx-orthogonal}
Consider any loss $\cL (U)$ of the form $f (UU^T)$ where $f : \mathbb{R}^{d \times d} \to \mathbb{R}$ is doubly differentiable. Consider any $\epsilon$-approximate first-order stationary point of $\cL + \lambda \cR_\beta$, denoted $U$. Consider any $i \ne j \in [k]$ and define $C_{ij} = \frac{\langle U e_i, U e_j \rangle}{\sqrt{\| U e_i \|_2^2 + \| U e_j \|_2^2}}$. Then, if
\begin{enumerate}
    \item $\min \{ \| U e_i\|_2, \| U e_j \|_2 \} \le 2\sqrt{\beta}$, then, $|C_{ij}| \le 2\sqrt{\beta}$.
    \item In the complement case, defining $\| Z \|_F = \sqrt{\| U e_i \|_2^2 + \| U e_j \|_2^2}$, if,
    \begin{align} \label{eq:corrlb}
        |C_{ij}| \ge 5 \sqrt{\frac{\epsilon + \gamma \min \{ \| Ue_i \|_2, \| Ue_i \|_2 \}}{\lambda}} \cdot \min \{ \| U e_i \|_2 , \| U e_j \|_2 \}.
    \end{align}
    then the Hessian of the regularized loss at $U$ has a large negative eigenvalue,
    \begin{align} \label{eq:mineig}
        \lambda_{\min} (\nabla^2 (\cL + \lambda \cR_\beta) (U)) < - \gamma.
    \end{align}
\end{enumerate}
In summary, at second order stationary points the columns are approximately orthogonal in the sense that $|C_{ij}|$ is small.
\end{lemma}
\begin{proof}
Note that, regardless of whether $U$ is approximately stationary or not, for any $i \ne j \in [k]$, if $\min \{ \| U e_i\|_2, \| U e_j \|_2 \} \le 2\sqrt{\beta}$, then,
\begin{align}
    \frac{| \langle Ue_i, U e_j \rangle}{\sqrt{\| U e_i \|_2^2 + \| U e_j \|_2^2}} \le \frac{\| U e_i \|_2 \| U e_j \|_2}{\max_{i \in \{ 1,2\}} \| U e_i \|_2^2} = \min_{i \in \{ 1,2\}} \| U e_i \|_2 \le 2\sqrt{\beta}.
\end{align}
This proves the first part of the theorem. Henceforth, we will assume that $\min \{ \| U e_i\|_2, \| U e_j \|_2 \} \ge 2\sqrt{\beta}$. This condition will turn out to lower bound the Frobenius norm of a matrix $Z$ which we will require later for the proof of the second part.

The approximate first-order stationarity of $U$ implies that
\begin{align}
    \| \nabla (\cL + \lambda \cR_\beta) (U) \|_F \le \epsilon.
\end{align}
As in the proof of \Cref{lemma:grad-lb-matrix-sensing}, the proof strategy is to construct explicit directions (matrices) capturing the directions of negative curvature of the regularized loss. From \Cref{lemma:grad-lb-matrix-sensing}, the approximate first-order stationarity of $U$ implies,
\begin{align} \label{eq:initbound}
    \frac{|\langle U e_i, U e_j \rangle|}{\sqrt{\| U e_i \|_2^2 + \| U e_j \|_2^2}} \le \frac{\epsilon}{\lambda} \cdot \frac{1}{|\Delta|}.
\end{align}
where $\Delta = D(U)_{ii} - D(U)_{jj}$.  Consider any tuple $i \ne j \in [k]$. Without loss of generality we will work with $i=1, j=2$.

Consider a perturbation $Z \in \mathbb{R}^{d \times k}$ which takes the form $Z = U W$ where $W$ is a skew-symmetric matrix which satisfies the ``support condition'',
\begin{align} \label{eq:supp_condition}
    \forall i' \ge 3,\ W e_{i'} = 0, e_{i'}^T W = 0
\end{align}
In general, when $\Delta$ is small in absolute value, by the support condition \cref{eq:supp_condition} of $W$, one would expect $W D(U) \approx D(U) W$ in case $\Delta$ is small. In particular, defining the diagonal matrix $L = \textsf{diag} \left( \frac{\Delta}{2}, -\frac{\Delta}{2}, 0, \cdots \right)$, we have that,
\begin{align} \label{eq:77}
    W D(U) - D(U) W = W L - L W.
\end{align}
From the gradient and Hessian computations in \Cref{lemma:R-grad-hessian},
\begin{align}
    &\vect(Z)^T [\nabla^2 \cR_\beta (U)] \vect (Z) \nonumber\\
    &= \tr (D(U) \cdot Z^T Z) - \sum_{i = 1}^{k} G(U)_{ii} \langle Z e_i, U e_i \rangle^2 \\
    &= \tr ( W D(U) Z^T U) - \sum_{i = 1}^{k} G(U)_{ii} \langle Z e_i, U e_i \rangle^2 \\
    &= \langle \nabla \cR_\beta(U), WZ^T \rangle + \tr ( (WL-LW) \cdot Z^T U) - \sum_{i = 1}^{k} G(U)_{ii} \langle Z e_i, U e_i \rangle^2 
    \label{eq:80}
\end{align}
where the last equation uses the gradient computation in \Cref{lemma:R-grad-hessian} and the fact that $W$ and $D(U)$ approximately commute per \cref{eq:77}. Likewise, analyzing the second order behavior of $\cL$ using the Hessian computations in \Cref{lemma:genL-grad-hessian} results in the following set of equations,
\begin{align}
    &\vect (Z)^T [\nabla^2 \cL (U)] \vect (Z) \\
    &= \vect (UZ^T + ZU^T) [(\nabla^2 f ) (UU^T) ] \vect(UZ^T + ZU^T) + 2 \langle (\nabla f) (UU^T), ZZ^T \rangle \label{eq:81}\\
    &\overset{(i)}{=} 2 \langle (\nabla f) (UU^T), ZZ^T \rangle \\
    &= \langle (\nabla f) (UU^T), UW Z^T + Z W^T U^T \rangle \\
    &= \langle \nabla \cL(U) , W Z^T \rangle \label{eq:83},
\end{align}
where $(i)$ uses the fact that $UZ^T + ZU^T = U (W^T + W) U^T = 0$ and the last equation uses the uses the gradient computations in \Cref{lemma:genL-grad-hessian}.

\noindent Summing up \cref{eq:80,eq:83}, we get,
\begin{align}
    &\vect (Z)^T [ \nabla^2 (\cL + \lambda \cR_\beta) (U)] \vect(Z) \nonumber\\
    &\overset{(i)}{=} \langle \nabla (\cL + \lambda \cR_\beta) (U), W Z^T \rangle + \lambda \tr ((W L - L W ) \cdot Z^T U) - \lambda \sum_{i=1}^{k} G(U)_{ii} \langle Z e_i, U e_i \rangle^2
\end{align}
Noting that $U$ is an $\epsilon$-approximate stationary point of $\cL + \lambda \cR_\beta$, by Cauchy–Bunyakovsky–Schwarz inequality,
\begin{align}
    &\vect (Z)^T [ \nabla^2 (\cL + \lambda \cR_\beta) (U)] \vect(Z) \nonumber\\
    &\le \epsilon \| W Z^T \|_F + \lambda \tr ( (W L - L W ) \cdot Z^T U ) - \lambda \sum_{i=1}^{k} G(U)_{ii} \langle Z e_i, U e_i \rangle^2 \label{eq:middleterm}
\end{align}
Now we are ready to choose $W$. Suppose $W$ is chosen as $( e_1 e_2^T - e_2 e_1^T)$. With this choice of $W$, we have that,
\begin{align}
    \| Z \|_F = \| UW \|_F = \sqrt{\| U e_1 \|_2^2 + \| U e_2 \|_2^2}
\end{align}
Likewise,
\begin{align}
    \| WZ^T \|_F &= \| W W^T U^T \|_F \\
    &= \| U (e_1 e_1^T + e_2 e_2^T) \|_F \\
    &= \sqrt{\| U e_1 \|_2^2 + \| U e_2 \|_2^2} = \| Z \|_F.
\end{align}
Simplifying \cref{eq:middleterm}, the last term can be evaluated to,
\begin{align}
    \sum_{i=1}^{k} G(U)_{ii} \langle Z e_i, U e_i \rangle^2 &= (G(U)_{11} + G(U)_{22} ) \langle U e_1, U e_2 \rangle^2.
\end{align}
With this choice of parameters, \cref{eq:middleterm} simplifies to,
\begin{align}
    &\vect (Z)^T [ \nabla^2 (\cL + \lambda \cR_\beta) (U)] \vect(Z) \\
    &\le \epsilon \| Z \|_F + \lambda \tr ( (W L - L W ) \cdot Z^T U ) - \lambda (G(U)_{11} + G(U)_{22} ) \langle U e_1, U e_2 \rangle^2 \label{eq:middleterm-2}
\end{align}

Next we prove a lemma upper bounding the middle term in \cref{eq:middleterm}.

\begin{lemma} \label[lemma]{lemma:WLLW}
The error term $\lambda \tr ( (W L - L W ) \cdot Z^T U )$ can be upper bounded by,
\begin{align}
\tr ((W L - L W ) \cdot Z^T U) \le 40 \Delta^2 \max_{i \in \{ 1,2 \}} \| U e_i \|_2^2 \min_{i \in \{1,2\}} \| U e_i \|_2.
\end{align}
\end{lemma}
The proof is deferred to \Cref{app:B2}. Combining \cref{eq:middleterm-2} with \Cref{lemma:WLLW},
\begin{align}
    \vect (Z)^T [ \nabla^2 (\cL + \lambda \cR_\beta) (U)] \vect(Z) &\le \epsilon \| Z \|_F + 40 \lambda \Delta^2 \| Z \|_F^2 \min_{i \in \{ 1,2\}} \| U e_i \|_2 \nonumber\\
    &\qquad - \lambda (G(U)_{11} + G(U)_{22}) \langle Ue_1, U e_2 \rangle^2
\end{align}
Therefore, considering any approximate stationary point $U$ that satisfies,
\begin{align} \label{eq:C1}
    \lambda (G(U)_{11} + G(U)_{22}) \frac{\langle Ue_1, Ue_2 \rangle^2}{\| Z \|_F^2 } \ge \frac{\epsilon}{\| Z \|_F} + 40 \lambda \Delta^2 \min_{i \in \{1,2\}} \| U e_i \|_2 + \gamma \tag{C1}
\end{align}
the Hessian $\nabla^2 (\cL + \lambda \cR_\beta) (U)$ has an eigenvalue which is at most $- \gamma$. Note by a similar analysis as \cref{eq:middleterm-5},
\begin{align}
    G(U)_{11} + G(U)_{22} \ge \max_{i \in \{ 1,2 \}} G(U)_{ii} &\ge \max_{i \in \{ 1,2 \}} \frac{\| U e_i \|_2^2 + 4 \beta}{(\| U e_i \|_2^2 + \beta)^{5/2}} \\
    &\ge \max_{i \in \{ 1,2 \}} \frac{1}{3 \| U e_i \|_2^3 + 12 \beta^{3/2}} \\
    &\ge \max_{i \in \{ 1,2 \}} \frac{1}{5 \| U e_i \|_2^3}.
\end{align}
where the second-to-last inequality uses the fact that $\min_{i \in \{ 1,2 \}} \| U e_i \|_2 \ge 2\sqrt{\beta}$.
Therefore, a sufficient condition to guarantee \cref{eq:C1} is,
\begin{align}
    \max_{i \in \{ 1,2\} } \frac{\lambda}{5 \| U e_i \|_2^3} \cdot \frac{\langle U e_1, U e_2 \rangle^2}{\| Z \|_F^2} \ge \frac{\epsilon}{\| Z \|_F} + 40 \lambda  \Delta^2 \min_{i \in \{1,2\}} \| U e_i \|_2 + \gamma. \label{eq:largehessian}
\end{align}
From \cref{eq:initbound}, we have that,
\begin{align} \label{eq:nearlyorthogonal}
    \frac{\langle Ue_1, Ue_2 \rangle^2}{\| Z \|_F^2 } \le \frac{\epsilon^2}{\lambda^2 \Delta^2}
\end{align}
When $\Delta$ is large, the columns are nearly orthogonal by \cref{eq:nearlyorthogonal}. When $\Delta$ is small, the barrier to setting up a negative eigenvalue of the Hessian of $\cL + \lambda \cR_\beta$ is small by \cref{eq:largehessian}. In particular, under the ``small-$\Delta$'' condition,
\begin{align}
    \Delta \le \frac{\epsilon}{4 \lambda^{1/2} \cdot \min_{i \in \{ 1,2\}} \| U e_i \|_2}
\end{align}
$\nabla^2 (\cL + \lambda \cR_\beta)(U)$ has a negative eigenvalue taking value at most $-\gamma$ under the sufficient condition,
\begin{align}
    &\max_{i \in \{ 1,2\} } \frac{\lambda}{5 \| U e_i \|_2^3} \cdot \frac{\langle U e_1, U e_2 \rangle^2}{\| Z \|_F^2} \ge \frac{\epsilon}{\| Z \|_F} + 40 \lambda  \Delta^2 \min_{i \in \{1,2\}} \| Ue_i \|_2 + \gamma\\
    \impliedby &\max_{i \in \{ 1,2\} } \frac{\lambda}{5 \| U e_i \|_2^3} \cdot \frac{\langle U e_1, U e_2 \rangle^2}{\| Z \|_F^2} \ge \frac{\epsilon}{\| Z \|_F} +  \frac{40 \epsilon \min_{i \in \{1,2\}} \| Ue_i \|_2 }{16 \min_{i \in \{1,2\}} \| U e_i \|_2^{2}} + \gamma\\
    \impliedby &\frac{\langle U e_1, U e_2 \rangle^2}{\| Z \|_F^2} \ge \frac{18 \epsilon \cdot \min_{i \in \{1,2\}} \| U e_i \|_2^2}{\lambda} + \frac{5\gamma}{\lambda} \cdot \min_{i \in \{ 1,2\}} \| U e_i \|_2^3
\end{align}
From \cref{eq:initbound},  under the large-$\Delta$ condition, we have that,
\begin{align}
    \frac{\langle Ue_1, Ue_2 \rangle^2}{\| Z \|_F^2 } &\le \frac{\epsilon^2}{\lambda^2} \frac{16 \lambda}{\epsilon} \cdot \min_{i \in \{ 1,2\} } \| U e_i \|_2^2 \le \frac{16 \epsilon}{\lambda} \cdot \min_{i \in \{ 1,2\} } \| U e_i \|_2^{2}.
\end{align}
This completes the proof.
\end{proof}

\Cref{lemma:apx-orthogonal} establishes the role of the regularizer - in making the large columns of $U$ nearly orthogonal at approximate second order stationary points. Next we show that at second order stationary points, $UU^T \approx U_\star U_\star^T$.

\newpage

\begin{lemma} \label[lemma]{lemma:negativeeig}
Consider an $\epsilon$-approximate first-order stationary point of $\cL + \lambda \cR_\beta$, $U$. If $\| UU^T - U_\star U_\star^T \|_F \ge 8 \max \left\{ \epsilon^{2/3} k^{1/6}, \lambda \sqrt{\frac{2k}{\beta}} \right\}$, then the Hessian of $\cL + \lambda \cR_\beta$ at $U$ has a large negative eigenvalue,
\begin{align} \label{eq:181}
    \lambda_{\min} [\nabla^2 (\cL + \lambda \cR_\beta) (U)] \vect(Z) \le - \frac{1}{\sqrt{2k}} \| UU^T - U_\star U_\star^T \|_F.
\end{align}
In other words, at an $(\epsilon,\gamma)$-approximate second order stationary point of $\cL + \lambda \cR_\beta$, $U$,
\begin{align}
    \| UU^T - U_\star U_\star^T \|_F \lesssim \max \left\{ \epsilon^{2/3} k^{1/6}, \lambda \sqrt{\frac{2k}{\beta}}, \gamma \sqrt{k} \right\} \label{eq:182}
\end{align}
\end{lemma}
\begin{proof}
In the rest of this proof we expand $U^\star$ to a $\mathbb{R}^{d \times k}$ matrix by appending with $0$ columns. This allows us to define $R_\star \in \arg\min_{R : RR^T = R^T R = I} \| U - U_\star R \|_F^2$, breaking ties arbitrarily. Define $Z = U - U_\star R_\star$. With this choice of $Z$, we invoke 3 results from \cite{rong} (Lemmas 7, 40 and 41),
\begin{align} \label{eq:L40}
        \| Z Z^T \|_F^2 \le 2 \| UU^T - U_\star U_\star^T \|_F^2.
\end{align}
And,
\begin{align} \label{eq:L41}
    \| Z U^T \|_F^2 \le \frac{1}{2\sqrt{2}-1} \| UU^T - U_\star U_\star^T \|_F^2.
\end{align}
And finally a bound on the Hessian,
\begin{align}
\vect(Z)^T [\nabla^2 (\cL + \lambda \cR_\beta) (U)] \vect(Z) &\le 2\| ZZ^T \|_F^2 - 6\| UU^T - U_\star U_\star^T \|_F^2 + 4 \langle \nabla (\cL + \lambda \cR_\beta) (U), Z \rangle \nonumber\\
&\qquad +\lambda \left[\vect(Z)^T [\nabla^2 \cR_\beta (U)] \vect (Z) - 4 \lambda \langle \nabla \cR_\beta (U), Z \rangle \right]. \label{eq:presimple}
\end{align}
Plugging \cref{eq:L40} into \cref{eq:presimple}, and using the fact that $U$ is an $\epsilon$-approximate first order stationary point,
\begin{align}
    \vect(Z)^T [\nabla^2 (\cL + \lambda \cR_\beta) (U)] \vect(Z) &\le - 2 \| UU^T - U_\star U_\star^T \|_F^2 + 4 \epsilon \| Z \|_F \nonumber\\
    &\qquad +\lambda \left[\vect(Z)^T [\nabla^2 \cR_\beta (U)] \vect (Z) - 4 \langle \nabla \cR_\beta (U), Z \rangle \right]. \label{eq:21112}
\end{align}
In the next lemma we upper bound the last term.

\begin{lemma} \label[lemma]{lemma:26}
The contribution from the regularizer in the last term of \cref{eq:21112} can be bounded by,
\begin{align}
    \vect(Z)^T [\nabla^2 \cR_\beta (U)] \vect(Z) - 4 \langle \cR_\beta (U), \vect(Z) \rangle \le 8 \sqrt{\frac{2k}{\beta}} \| U U^T - U_\star U_\star^T \|_F.
\end{align}
\end{lemma}
The proof of this result is deferred to \Cref{app:B2}

Finally, combining \Cref{lemma:26} with \cref{eq:21112},
\begin{align}
    \vect(Z)^T [\nabla^2 (\cL + \lambda \cR_\beta) (U)] \vect(Z) &\le - 2 \| UU^T - U_\star U_\star^T \|_F^2 + 4 \epsilon \| Z \|_F \nonumber\\
    &\qquad + 8 \lambda \sqrt{\frac{2k}{\beta}} \| U U^T - U_\star U_\star^T \|_F \label{eq:1821}
\end{align}
Yet again using the inequality $\| Z \|_F^4 \le k \| Z Z^T \|_F^2\le 2 k \| UU^T - U_\star U_\star^T \|_F^2$, this results in the bound,
\begin{align}
    \vect(Z)^T [\nabla^2 (\cL + \lambda \cR_\beta) (U)] \vect(Z) &\le - 2 \| UU^T - U_\star U_\star^T \|_F^2 + 8 \epsilon k^{1/4} \| UU^T - U_\star U_\star^T \|_F^{1/2} \nonumber\\
    &\qquad + 8 \lambda \sqrt{\frac{2k}{\beta}} \| U U^T - U_\star U_\star^T \|_F  \label{eq:1822}
\end{align}
If $\| UU^T - U_\star U_\star^T \|_F \ge 8 \max \left\{ \epsilon^{2/3} k^{1/6}, \lambda \sqrt{\frac{2k}{\beta}} \right\}$, the RHS is upper bounded by $- \| UU^T - U_\star U_\star^T \|_F^2$. Finally, noting that $\| Z \|_F^4 \le k \| Z Z^T \|_F^2 \le 2k \| UU^T - U_\star U_\star^T \|_F^2$, we have that,
\begin{align}
    \lambda_{\min} [ \nabla^2 (\cL + \lambda \cR_\beta) (U) ] &\le \frac{\vect(Z)^T [\nabla^2 (\cL + \lambda \cR_\beta) (U)] \vect(Z)}{\| Z \|_F^2} \\
    &\le - \frac{1}{\sqrt{2k}} \| UU^T - U_\star U_\star^T \|_F \label{eq:1823}
\end{align}
\end{proof}

Let $V_r$ denote the matrix with columns as the non-zero eigenvectors of $U_\star U_\star^T$. In the next lemma, we show that for any stationary point $U$, all of its columns are almost entirely contained in the span of $V_r$, in that the angle between $U e_i$ and its projection onto $V_r^\perp$ is almost 90$^\circ$. In other words, the columns of $U$ approximately lie in the correct subspace.

\begin{lemma} \label[lemma]{lemma:ZFbound}
Consider an $\epsilon$-approximate first order stationary point of $\cL + \lambda \cR_\beta$, $U$, satisfying $\| U \|_{\text{op}} \le 3$. Let $V_r$ denote the matrix with columns as the non-zero eigenvectors of $U_\star U_\star^T$. Then, assuming $\beta < 1$,
\begin{align}
    \| V_r^\perp (V_r^\perp)^T U \|_F \le 3\epsilon / \lambda.
\end{align}
\end{lemma}
\begin{proof}
By \cref{eq:plass1} for $Z = V_r^\perp (V_r^\perp)^T U$, by the approximate stationarity of $U$,
\begin{align} \label{eq:plass2}
    2\langle UU^T - U_\star U_\star^T, U Z^T + ZU^T \rangle + \lambda \tr (U D(U) Z^T) \le \epsilon \| Z \|_F
\end{align}
The LHS is lower bounded by,
\begin{align}
    &4\tr (UU^T V_r^\perp (V_r^\perp)^T U U^T ) + \lambda \tr (U D(U) U^T V_r^\perp (V_r^\perp)^T ) \label{eq:cp1} \\
    &= 4\| V_r^\perp (V_r^\perp)^T U U^T  \|_F^2 + \lambda \sum_{i=1}^k (D(U))_{ii} \| V_r^\perp (V_r^\perp)^T U e_i \|_2^2 \\
    &\overset{(i)}{\ge} \frac{\lambda}{3} \sum_{i=1}^k \| V_r^\perp (V_r^\perp)^T U e_i \|_2^2 \\
    &= \frac{\lambda}{3} \| Z \|_F^2. \label{eq:cp2}
\end{align}
where $(i)$ uses the fact that since $\| U \|_{\text{op}} \le 3$ and $\beta < 1$, $(D(U))_{ii} = \frac{2 \beta + \| U e_i \|_2^2}{(\| Ue_i \|_2^2 + \beta )^{3/2}} \ge \frac{2 \beta + 9}{(9 + \beta)^{3/2}} \ge 1/3$.
Combining with \cref{eq:plass2} completes the proof.
\end{proof}

\begin{lemma} \label[lemma]{lemma:proj-similar}
Consider an $\epsilon$-approximate first order stationary point $U$ of $\cL + \lambda \cR_\beta$. Let $V_r$ be as defined in \Cref{lemma:ZFbound}. Let $S$ denote the set of columns $i \in [k]$ such that $\| U e_i \|_2 \ge 2\sqrt{\beta}$. Then, for any $i \in S$,
    \begin{align} \label{eq:192}
        \frac{\| V_r^\perp (V_r^\perp)^T U e_i \|_2}{\| U e_i \|_2} \le \frac{2\epsilon}{\lambda \beta^{1/4}}.
    \end{align}
    Note that the LHS is the cosine of the angle between $U e_i$ and its projection onto $V_r^\perp$, (or the sine of the angle between $U e_i$ and its projection onto $V_r$). Likewise, for the remaining columns,
    \begin{align} \label{eq:192-0}
        \sum_{i \in [k] \setminus S} \| V_r^\perp (V_r^\perp)^T U e_i \|_2^2 \le \frac{2\epsilon^2 \sqrt{\beta}}{\lambda^2}.
    \end{align}
\end{lemma}
\begin{proof}
At an $\epsilon$-approximate first order stationary point the gradient is upper bounded in $L_2$-norm by $\epsilon$. From the gradient and Hessian computations in \Cref{lemma:L-grad-hessian} and \Cref{lemma:R-grad-hessian},
\begin{align} \label{eq:plass1}
    2\langle UU^T - U_\star U_\star^T, U Z^T + ZU^T \rangle + \lambda \tr ( D(U) Z^T U) \le \epsilon \| Z \|_F
\end{align}
Choosing $Z = V_r^\perp (V_r^\perp)^T U$, and noting that $U_\star$ has rank $r$ and is orthogonal to $V_r^\perp$, the LHS simplifies as,
\begin{align}
    \epsilon \| Z \|_F
    &\ge 4\tr (UU^T V_r^\perp (V_r^\perp)^T U U^T ) + \lambda \tr (U D(U) U^T V_r^\perp (V_r^\perp)^T ) \label{eq:00000}\\
    &\ge \lambda \sum_{i=1}^{k} D(U)_{ii} \| V_r^\perp (V_r^\perp)^T U e_i \|_2^2 \label{eq:0002922231} \\
    &\ge \lambda \sum_{i \in S} \frac{\| V_r^\perp (V_r^\perp)^T U e_i \|_2^2}{\| U e_i \|_2}, \label{eq:000001}
\end{align}
where the last inequality uses the fact that for any column $i$ such that $\| U e_i \|_2 \ge 2\sqrt{\beta}$, since we have that $D(U)_{ii} \ge \frac{1}{\| U e_i \|_2}$ by \Cref{lemma:DUbound}. Putting everything together, we get that,
\begin{align}
    \lambda \sum_{i \in S} \| U e_i \|_2 \cdot \frac{\| V_r^\perp (V_r^\perp)^T U e_i \|_2^2}{\| U e_i \|_2^2} \le \epsilon \| Z \|_F \label{eq:00001}
\end{align}
Since $\| U e_i \|_2 \ge 2\sqrt{\beta}$, rearranging the terms around and finally upper bounding $\| Z \|_F$ using \Cref{lemma:ZFbound} completes the proof of \cref{eq:192}. To prove \cref{eq:192-0}, notice from \cref{eq:0002922231} that,
\begin{align}
    \epsilon \| Z \|_F &\ge \lambda \sum_{i \in [k] \setminus S} D(U)_{ii} \| (V_r^\perp) (V_r^\perp)^T U e_i \|_2^2 \\
    &\ge \frac{\lambda}{2\sqrt{\beta}} \sum_{i \in [k] \setminus S} \| (V_r^\perp) (V_r^\perp)^T U e_i \|_2^2
\end{align}
which uses the fact that when $\| U e_i \|_2 \le 2 \sqrt{\beta}$, then $D(U)_{ii}= \frac{\| U e_i \|_2^2 + 2\beta}{(\| U e_i \|_2^2 + \beta)^{3/2}} \ge \frac{1}{2 \beta^{1/2}}$. Rearranging and substituting the bound on $\| Z \|_F$ from \Cref{lemma:ZFbound} completes the proof of \cref{eq:192-0}.
\end{proof}

\begin{lemma} \label{lemma:proj-0}
Consider any $\epsilon$-approximate first-order stationary point of $\cL + \lambda \cR_\beta$, $U$, and let $V_r$ be as defined in \Cref{lemma:ZFbound}. Under the assumption that $\epsilon \le \lambda/2$, for any column $i \in [k]$ such that $\| V_r V_r^T U e_i \|_2 \le \| V_r^\perp (V_r^\perp)^T U e_i \|_2 $, $\| U e_i \|_2 \le 2 \sqrt{\beta}$.
\end{lemma}
\begin{proof}
Consider any $i$ such that $\| V_r V_r^T U e_i \|_2 \le \| V_r^\perp (V_r^\perp)^T U e_i \|_2$. Therefore, from \cref{eq:plass1}, choosing $Z = V_r^\perp (V_r^\perp)^T U e_i e_i^T$,
\begin{align}
    &\epsilon \| V_r^\perp (V_r^\perp)^T U e_i \|_2 \nonumber\\
    &\ge 4\tr (UU^T V_r^\perp (V_r^\perp)^T U e_i e_i^T U^T ) + \lambda \tr (U D(U) e_i e_i^T U^T V_r^\perp (V_r^\perp)^T ) \nonumber\\
    &= 4\| U^T V_r^\perp (V_r^\perp)^T U e_i \|_2^2 + 4 \tr (e_i^T U^T V_r V_r^T UU^T V_r^\perp (V_r^\perp)^T U e_i ) + \lambda (D(U))_{ii} \| V_r^\perp (V_r^\perp)^T U e_i \|_2^2 \nonumber\\
    &\ge 4\| U^T V_r^\perp (V_r^\perp)^T U e_i \|_2^2 - 4 \| U^T V_r V_r^T U e_i \|_2 \| U^T V_r^\perp (V_r^\perp)^T U e_i \|_2 + \lambda (D(U))_{ii} \| V_r^\perp (V_r^\perp)^T U e_i \|_2^2 \nonumber\\
    &\ge \lambda (D(U))_{ii} \| V_r^\perp (V_r^\perp)^T U e_i \|_2, \label{eq:203}
\end{align}
where the last inequality uses the assumption that $\| V_r V_r^T U e_i \|_2 \le \| V_r^\perp (V_r^\perp)^T U e_i \|_2$. Finally, observe the following bound on $(D(U))_{ii}$ computed using \Cref{lemma:DUbound}: for any column $i$ such that $\| U e_i \|_2 \ge 2 \sqrt{\beta}$,
\begin{align}
    (D(U))_{ii} \ge \frac{1}{\| U e_i \|_2} = \frac{1}{\sqrt{\| V_r^\perp (V_r^\perp)^T U e_i \|_2^2 + \| V_r V_r^T U e_i \|_2^2}} \ge \frac{1}{2 \| V_r^\perp (V_r^\perp)^T U e_i \|_2}.
\end{align}
Plugging this into \cref{eq:203}, we get the inequality,
\begin{align}
    \epsilon \| V_r^\perp (V_r^\perp)^T U e_i \|_2  &\ge \frac{\lambda}{2} \| V_r^\perp (V_r^\perp)^T U e_i \|_2.
\end{align}
When $\epsilon \le \frac{\lambda}{2}$, the only solution to inequality \cref{eq:203} is $\| V_r^\perp (V_r^\perp)^T Ue_i \|_2 = 0$. By the condition $\| V_r (V_r)^T U e_i \|_2 \le \| V_r^\perp (V_r^\perp)^T U e_i \|_2$, this implies that $\| U e_i \|_2 = 0$. This contradicts the initial assumption that $\| U e_i \|_2 \ge 2 \sqrt{\beta}$, thus concluding the proof of the second part of \Cref{lemma:proj-0}.
\end{proof}

\begin{theorem} \label{theorem:prune-prebound}
Consider any $(\epsilon,\gamma)$-approximate second-order stationary point of $\cL + \lambda \cR_\beta$ denoted $U$. Suppose $\| U \|_{\text{op}} \le 3$. Construct $U_{\text{prune}}$ by following the pruning condition,
\begin{align} \label{eq:pruning_condition}
    \| U e_i \|_2 \le 2\sqrt{\beta} \implies U_{\text{prune}} e_i \gets 0
\end{align}
Suppose $\epsilon \le c \lambda \beta^{1/4} / r$ and $\epsilon + \gamma \le c\lambda /r^2$ for a sufficiently small constant $c > 0$. Then,
\begin{enumerate}
    \item $U_{\text{prune}}$ has at most $r$ non-zero columns.
    \item Furthermore,
\begin{align}
    \| U_{\text{prune}} U_{\text{prune}}^T - U_\star U_\star^T \|_F \lesssim \gamma \sqrt{k} + \lambda \sqrt{\frac{k}{\beta}} + \epsilon^{2/3} k^{1/6} + r \beta + \frac{2 \epsilon^2 \sqrt{\beta}}{\lambda^2}.
    \end{align}
\end{enumerate}
\end{theorem}
\begin{proof}
Recall that the columns of $U$ having $L_2$ norm at most $2 \sqrt{\beta}$ are set to $0$ in $U_{\text{prune}}$. Since $U$ is an $(\epsilon,\gamma)$-approximate second-order stationary point, the eigenvalues of the Hessian of $\cL + \lambda \cR_\beta$ at the point $U$ are all at least $-\gamma$. 

In \Cref{lemma:apx-orthogonal}, by violation of the condition \cref{eq:corrlb}, for every pair of column indices $i \ne j \in S$ where $S = \{ i \in [k] : \| U e_i \|_2 > 2 \sqrt{\beta} \}$,
\begin{align}
    |C_{ij}| = \frac{|\langle U e_i, U e_j \rangle|}{\| Z_{ij} \|_2} \lesssim \sqrt{ \frac{\epsilon + 3\gamma}{\lambda}} \min \{ \| Ue_i \|_2 , \| Ue_j \|_2 \}.
\end{align}
where we use the assumption that the spectral norm $\| U \|_{\text{op}} \le 3$, and therefore $\| Ue_i \|_2 \le 3$ for all $i \in [k]$. Note that $\| Z_{ij} \|_2 = \sqrt{\| U e_i \|_2^2 + \| Ue_j \|_2^2} \le \sqrt{2} \max \left\{ \| U e_i \|_2, \| U e_j \|_2 \right\}$, and therefore,
\begin{align}
    |\langle U e_i, Ue_j \rangle| \lesssim \sqrt{\frac{\epsilon + \gamma}{\lambda}} \| Ue_i \|_2 \| Ue_j \|_2
\end{align}
In other words, for each $i \ne j \in S$,
\begin{align}
    | \cos \theta_{ij} | \lesssim \sqrt{\frac{\epsilon + \gamma}{\lambda}} \implies \left| \theta_{ij} - \frac{\pi}{2} \right| \lesssim \sqrt{\frac{\epsilon + \gamma}{\lambda}} \label{eq:thetaij}
\end{align}
Where $\theta_{ij}$ is the angle between $U e_i$ and $U e_j$ and the implication follows by assuming that $\epsilon, \gamma \le c\lambda$ for a sufficiently small $c > 0$.

At any $\epsilon$-approximate first order stationary point $U$, by \Cref{lemma:proj-similar}, for any $i$ such that $\| U e_i \|_2 \ge 2\sqrt{\beta}$ the angle between $U e_i$ and its projection onto $V_r$ is intuitively small. By \cref{eq:thetaij} we expect $U e_i$ and $U e_j$ to have an angle close to 90$^\circ$ between them. Therefore, we expect the projections of $U e_i$ and $U e_j$ onto $V_r$ to also have an angle close to 90$^\circ$ between them. Specifically, given vectors $v_1,v_2,v_3$ and $v_4$ such that $|\angle v_1, v_2| \le \varepsilon_1$, $|\angle v_2, v_3 - 90^\circ| \le \varepsilon_2$ and $|\angle v_3, v_4 | \le \varepsilon_3$, then $|\angle v_1, v_4 - 90^\circ| \le \varepsilon_1 + \varepsilon_2 + \varepsilon_3$. This means that for any columns $Ue_i$ and $U e_j$ such that $\| U e_i \|_2$ and $\| U e_j \|_2 \ge 2\sqrt{\beta}$,
\begin{align}
    \left| \theta^{\text{proj}}_{ij} - \frac{\pi}{2} \right| \lesssim \sqrt{\frac{\epsilon + \gamma}{\lambda}} + \frac{\epsilon}{\lambda \beta^{1/4}} \le \frac{c_1}{r} \label{eq:orthproj}
\end{align}
for some small absolute constant $0 \le c_1 \le 1/3$ and where $\theta_{ij}^{\text{proj}}$ is the angle between $V_r V_r^T U e_i$ and $V_r V_r^T U e_j$. This assumes that $\epsilon \le c\lambda \beta^{1/4}/r$ and $\epsilon + \gamma \le c \lambda /r^2$ for a sufficiently small constant $c > 0$. Since $\sin(x) \le x$ for $x > 0$, this means that,
\begin{align} \label{eq:orthproj:vectors}
    \forall i \ne j \in S,\ \frac{|\langle V_r V_r^T U e_i, V_r V_r^T U e_j \rangle|}{\| V_r V_r^T U e_i \|_2 \| V_r V_r^T U e_j \|_2} \le \frac{c_1}{r}.
\end{align}
\cref{eq:orthproj:vectors} implies that the projections of the large columns of $U$ onto $V_r$ are approximately at 90$^\circ$ to each other. However, intuitively, since $V_r$ is an $r$ dimensional space, this means that at most $r$ of these projections can be large in norm. In fact it turns out that since the columns are sufficiently orthogonal to one another, \textit{exactly} $r$ columns have a non-zero projection onto $V_r$.

\begin{lemma} \label[lemma]{lemma:atmostr}
At most $r$ columns of $U$ can have a non-zero projection onto $V_r$.
\end{lemma}
The proof of this result is deferred to \Cref{app:B2}.

What is the implication of \Cref{lemma:atmostr}? At most $r$ columns of $U$ (and therefore $U_{\text{prune}}$) have non-zero projections onto $V_r$. Any of remaining columns of $U$, say $i$, therefore satisfies the condition $0 = \| V_r V_r^T U e_i \|_2 \le \| V_r^\perp (V_r^\perp)^T U e_i \|_2$. By implication of \Cref{lemma:proj-0}, the remaining columns of $U$ have $L_2$ norm at most $2 \sqrt{\beta}$ assuming that $\epsilon \le \lambda/2$. Thus these columns are pruned away in $U_{\text{prune}}$. Overall, this implies that at most $r$ columns of $U_{\text{prune}}$ are non-zero.

For the analysis of $\| U_{\text{prune}} U_{\text{prune}}^T - U_\star U_\star^T \|_F$, note that, the columns of $U$ with $L_2$ norm smaller than $2\sqrt{\beta}$ are set to $0$ in $U_{\text{prune}}$. This implies that,
\begin{align}
    &\hspace{-2em}\| U_{\text{prune}} U_{\text{prune}}^T - U_\star U_\star^T \|_F - \| UU^T - U_\star U_\star^T \|_F \nonumber\\
    &\le \| U_{\text{prune}} U_{\text{prune}}^T - U U^T \|_F \label{eq:komega1}\\
    &\le \left\| \sum_{i \in [k]} U e_i (U e_i)^T \mathbb{I} ( \| U e_i \|_2 \le 2\sqrt{\beta}) \right\|_F \\
    &\le \sum_{i \in [k]} \left\| U e_i (U e_i)^T \mathbb{I} ( \| U e_i \|_2 \le 2\sqrt{\beta}) \right\|_F \\
    &\le \sum_{i \in [k]} \| U e_i \|_2^2 \mathbb{I} ( \| U e_i \|_2 \le 2\sqrt{\beta}; \| V_r V_r^T U e_i \|_2 > 0) + \sum_{i \in [k]} \| V_r V_r^T U e_i \|_2^2 \mathbb{I} ( \| U e_i \|_2 \le 2\sqrt{\beta}) \\
    &\overset{(i)}{\le} r \beta + \sum_{i \in [k]} \| V_r V_r^T U e_i \|_2^2 \mathbb{I} ( \| U e_i \|_2 \le 2\sqrt{\beta}) \\
    &\overset{(ii)}{\le} r \beta + \frac{2 \epsilon^2 \sqrt{\beta}}{\lambda^2} \label{eq:komega2}
\end{align}
where $(i)$ uses the fact that there are at most $r$ columns of $U$ having non-zero projections onto $V_r$, and $(ii)$ follows from \cref{eq:192-0} of \Cref{lemma:proj-similar}. Combining this inequality with the guarantee on $\| UU^T - U_\star U_\star^T \|_F$ in \Cref{lemma:negativeeig} completes the proof.
\end{proof}

\subsection{Proof of \Cref{theorem:main-population}}
\label{app:B2}

\begin{proof}
From \Cref{theorem:prune-prebound}, we have that,
\begin{align} \label{eq:91103332}
    \| U_{\text{prune}} U_{\text{prune}}^T - U_\star U_\star^T \|_F \lesssim \gamma \sqrt{k} + \lambda \sqrt{\frac{k}{\beta}} + \epsilon^{2/3} k^{1/6} + r \beta + \frac{2 \epsilon^2 \sqrt{\beta}}{\lambda^2}.
\end{align}
Recall that the optimization and smoothing parameters are chosen as,
\begin{align}
    \beta &= c_\beta \frac{(\sigma_r^\star)^2}{r} \\
    \lambda &= c_\lambda \frac{(\sigma_r^\star)^3}{\sqrt{kr}} \\
    \gamma &\le c_\gamma \frac{(\sigma^\star_r)^3}{\sqrt{k} r^{5/2}} \\
    \epsilon &\le c_\epsilon \frac{(\sigma_r^\star)^{7/2}}{\sqrt{k} r^{5/2}}.
\end{align}
For sufficiently small absolute constants $c_\beta, c_\lambda, c_\gamma, c_\epsilon > 0$. Under these choices, it is easily verified from \cref{eq:91103332} that the conditions,
\begin{align}
    \epsilon \le c \frac{\lambda \beta^{1/4}}{r}; \qquad \epsilon + \gamma \le \frac{c \lambda}{r^2}.
\end{align}
which are conditions required by \Cref{theorem:prune-prebound}. The bound in \cref{eq:91103332} results in,
\begin{align} \label{eq:sigmarstar}
    \| U_{\text{prune}} U_{\text{prune}}^T - U_\star U_\star^T \|_F \le \frac{1}{2} (\sigma_r^\star)^2.
\end{align}
Note also that $U_{\text{prune}}$ has at most $r$ non-zero columns. However, by \cref{eq:sigmarstar}, this means that $U_{\text{prune}}$ has exactly $r$ columns, since if it had $r-1$ non-zero columns or fewer, the error of the best solution must be at least $(\sigma_r^\star)^2$.
\end{proof}

\subsection{Missing proofs in \Cref{app:2}}

\subsubsection{Proof of \Cref{lemma:WLLW}}
\begin{proof}
By definition of $W$ and $L$, $\tr ((W L - L W ) \cdot Z^T U)$, can be simplified to
\begin{align}
    \tr ((W L - L W ) \cdot Z^T U) &= \Delta \tr ( (e_2 e_2^T - e_1 e_1^T) U^TU) \\
    &= \Delta (\| U e_2 \|_2^2 - \| U e_1 \|_2^2) \label{eq:middleterm-3}
\end{align}
Note that $D(U)_{11} - D(U)_{22} = \Delta$ and $|\Delta| \le \frac{\epsilon}{\lambda}$, so one can expect $\| Ue_1 \|_2$ and $\| U e_2 \|_2$ to be close to each other. We bound $| \| U e_1 \|_2^2 - \| U e_2 \|_2^2 |$ in two ways depending on the relative values of $\| U e_1 \|_2$ and $\| U e_2 \|_2$.

\paragraph{Case 1: Column norms are similar.} $\max_{i \in \{1,2\}} \| U e_i \|_2 \le 2 \min_{i \in \{1,2\}} \| U e_i \|_2$.

\noindent By definition of $\Delta$,
\begin{align}
    |\Delta| &= \left| \frac{\| U e_1 \|_2^2 + 2 \beta}{(\| U e_1 \|_2^2 + \beta)^{3/2}} - \frac{\| U e_2 \|_2^2 + 2 \beta}{(\| U e_2 \|_2^2 + \beta)^{3/2}} \right| \\
    &= \left| \int_{\| U e_2 \|_2^2}^{\| U e_1 \|_2^2} \frac{3/2 \cdot x + 4 \beta}{(x + \beta)^{5/2}} \mathrm{d} x \right| \\
    &\ge \min_{i \in \{ 1,2\} }\frac{3/2 \| U e_i \|_2^2 + 4 \beta}{(\| U e_i \|_2^2 + \beta)^{5/2}} | \| U e_1 \|_2^2 - \| U e_2 \|_2^2 |. \label{eq:middleterm-4}
\end{align}
Furthermore, note that for any $i \in \{1,2\}$,
\begin{align}
    \frac{(\| Ue_i \|_2^2 + \beta)^{5/2}}{\frac{3}{2} \| U e_i \|_2^2 + 4 \beta} &\le \left( \frac{3}{2} \| U e_i \|_2^2 + 4\beta \right)^{3/2} \\
    &\overset{(i)}{\le} \sqrt{2} \left( \frac{3}{2} \| U e_i \|_2^2 \right)^{3/2} + \sqrt{2} \left( 4\beta \right)^{3/2} \\
    &\le 3 \| U e_i \|_2^3 + 12 \beta^{3/2} \\
    &\le 5 \| U e_i \|_2^3 \label{eq:middleterm-5}
\end{align}
where $(i)$ uses the convexity of $(\cdot)^{3/2}$ for positive arguments and applies Jensen's inequality, and the last inequality uses the fact that $\| U e_1 \|_2, \| U e_2 \|_2 \ge 2\sqrt{\beta}$. Overall, combining \cref{eq:middleterm-3} with \cref{eq:middleterm-4} and \cref{eq:middleterm-5} results in the bound,
\begin{align}
    \tr ((WL-LW) \cdot Z^T U) &\le 5 \Delta^2 \max_{i \in \{ 1,2 \}} \| U e_i \|_2^3 \\
    &\le 40 \Delta^2 \min_{i \in \{ 1,2 \}} \| U e_i \|_2^3 \label{eq:WLLW-1}
\end{align}
where the last inequality uses the assumption that $\max_{i \in \{ 1,2\}} \| U e_i \|_2 \le 2 \min_{i \in \{ 1,2\}} \| U e_i \|_2$.

\paragraph{Case 2: Column norms are separated.} $\max_{i \in \{1,2\}} \| U e_i \|_2 \ge 2 \min_{i \in \{1,2\}} \| U e_i \|_2$.

\noindent WLOG suppose $\| U e_1 \|_2 \ge \| U e_2\|_2$ which implies that $D(U)_{22} \ge D(U)_{11}$. The above assumption implies that $\| U e_1 \|_2 \ge 2 \| U e_2 \|_2$. Then,
\begin{align}
    D(U)_{11} &= \frac{\| U e_1 \|_2^2 + 2 \beta}{(\| U e_1 \|_2^2 + \beta)^{3/2}} \\
    &\overset{(i)}{\le} \frac{4\| U e_2 \|_2^2 + 2 \beta}{(4\| U e_2 \|_2^2 + \beta)^{3/2}} \\
    &\overset{(ii)}{\le} \frac{4}{3^{3/2}} \frac{\| U e_2 \|_2^2 + 2 \beta}{(\| U e_2 \|_2^2 + \beta)^{3/2}} \\
    &\le \frac{4}{5} D(U)_{22}
\end{align}
where $(i)$ uses the fact that $\frac{x+2\beta}{(x + \beta)^{3/2}}$ is a decreasing function in $x$ and $(ii)$ uses the fact that $\min \{ \| U e_1\|_2, \| U e_2 \|_2 \} \ge 2 \sqrt{\beta}$, and therefore, $4 \| U e_2 \|_2^2 + \beta \ge 3(\| U e_2 \|_2^2 + \beta)$. Therefore,
\begin{align}
    |D (U)_{11} - D(U)_{22}| \ge \frac{1}{5} D(U)_{22} &= \frac{1}{5} \frac{\| U e_2 \|_2^2 + 2 \beta}{(\| U e_2 \|_2^2 + \beta)^{3/2}} \\
    &\overset{(i)}{\ge} \frac{1}{5} \frac{\| U e_2 \|_2^2}{\sqrt{2} \| U e_2 \|_2^{3} + \sqrt{2} \beta^{3/2}} \\
    &\overset{(ii)}{\ge} \frac{1}{10} \frac{1}{\| U e_2 \|_2} \\
    &\ge \frac{1}{10} \frac{1}{\| U e_2 \|_2 \| U e_1 \|_2^2} | \| U e_1 \|_2^2 - \| U e_2 \|_2^2 |
\end{align}
where $(i)$ uses the convexity of $(\cdot)^{3/2}$ for positive arguments and an application of Jensen's inequality, while $(ii)$ uses the fact that $\min \{ \| U e_1 \|_2, \| U e_2 \|_2 \} \ge 2 \sqrt{\beta}$. Since $|D(U)_{11} - D(U)_{22}| \le |\Delta|$, this results in an upper bound on $|\| U e_1 \|_2^2 - \| U e_2 \|_2^2|$. Combining this with \cref{eq:middleterm-3} results in,
\begin{align} \label{eq:WLLW-2}
    \tr ((WL-LW) \cdot Z^T U) &\le 10 \Delta^2 \min_{i \in \{ 1,2 \}} \| U e_i \|_2 \max_{i \in \{ 1,2 \}} \| U e_i \|_2^2.
\end{align}
Combining \cref{eq:WLLW-1,eq:WLLW-2} completes the proof of the lemma.
\end{proof}

\subsubsection{Proof of \Cref{lemma:26}}
\begin{proof}
By definition,
\begin{align}
    &\vect(Z)^T [\nabla^2 \cR_\beta (U)] \vect(Z) - 4 \langle \cR_\beta (U), \vect(Z) \rangle \\
    &\le \tr (D(U) Z^T Z ) - \sum_{i=1}^k (G(U))_{ii} \langle U e_i, Z e_i \rangle^2 - 4 \tr ( D(U) U^T Z) \\
    &\le \tr (D(U) Z^T Z ) - 4 \tr ( D(U) U^T Z) \\
    &\le \frac{2}{\sqrt{\beta}} \| Z \|_F^2 + 4 \| D(U) \|_F \| U^T Z \|_F \\
    &\overset{(i)}{\le} \frac{2}{\sqrt{\beta}} \| Z \|_F^2 + 5 \sqrt{\frac{k}{\beta}} \| U U^T - U_\star U_\star^T \|_F \\
    &\le 8\sqrt{\frac{2k}{\beta}} \| U U^T - U_\star U_\star^T \|_F
\end{align}
where $(i)$ follows from the fact that for any $i \in [k]$, $| (D(U))_{ii} | \le \max_{x \ge 0} \frac{x+2\beta}{(x + \beta)^{3/2}} = \frac{2}{\sqrt{\beta}}$, $(ii)$ uses \cref{eq:L41} and the fact that $\| Z \|_F^4 = \| Z^T \|_F^4 \le \textsf{rank} (Z^T) \| Z Z^T \|_F^2\le 2 k \| UU^T - U_\star U_\star^T \|_F^2$ by \cref{eq:L40}.
\end{proof}

\subsubsection{Proof of \Cref{lemma:atmostr}}
\begin{proof}
Let $S_{\max}$ denote the set of $r$ column indices $i \in [k]$ such that $\| V_r V_r^T U e_i \|_2$ are maximum and non-zero. If fewer than $r$ columns of $U$ have non-zero projections onto $V_r$ (i.e. fewer than $r$ candidates for $S_{\max}$ exist), then all the remaining columns of $U$ must have zero projections onto $V_r$ and we are done. Therefore, henceforth, we will assume that $|S_{\max}| = r$. Let $V_{\max} = \textsf{span} (\{ V_r V_r^T U e_i : i \in S_{\max} \})$, a subspace contained in $V_r$.

Define $M$ as the matrix $M = \begin{bmatrix} \frac{V_r V_r^T Ue_i}{\| V_r V_r^T Ue_i \|_2} : i \in S_{\max} \end{bmatrix}$. Let $V$ be the set of non-zero eigenvectors of $MM^T$. Note that the matrix $M^TM \in \mathbb{R}^{|S_{\max}| \times |S_{\max}|}$ has $1$'s on the diagonals, and has its off-diagonal entries upper bounded in absolute value by $c_1/r$ by \cref{eq:orthproj:vectors}. Therefore, by the Gersgorin circle theorem, all eigenvalues of $M^TM$ and consequently all non-zero eigenvalues of $MM^T$ lie in the range $[1 - c_1, 1+ c_1]$. This means that $M \in \mathbb{R}^{d \times r}$ is full rank, and its columns are linearly independent.

This implies that $\textsf{span}(V) = \textsf{span} (V_r)$ and by the orthogonality condition, $VV^T = V_r V_r^T$. Therefore,
\begin{align}
    \| V_r V_r^T U e_i \|_2 &= \| V V^T U e_i \|_2 \\
    &\le \frac{1}{\lambda_{\min} (M)} \| M^T U e_i \|_2 \\
    &= \frac{1}{\lambda_{\min} (M)} \sqrt{\sum_{j \in S_{\max}} \frac{\langle V_r V_r^T U e_j, V_r V_r^T U e_i \rangle^2}{\| V_r V_r^T Ue_j \|_2^2}} \\
    &\overset{(i)}{\le} \frac{1}{\lambda_{\min} (M)} \sqrt{\sum_{j \in S_{\max}} \frac{c_1^2}{r^2} \| V_r V_r^T U e_i \|_2^2}, \label{eq:0091211}
\end{align}
where $(i)$ follows from \cref{eq:orthproj}.
Consequently by \cref{eq:0091211}, we have,
\begin{align} \label{eq:210}
    \| V_r V_r^T U e_i \|_2 \le \frac{c_1}{1 - c_1} \| V_r V_r^T U e_i \|_2.
\end{align}
For $c_1 \le 1/3$ this means that $\| V_r V_r^T U e_i \|_2 = 0$ for every $i \in [k] \setminus S_{\max}$. This proves the claim.
\end{proof}

\newpage

\section{Finite sample guarantees: Proof of \Cref{theorem:main_finite}}
\label{app:finite}

In this section, we provide guarantees on the approximate second order stationary points of the regularized loss $f_{\text{emp}}$ (\cref{eq:fregemp}) when the dataset is finite in size, and satisfies the RIP condition. We provide a formal proof of \Cref{theorem:main_finite}.

\noindent For completeness, in the finite sample setting, the empirical loss is defined as,
\begin{align} \label{eq:empirical_loss}
    \cL_{\text{emp}} (U) = \frac{1}{n} \sum_{i=1}^n (\langle U U^T - U_\star U_\star^T, A_i \rangle + \varepsilon_i)^2
\end{align}
where $\varepsilon_i \overset{\text{i.i.d.}}{\sim} \mathcal{N} (0,\sigma^2)$ is the measurement noise. Without loss of generality we assume that the measurement matrices $\{ A_i \}_{i=1}^n$ are symmetric, since the empirical loss $\cL_{\text{emp}}$ is unchanged by the symmetrization $A_i \to \frac{A_i + A_i^T}{2}$. Note that the loss can be expanded as,
\begin{align}
    \| UU^T - U_\star U_\star^T \|_{\mathcal{H}}^2 + 2 \left\langle \frac{1}{n} \sum_{i=1}^n \varepsilon_i A_i, UU^T - U_\star U_\star^T \right\rangle + \frac{1}{n} \sum_{i=1}^n \varepsilon_i^2 \label{eq:lossexpanded}
\end{align}
where $\langle X, Y \rangle_{\mathcal{H}} = \frac{1}{m} \sum_{i=1}^m \langle X, A_i \rangle \langle Y, A_i \rangle$ and $\| X \|_{\mathcal{H}}^2 = \langle X, X \rangle_{\mathcal{H}}$. \\

\noindent We will assume that the measurement matrices $\{ A_i \}_{i=1}^n$ satisfy the RIP condition. At a high level this condition guarantees that $\langle \cdot, \cdot \rangle_{\mathcal{H}} \approx \langle \cdot, \cdot \rangle$ when the arguments are low rank matrices.

\begin{definition}[RIP condition]
A set of linear measurement matrices $A_1, \ldots, A_m$ in $\mathbb{R}^{d \times d}$ satisfies the $(k, \delta)$-restricted isometry property (RIP) if for any $d \times d$ matrix $X$ with rank at most $k$,
\begin{align}
    (1-\delta)\|X\|_F^2 \leq \frac{1}{m} \sum_{i=1}^m\left\langle A_i, X\right\rangle^2 \leq(1+\delta)\|X\|_F^2 
\end{align}
\end{definition}

The crucial property we use in this section is that the RIP condition on the measurements guarantees that $\langle X, Y \rangle_{\mathcal{H}} \approx \langle X,Y \rangle$. when $X$ and $Y$ are low rank matrices.

\begin{lemma} \label{lemma:29}
Let $\left\{A_i\right\}_{i=1}^m$ be a family of matrices in $\mathbb{R}^{d \times d}$ that satisfy $(k, \delta)$ RIP. Then for any pair of matrices $X, Y \in \mathbb{R}^{d \times d}$ with rank at most $k$, we have:
\begin{align}
    \left| \langle X, Y \rangle_{\mathcal{H}} - \langle X, Y\rangle\right| \leq \delta\|X\|_F\|Y\|_F
\end{align}
\end{lemma}

\begin{lemma}
Let $\{ A_i \}_{i=1}^n$ be a set of matrices which satisfy the $(2k,\delta)$-RIP for some $\delta \le 1/10$. Let $\varepsilon_i \overset{\text{i.i.d.}}{\sim} \mathcal{N} (0, \sigma^2)$. Consider any $\eta \in (0,1)$. Then,
\begin{align}
\mathbb{P}\left(\left\|\frac{1}{n} \sum_i^n A_i \epsilon_i\right\|_{\text{op}} \ge 4 \sigma \sqrt{\frac{d \log(d/\eta)}{n}}\right) \leq \eta.
\end{align}
\end{lemma}
\begin{proof}
Note that $\{ A_i \}_{i=1}^n$ satisfies the $(2k,\delta)$-RIP for $\delta \le 1/10$. That is, for any rank $\le 2k$ matrix $X$,
\begin{align}
    \frac{9}{10} \| X \|_F^2 \le \frac{1}{n} \sum_{i=1}^n \langle A_i, X \rangle^2 \le \frac{11}{10} \| X \|_F^2.
\end{align}
Choosing $X = v u^T$ for any pair of vectors $u,v$, we get,
\begin{align}
    \forall u,v \in \mathbb{R}^d,\quad \frac{1}{n} \sum_{i=1}^n(u^T A_i v)^2 \le \frac{11}{10} \| u \|_2^2 \| v \|_2^2.
\end{align}
This implies that for each $i \in [k]$, $\| A_i \|_{\text{op}} \le \sqrt{2n}$. Furthermore, plugging in $v = e_1,\cdots,e_d$ and summing,
\begin{align}
    \frac{1}{n} \sum_{i=1}^n \| A_i^T u \|_2^2 \le \frac{11}{10} d \| u \|_2^2
\end{align}
This implies that $\left\| \frac{1}{n} \sum_{i=1}^n A_i A_i^T \right\|_{\text{op}} = \lambda_{\max} \left( \frac{1}{n} \sum_{i=1}^n A_i A_i^T \right) \le 2 d$. Hence, we obtain,
\begin{align}
    \left\| \sum_{i=1}^n \mathbb{E}\left( (\epsilon_i A_i) ( \epsilon_i A_i )^T \right)\right\|_{\text{op}} \le \sigma^2 \left\| \sum_{i=1}^n A_i A_i^T \right\|_{\text{op}} \le 2 \sigma^2 d n.
\end{align}
By applying the matrix Bernstein inequality \cite{Tropp_2011,wainwright_2019}, for any $t > 0$,
\begin{align}
    \prob{\left\|\frac{1}{n} \sum_{i=1}^n A_i \epsilon_i\right\|_2 \geq t} \leq d \cdot \exp \left(\frac{-3 t^2 n^2}{12 d n \sigma^2+2 \sigma \sqrt{2n} n t}\right)=d \cdot \exp \left(\frac{-3 t^2 n}{12 d \sigma^2 + 4 \sigma \sqrt{n} t}\right). \nonumber
\end{align}
Choosing $t = 3\sigma \sqrt{\frac{d\log ( d/\eta)}{n}}$, the RHS is upper bounded by $\eta$.
This completes the proof.
\end{proof}

\noindent By virtue of this result, we expect that for all $U$, with probability $\ge 1 - \eta$, the loss $\cL$ in \cref{eq:lossexpanded} is $\approx \| UU^T - U_\star U_\star^T \|_F^2$ (up to additive constants) when $n$ is large. Indeed, by an application of Holder's inequality, $\tr (AB) \le \| A \|_{\text{op}} \| B \|_* \le \| A \|_{\text{op}} \sqrt{\textsf{rank} (B)} \| B \|_F$ and,
\begin{align} \label{eq:functionerror}
    \left| \cL_{\text{emp}} (U) - \| UU^T - U_\star U_\star^T \|_{\mathcal{H}}^2 - \frac{1}{n} \sum_{i=1}^n \varepsilon_i^2 \right| \lesssim \sigma  \sqrt{\frac{kd \log(d/\eta)}{n}} \| UU^T - U_\star U_\star^T \|_F.
\end{align}
Likewise,
\begin{align} \label{eq:graderror}
    \left| \langle \nabla \cL_{\text{emp}} (U), Z \rangle - \langle \nabla \| UU^T - U_\star U_\star^T \|_{\mathcal{H}}^2, Z \rangle \right| \lesssim \sigma  \sqrt{\frac{kd \log(d/\eta)}{n}} \| UZ^T + ZU^T \|_F
\end{align}
And finally,
\begin{align} \label{eq:hessianerror}
    \left| \vect(Z)^T [\nabla^2 \cL_{\text{emp}} (U)] \vect(Z) - \vect(Z)^T [\nabla^2 \| UU^T - U_\star U_\star^T \|_{\mathcal{H}}^2] \vect(Z) \right| \lesssim \sigma  \sqrt{\frac{kd \log(d/\eta)}{n}} \| ZZ^T \|_F
\end{align}
where the hidden constants in each of these inequalities are at most $16$.

In the sequel, we condition on the event that \cref{eq:functionerror,eq:graderror,eq:hessianerror} hold, which occurs with probability $\ge 1- \eta$.

\paragraph{\Cref{lemma:grad-lb-matrix-sensing,lemma:apx-orthogonal}:} The conclusion of these lemmas can still be applied since the finite sample loss function \cref{eq:empirical_loss} is still of the form $f ( UU^T)$ for a doubly differentiable $f$.

\paragraph{\Cref{lemma:negativeeig}:} This lemma is slightly modified in the finite sample setting. The new result is provided below.

\begin{lemma}[Modified \Cref{lemma:negativeeig}] \label[lemma]{lemma:negativeeig-finite}
At an $(\epsilon,\gamma)$-approximate second order stationary point of $\cL + \lambda \cR_\beta$,
\begin{align}
    \| UU^T - U_\star U_\star^T \|_F \lesssim \max \left\{ \epsilon^{2/3} k^{1/6}, \lambda \sqrt{\frac{k}{\beta}} + \sigma  \sqrt{\frac{kd \log(d/\eta)}{n}}, \sqrt{k} \gamma \right\}. \label{eq:182-finite}
\end{align}
\end{lemma}
\begin{proof}
From \cite[Lemma 7]{rong} and \cref{eq:hessianerror} as in \Cref{lemma:negativeeig}, for the choice $Z = U - U_\star R_\star$, with $R_\star \in \arg\min_{R : RR^T = R^T R = I} \| U - U_\star R \|_F^2$, we have the following bound,
\begin{align}
\vect(Z)^T [\nabla^2 \cL_{\text{emp}} (U) ] \vect(Z) &\le 2\| ZZ^T \|_{\mathcal{H}}^2 - 6\| UU^T - U_\star U_\star^T \|_{\mathcal{H}}^2 + 4 \langle \nabla \cL_{\text{emp}} (U), Z \rangle \nonumber\\
&\qquad +\lambda \left[\vect(Z)^T [\nabla^2 \cR_\beta (U)] \vect (Z) - 4 \lambda \langle \nabla \cR_\beta (U), Z \rangle \right] \nonumber \\
&\qquad + 16 \sigma  \sqrt{\frac{kd \log(kd/\eta)}{n}} \| Z Z^T \|_F \label{eq:presimple-finite}
\end{align}
Plugging \cref{eq:L40} into \cref{eq:presimple-finite} and using \Cref{lemma:26} and the fact that $\| \nabla \cL_{\text{emp}} (U) \|_F \le \epsilon$,
\begin{align}
    \vect(Z)^T [\nabla^2 \cL_{\text{emp}} (U)] \vect(Z) &\le 2 \| ZZ^T \|_{\mathcal{H}}^2 - 6 \| UU^T - U_\star U_\star^T \|_{\mathcal{H}}^2 + 4 \epsilon \| Z \|_F \nonumber\\
    &\quad + \left( 8\lambda \sqrt{\frac{2k}{\beta}} + 32 \sigma  \sqrt{\frac{kd \log(d/\eta)}{n}} \right) \| UU^T - U_\star U_\star^T \|_F. \label{eq:21112-finite}
\end{align}

By the RIP condition on the measurements, $\| ZZ^T \|_{\mathcal{H}}^2 \le (1 + \delta) \| ZZ^T \|_F^2$ and likewise, $\| UU^T - U_\star U_\star^T \|_{\mathcal{H}}^2 \ge (1 - \delta) \| UU^T - U_\star U_\star^T \|_F^2$. Therefore, assuming $\delta \le 1/10$ and simplifying as done in \cref{eq:1821,eq:1822}, we get,
\begin{align}
    \vect(Z)^T [\nabla^2 \cL_{\text{emp}} (U)] \vect(Z) &\le - \| UU^T - U_\star U_\star^T \|_F^2 + 8 \epsilon k^{1/4} \| UU^T - U_\star U_\star^T \|_F^{1/2} \nonumber\\
    &\quad + \left( 8\lambda \sqrt{\frac{2k}{\beta}} + 32 \sigma  \sqrt{\frac{kd \log(d/\eta)}{n}} \right) \| UU^T - U_\star U_\star^T \|_F. \nonumber
\end{align}
The rest of the analysis resembles the calculations from \cref{eq:1822} to \cref{eq:1823}.
\end{proof}

\paragraph{\Cref{lemma:ZFbound,lemma:proj-similar,lemma:proj-0}:} These lemmas are slightly changed in the finite sample setting. The new results essentially replace $\epsilon$ by a slightly larger value $\nu$ defined below,
\begin{align} \label{eq:nu-def}
    \nu = \epsilon + \delta \| UU^T - U_\star U_\star^T \|_F + \sigma  \sqrt{\frac{k d \log(d/\eta)}{n}}.
\end{align}

\begin{lemma} \label[lemma]{lemma:proj-similar-finite}
Consider an $\epsilon$-approximate first order stationary point $U$ satisfying $\| U \|_{\text{op}} \le 3$. Let $V_r$ denote the top-$r$ eigenspace of $U_\star U_\star^T$. Then, we have that,
\begin{align}
    \| V_r^\perp (V_r^\perp)^T U \|_F^3 \lesssim \nu k.
\end{align}
Let $S = \{ i \in [k] : \| U e_i \|_2 \ge 2 \sqrt{\beta} \}$ be the set of large norm columns of $U$. For any column $i \in S$,
\begin{align}
    \frac{\| V_r^\perp (V_r^\perp)^T U e_i \|_2}{\| U e_i \|_2} \le \frac{2 \nu}{\lambda \beta^{1/4}}. \label{eq:projangle90-finite}
\end{align}
In contrast, for the remaining columns,
\begin{align}
    \sum_{i \in [k] \setminus S} \| V_r^\perp (V_r^\perp)^T U e_i \|_2^2 \le \frac{2\nu^2 \sqrt{\beta}}{\lambda^2}.
\end{align}
Lastly, for any column $i \in [k]$ such that $\| V_r V_r^T U e_i \|_2 \le \| V_r^\perp (V_r^\perp)^T U e_i \|_2$, assuming $ \nu \le \lambda/2$,
\begin{align} \label{eq:proj-case-finite}
    \| U e_i \|_2 \le 2 \sqrt{\beta}.
\end{align}
\end{lemma}
\begin{proof}
From \cref{eq:graderror} and the gradient computations in \Cref{lemma:L-grad-hessian-finite} and \Cref{lemma:R-grad-hessian}, with $Z = V_r^\perp (V_r^\perp)^T U$, approximate first-order stationarity of $U$ implies that,
\begin{align}
    \langle UU^T - U_\star U_\star^T, U Z^T + ZU^T \rangle_{\mathcal{H}} + \lambda \tr ( D(U) Z^T U) - \sigma  \sqrt{\frac{kd \log(d/\eta)}{n}} \| UZ^T + ZU^T \|_F \le \epsilon \| Z \|_F
\end{align}
In the infinite sample analyses (\Cref{lemma:ZFbound,lemma:proj-similar}) the first term on the LHS is non-negative for this choice of $Z$. In the finite sample case, we instead lower bound as follows. Noting that $Z = V_r^\perp (V_r^\perp)^T U$ is rank $k-r$, this means that $UZ^T + ZU^T$ is a symmetric matrix of rank $\le 2 (k-r) \le 2k$ by the subadditivity of the rank of matrices. Likewise, $UU^T - U_\star U_\star^T$ is of rank $\le k + r \le 2 k$. Therefore, by \Cref{assump:RIP},
\begin{align}
    &\langle UU^T - U_\star U_\star^T, U Z^T + ZU^T \rangle_{\mathcal{H}} \\
    &\ge \langle UU^T - U_\star U_\star^T, U Z^T + ZU^T \rangle - \delta \| UU^T - U_\star U_\star^T \|_F \| UZ^T + ZU^T \|_F \\
    &\ge \langle UU^T - U_\star U_\star^T, U Z^T + ZU^T \rangle - 6 \delta \| Z \|_F \| UU^T - U_\star U_\star^T \|_F
\end{align}
where the last inequality uses the assumption that $\| U \|_{\text{op}} \le 3$ which results in the bound,
\begin{align}
    \| ZU^T + UZ^T \|_F \le 2 \| U \|_{\text{op}} \| Z \|_F \le 6 \| Z \|_F.
\end{align}
Therefore, in the finite sample case, instead of \cref{eq:plass2} (and likewise, \cref{eq:plass1}) we have,
\begin{align}
    &\langle UU^T - U_\star U_\star^T, U Z^T + ZU^T \rangle + \lambda \tr (D(U) Z^T U) \\
    &\lesssim \left( \epsilon + \delta \| UU^T - U_\star U_\star^T \|_F 
 + \sigma  \sqrt{\frac{kd \log(d/\eta)}{n}} \right) \| Z \|_F \\
    &= \nu \| Z \|_F.
\end{align}
The proof of \Cref{lemma:proj-similar-finite} follows directly by replacing $\epsilon$ by $\nu$ everywhere in the remainder of the proofs of \Cref{lemma:ZFbound,lemma:proj-similar,lemma:proj-0}.
\end{proof}

Finally, we extend \Cref{theorem:prune-prebound} to the finite sample setting and in combination with the choice of parameters present the main result in the finite sample setting, a restatement of \Cref{theorem:main_finite}.

\begin{theorem}[Main result in the finite sample setting] \label[theorem]{theorem:main-finite}
Suppose the parameters $\epsilon,\gamma,\lambda$ and $\beta$ are chosen as in the population setting (\Cref{theorem:main-population}) and consider the solution $U_{\text{prune}}$ returned by \Cref{alg:main}. Suppose,
\begin{align}
    n \ge c_n \frac{\sigma^2 kd \log(d/\eta)}{(\epsilon^\star)^2} = \Theta \left( \frac{\sigma^2 k^2 r^5 d \log(d/\eta)}{ (\sigma_r^\star)^4} \right) ; \qquad \delta \le c_\delta \frac{(\sigma_r^\star)^{3/2}}{\sqrt{k} r^{5/2}}.
\end{align}
for appropriate absolute constants $c_n, c_\delta > 0$. Then, with probability $\ge 1 - \eta$,
\begin{enumerate}
    \item $U_{\text{prune}}$ has exactly $r$ non-zero columns.
    \item Furthermore,
\begin{align}
    \| U_{\text{prune}} U_{\text{prune}}^T - U_\star U_\star^T \|_F \le \frac{1}{2} (\sigma_r^\star)^2.
    \end{align}
\end{enumerate}
In other words, \Cref{alg:main} results in a solution $U_{\text{prune}}$ having exactly $r$ non-zero columns, and also serving as a spectral initialization.
\end{theorem}
\begin{proof}
Up until \cref{eq:thetaij}, the proof is unchanged. Plugging in the new upper bound on the cosine of the angle between $V_r^\perp (V_r^\perp)^T U e_i$ and $U e_i$ in \cref{eq:projangle90-finite} in \Cref{lemma:proj-similar-finite}, we get,
\begin{align} \label{eq:247}
    \left| \theta_{ij}^{\text{proj}} - \frac{\pi}{2} \right| \lesssim \sqrt{\frac{\epsilon+\gamma}{\lambda}} + \frac{\nu}{\lambda \beta^{1/4}}.
\end{align}
Assume for a sufficiently small constant $c_2 > 0$, $\epsilon + \gamma \le c_2 \lambda/r^2$. Then the first term on the RHS is upper bounded by $\frac{c_1}{2r}$ for a sufficiently small $c_1 \le 1/3$. Recall the definition,
\begin{align} \label{eq:nu-def-repeat}
    \nu = \epsilon + 2 \delta \| UU^T - U_\star U_\star^T \|_F + \sigma \sqrt{\frac{kd \log (d/\eta)}{n}}.
\end{align}
And furthermore, by \Cref{lemma:negativeeig-finite},
\begin{align} \label{eq:UUT-ub}
    \| UU^T - U_\star U_\star^T \|_F \lesssim \max \left\{ \epsilon^{2/3} k^{1/6}, \lambda \sqrt{\frac{k}{\beta}} + \sigma  \sqrt{\frac{kd \log(d/\eta)}{n}}, \sqrt{k} \gamma \right\}
\end{align}
Observe that if everywhere $\nu$ was replaced by $\epsilon$ (or even a constant factor approximation to it), the proof of \Cref{theorem:prune-prebound} essentially carries over unchanged. Let $\epsilon^\star = c_\epsilon \frac{(\sigma_r^\star)^{7/2}}{\sqrt{k} r^{5/2}}$ be the choice of $\epsilon$ in the population setting. This is the ``target'' value of $\epsilon$ in the empirical setting and we show that as long as $\delta$ is sufficiently small ($O(1/\sqrt{k})$) and $n$ is sufficiently large ($\widetilde{\Omega} (d k^2)$) $\nu$ is at most $3 \epsilon^\star$.

In particular, with the same choice of parameters ($\epsilon,\gamma,\beta$ and $\lambda$) as in the population setting, suppose that,
\begin{align}
    n \ge c_n \frac{\sigma^2 kd \log(d/\eta)}{(\epsilon^\star)^2} = \Theta \left( \frac{\sigma^2 k^2 r^5 d \log(d/\eta)}{ (\sigma_r^\star)^4} \right) ; \qquad \delta \le c_\delta \frac{(\sigma_r^\star)^{3/2}}{\sqrt{k} r^{5/2}}. \label{eq:n-bound}
\end{align}
Note that by choice of the parameters $\beta,\lambda,\gamma,\epsilon$, observe that,
\begin{align}
    2 \delta \| UU^T - Y_\star U_\star^T \|_F \lesssim 2 \delta \max \left\{ \epsilon^{2/3} k^{1/6} ,
    \lambda \sqrt{\frac{k}{\beta}} , \sqrt{k} \gamma \right\} &\le \epsilon^\star. \label{eq:delt-bound}
\end{align}
Therefore, combining \cref{eq:delt-bound,eq:n-bound} with the definition of $\nu$ in \cref{eq:nu-def-repeat}, we have that $\nu \in [\epsilon^\star, 3 \epsilon^\star]$. With this choice, 
since $\epsilon = \epsilon^\star$ and $\nu$ are within constant multiples of each other, the rest of the proof of \Cref{theorem:prune-prebound} in the population setting carries over. Subsequently plugging in the choice of $\lambda$ ,$\beta$, $\epsilon$ and $\gamma$ results in the following two statements:
\begin{enumerate}
    \item $U_{\text{prune}}$ has at most $r$ non-zero columns.
    \item $\| U_{\text{prune}} U_{\text{prune}} - U_\star U_\star^T \|_F \le \frac{(\sigma_r^\star)^2}{2}$.
\end{enumerate}
However under these two results, $U_{\text{prune}}$ must have exactly $r$ non-zero columns; if it had strictly fewer than $r$ non-zero columns, then it is impossible to satisfy $\| U_{\text{prune}} U_{\text{prune}} - U_\star U_\star^T \|_F < (\sigma_r^\star)^2$. This completes the proof of the theorem.
\end{proof}

\section{Efficiently finding approximate second order stationary points} 
\label{app:oracle}

In this section we discuss efficiently finding second order stationary points of the loss $\cL + \lambda \cR_\beta$. We establish smoothness conditions on $\cL$ and $\cR_\beta$ and show that running gradient descent with sufficiently bounded perturbations satisfies the property that all iterates are bounded in that $\| U_t \|_2 \le 3$, as long as the algorithm is initialized within this ball.

\paragraph{Perturbed gradient descent:} We consider a first order method having the following update rule: for all $t \ge 0$,
\begin{align} \label{eq:gd_repeat}
    U_{t+1} \gets U_t - \alpha ( \nabla (\cL + \lambda \cR_\beta) (U_t) + P_t)
\end{align}
where $P_t$ is a perturbation term, which for example could be the stochastic noise present in SGD. Over the course of running the update rule \cref{eq:gd}, we show that the iterates $\| U_t \|_2$ remains bounded under mild conditions.

\subsection{Proof of \Cref{theorem:lipgradhess}}
In this section we prove \Cref{theorem:lipgradhess}. Below we first show gradient Lipschitzness on the domain $\{ U : \| U \|_{\text{op}} \le 3 \}$. Observe that,
\begin{align}
    &\| \nabla \cL_{\text{pop}} (U) - \nabla \cL_{\text{pop}} (V) \|_F \\
    &= \| (UU^T - U_\star U_\star^T) U - (VV^T - U_\star U_\star^T) V \|_F \\
    &= \| UU^T (U - V) + (U(U - V)^T + (U-V)V^T) V - U_\star U_\star^T (U - V) \|_F \\
    &\overset{(i)}{\le} \| U \|_{\text{op}}^2 \| U - V \|_F + \| U \|_{\text{op}} \| U - V \|_F \| V \|_{\text{op}} + \| U-V \|_F \| V \|_{\text{op}}^2 + \| U_\star U_\star^T \|_{\text{op}} \| U - V \|_F \\
    &\lesssim \| U - V \|_F, \label{eq:Lpopgradbound}
\end{align}
where $(i)$ repeatedly uses the bound $\| AB \|_F \le \| A \|_{\text{op}} \| B \|_F$ and the last inequality follows from the fact that $\| U \|_{\text{op}} \le 3$ and $\| U_\star \|_{\text{op}} = 1$. On the other hand, by the gradient computations in \Cref{lemma:R-grad-hessian},
\begin{align}
    \| \nabla^2 \cR_\beta (U) \|_{\text{op}} &\le \sup_{Z : \| Z \|_F \le 1} \vect(Z) [\nabla^2 \cR_\beta (U) ] \vect(Z) \\
    &\le \sup_{Z : \| Z \|_F \le 1} \langle D(U), Z^T Z \rangle \\
    &= \sup_{Z : \sum_{i=1}^k \| Z e_i \|_2^2 \le 1} \sum_{i=1}^k D(U)_{ii} \| Z e_i \|_2^2 \\
    &= \max_{i \in [k]} D(U)_{ii} \le \frac{2}{\sqrt{\beta}}. \label{eq:Rb-gradbound}
\end{align}
where the last inequality uses the fact that $\frac{x + 2\beta}{(x + \beta)^{3/2}} \le \frac{2}{\sqrt{\beta}}$. This implies that,
\begin{align}
    \| \nabla \cR_\beta (U) - \nabla \cR_\beta (V) \|_{\text{op}} \le \frac{2}{\sqrt{\beta}}\| U - V \|_F.
\end{align}
Combining \cref{eq:Lpopgradbound,eq:Rb-gradbound}, by triangle inequality,
\begin{align}
    \| \nabla f_{\text{pop}} (U) - \nabla f_{\text{pop}} (V) \|_F \lesssim \| U - V \|_F
\end{align}
where the assumption $\lambda/\sqrt{\beta} \le 1$ is invoked.

Next we prove Hessian Lipschitzness of $f_{\text{pop}}$. By \cite[Theorem 3]{hessianbound}, the Hessian of $\cL_{\text{pop}}$ satisfies the condition,
\begin{align}
    \| \nabla^2 \cL_{\text{pop}} (U) - \nabla^2 \cL_{\text{pop}} (V) \|_{\text{op}} \lesssim \| U - V \|_F. \label{eq:Lpophess}
\end{align}
Although the analysis in \cite{hessianbound} is provided for the exactly specified case ($k=r$), it carries over unchanged to the overparameterized case as well. On the other hand, for any $Z, Z_U \in \mathbb{R}^{d \times k}$ such that $\| Z \|_F, \| Z_U \|_F \le 1$, the rate of change of the Hessian at $U$ evaluated along the direction $Z$,
\begin{align}
    &\lim_{t \to 0} \vect(Z)^T \left[ \frac{\nabla^2 \cR_\beta (U) - \nabla^2 \cR_\beta (U + tZ_U)}{t} \right] \vect(Z) \\
    &= \lim_{t \to 0}\frac{\langle D(U) - D(U + tZ_U), Z^T Z \rangle}{t} \\
    &\qquad\qquad + \frac{\sum_{i=1}^k G(U + tZ_U)_{ii} \langle Ze_i, (U + tZ_U)e_i \rangle^2 - G(U)_{ii} \langle Ze_i, Ue_i \rangle^2}{t}. \label{eq:limt}
\end{align}
The first term of the RHS of \cref{eq:limt} can be bounded as,
\begin{align}
    \lim_{t \to 0}\frac{\langle D(U) - D(U + tZ_U), Z^T Z \rangle}{t} &= - \sum_{i=1}^k \frac{\| U e_i \|_2^2 + 4 \beta}{( \| U e_i \|_2^2 + \beta)^{5/2}} \langle U e_i, Z_U e_i \rangle \cdot \| Z e_i \|_2^2 \\
    &\overset{(i)}{\le} \max_{i \in [k]} \left| \frac{\| U e_i \|_2^2 + 4 \beta}{( \| U e_i \|_2^2 + \beta)^{5/2}} \langle U e_i, Z_U e_i \rangle \right| \sum_{i=1}^k \| Z e_i \|_2^2 \\
    &\overset{(ii)}{\le} \max_{i \in [k]} \frac{\| U e_i \|_2^2 + 4 \beta}{( \| U e_i \|_2^2 + \beta)^{5/2}} \| U e_i \|_2 \\
    &\lesssim \frac{1}{\beta}, \label{eq:last1}
\end{align}
where $(i)$ follows by Holder inequality, and $(ii)$ uses the fact that $\| Z \|_F, \| Z_U \|_F \le 1$. The last inequality uses the fact that $\min_{x \ge 0} \frac{\sqrt{x} (x+4\beta)}{(x+\beta)^{5/2}} \lesssim \frac{1}{\beta}$.

On the other hand, by the chain rule, the second term on the RHS of \cref{eq:limt} can be computed in two parts, the first being,
\begin{align}
    \lim_{t \to 0} \frac{\left( G(U+tZ_U)_{ii} - G(U)_{ii} \right) \langle Z e_i, U e_i \rangle^2}{t} &= - \sum_{i=1}^k \frac{3(\| U e_i \|_2^2 + 6 \beta)}{(\| U e_i \|_2^2 + \beta)^{7/2}} \langle Ue_i, Z_U e_i \rangle \langle Z e_i, U e_i \rangle^2 \nonumber\\
    &\le \sum_{i=1}^k H(U)_{ii} \|  U e_i \|_2 \langle Z e_i, U e_i \rangle^2 \label{eq:004294729}
\end{align}
which uses the fact that $\| Z_U \|_F \le 1$ $H(U)_{ii} = \frac{3(\| U e_i \|_2^2 + 6 \beta) \| Ue_i \|_2}{(\| U e_i \|_2^2 + \beta)^{7/2}}$. Note that for any $c \ge 0$, the optimization problem which takes the form $\max_{v : \| v \|_2^2 \le c} \langle v, U e_i \rangle$ is maximized when $v \propto U e_i$. Therefore, fixing the solution to the remaining coordinates and optimizing over $Z$ for a single $i$, one may substitute $Z e_i = x_i \cdot U e_i / \| U e_i \|_2$ for some $x_i \in \mathbb{R}$ and bound \cref{eq:004294729} as,
\begin{align}
    \lim_{t \to 0} \frac{\left( G(U+tZ_U)_{ii} - G(U)_{ii} \right) \langle Z e_i, U e_i \rangle^2}{t}
    &\le \max_{\sum_{i=1}^k x_i^2 \le 1} \sum_{i=1}^k H(U)_{ii} x_i^2 \| U e_i \|_2^3\\
    &= \max_{i \in [k]} H(U)_{ii} \| U e_i \|_2^3\\
    &\lesssim \frac{1}{\beta}, \label{eq:last2}
\end{align}
where the last inequality uses the fact that $\min_{x \ge 0} \frac{3\sqrt{x^3}(x+6\beta)}{(x + \beta)^{7/2}} \lesssim \frac{1}{\beta}$. The second part of the second term on the RHS of \cref{eq:limt} is,
\begin{align}
    &\lim_{t \to 0} \sum_{i=1}^k \frac{G(U)_{ii} \left( \langle Z e_i, (U + tZ_U)e_i \rangle^2 - \langle Z e_i, U e_i \rangle^2 \right)}{t} \\
    &= 2 \sum_{i=1}^k G(U)_{ii} \langle Ze_i, U e_i \rangle \langle Z e_i, Z_U e_i \rangle \\
    &\lesssim \sum_{i=1}^k G(U)_{ii} \| U e_i \|_F \langle Z e_i, Z_U e_i \rangle \\
    &\lesssim \sum_{i=1}^k G(U)_{ii} \| U e_i \|_F \| Z e_i \|_2 \| Z_U e_i \|_2 \\
    &\le \sqrt{\sum_{i=1}^k (G(U)_{ii})^2 \| U e_i \|_F^2 \| Z e_i \|_2^2}  \\
    &\le \max_{i \in [k]} G(U)_{ii} \| U e_i \|_F \\
    &\lesssim \frac{1}{\beta} \label{eq:last3}
\end{align}
where in $(i)$ we use the fact that $x_i = \langle Z e_i, Z_U e_i \rangle$ satisfies $|\sum_i x_i| = |\langle Z, Z_U \rangle| \le \| Z \|_F \| Z_U \|_F \le 1$. The last inequality uses the fact that $\max_{x \ge 0} \frac{\sqrt{x} (x + 4 \beta)}{(x + \beta)^{5/2}} \lesssim 1/\beta$. Combining \cref{eq:last1,eq:last2,eq:last3} in \cref{eq:limt} results in the bound,
\begin{align}
    \max_{Z, Z_U : \| Z \|_F, \| Z_U \|_F \le 1} \lim_{t \to 0} \frac{\vect(Z)^T [\nabla^2 \cR_\beta (U) - \cR_\beta (U + t Z_U)] \vect(Z)}{t} \lesssim \frac{1}{\beta}
\end{align}
By choosing $Z_U = (U-V)/\|U-V\|_F$, this implies that,
\begin{align}
    \max_{Z : \| Z \|_F \le 1} \vect(Z)^T [\nabla^2 \cR_\beta (U) - \nabla^2 \cR_\beta (V)] \vect(Z) \lesssim \frac{1}{\beta} \| U - V \|_F.
\end{align}
This implies that,
\begin{align}
    \| \nabla^2 \cR_\beta (U) - \nabla^2 \cR_\beta (V) \|_{\text{op}} \lesssim \frac{1}{\beta} \| U - V \|_F.
\end{align}
Combining with \cref{eq:Lpophess} by triangle inequality and noting the assumption that $\lambda \le \beta$ results in the equation,
\begin{align}
    \| \nabla^2 f_{\text{pop}} (U) - f_{\text{pop}} (V) \|_{\text{op}} \lesssim \| U - V \|_F.
\end{align}
This completes the proof.

\subsection{Proof of \Cref{theorem:bounded}} 

In this section we prove \Cref{theorem:bounded} formally. By the gradient computations in \Cref{lemma:L-grad-hessian,lemma:R-grad-hessian} and triangle inequality,
\begin{align}
    \| U_{t+1} \|_{\text{op}} &\le \| (I - \alpha U_tU_t^T) U_t \|_{\text{op}} + \alpha \| U_\star U_\star^T U_t \|_{\text{op}} + \alpha \lambda \| D(U_t) \|_{\text{op}} \| U_t \|_{\text{op}} + \alpha \\
    &\le \| (I - \alpha U_tU_t^T) U_t \|_{\text{op}} + \alpha \| U_\star U_\star^T U_t \|_{\text{op}} + \alpha \frac{\lambda}{\sqrt{\beta}} \| U_t \|_{\text{op}} + \alpha
\end{align}
where the last inequality follows from the fact that $D(U_t)$ is a diagonal matrix with largest entry upper bounded by $2/\sqrt{\beta}$. Note that $\| (I - \alpha U_tU_t^T) U_t \|_{\text{op}} = \max_i \sigma_i ( 1 - \alpha \sigma_i^2)$ where $\{ \sigma_i \}_{i=1}^k$ denotes the singular values of $U_t$. For $x \in [0,1/2\alpha]$, $x ( 1 - \alpha x^2)$ is an increasing function. Therefore, assuming $\| U_t \|_{\text{op}} \in [0,1/2\alpha]$,
\begin{align}
    \| (I - \alpha U_tU_t^T) U_t \|_{\text{op}} \le \| U_t \|_{\text{op}} ( 1 - \alpha \| U_t \|_{\text{op}}^2 ).
\end{align}
Combining everything together,
\begin{align}
    \| U_{t+1} \|_{\text{op}} &\le \| U_t \|_{\text{op}} (1 - \alpha \| U_t \|_{\text{op}}^2) + \alpha \| U_t \|_{\text{op}} + \alpha \frac{\lambda}{\sqrt{\beta}} \| U_t \|_{\text{op}} + \alpha \\
    &\le \| U_t \|_{\text{op}} (1 - \alpha \| U_t \|_{\text{op}}^2) + 2\alpha \| U_t \|_{\text{op}} + \alpha
\end{align}
where the last inequality assumes $\lambda/\sqrt{\beta} \le 1$. Depending on whether $\| U_t \|_{\text{op}} \ge 2$ or $\| U_t \|_{\text{op}} \le 2$, and recalling that $\alpha \le 1/8$, we may derive two bounds from the above,
\begin{align}
    &0 \le \|U_t \|_{\text{op}} \le 2 \implies \| U_{t+1} \|_{\text{op}} \le \| U_t \|_{\text{op}} (1 + 2\alpha) + \alpha \label{eq:init_iterate}\\
    &2 \le \|U_t \|_{\text{op}} \le 4 \implies \| U_{t+1} \|_{\text{op}} \le \| U_t \|_{\text{op}} (1 - \alpha) \label{eq:final_iterate}
\end{align}
Suppose the initial iterate satisfied $\| U_0 \|_{\text{op}} \le 2$, and let $t_0$ be any iteration where $\| U_{t_0} \|_{\text{op}} \le 2$, but $\| U_{t_0+1} \|_{\text{op}} \ge 2$. Then by \cref{eq:init_iterate}, $\| U_{t_0+1} \|_{\text{op}} \le 3$. However, in the subsequent iterations, $\| U_t \|_{\text{op}}$ must decrease until it is no larger than $2$, by virtue of \cref{eq:final_iterate}. This implies the statement of the theorem.

\section{Gradient and Hessian computations}

\subsection{Population mean square error}

\begin{lemma} \label[lemma]{lemma:L-grad-hessian}
Consider the function $\cL_{\text{pop}} (U) = \| U U^T - U_\star U_\star^T \|_F^2$. Then,
\begin{enumerate}
    \item The gradient of $\cL$ satisfies $\langle \nabla \cL_{\text{pop}} (U), Z \rangle = 2 \langle (UU^T - U_\star U_\star^T), U Z^T + ZU^T \rangle$.
    \item For any $Z \in \mathbb{R}^{d \times k}$,
    \begin{align}
        \vect (Z)^T [ \nabla^2 \cL_{\text{pop}} (U)] \vect (Z) = 4 \langle Z, (UU^T - U_\star U_\star^T ) Z \rangle + 2 \| U Z^T + Z U^T \|_F^2.
    \end{align}
\end{enumerate}
\end{lemma}
\begin{proof}
The first part is proved in \cite[eq. (59)]{cong-wang-chi-chen-17}. The second part is proved shortly after in \cite[eq. (61)]{cong-wang-chi-chen-17}.
\end{proof}

\begin{lemma} \label[lemma]{lemma:genL-grad-hessian}
Consider the function $\cL (U) = f(UU^T) : \mathbb{R}^{d \times k} \to \mathbb{R}$ for some doubly differentiable function $f : \mathbb{R}^{d \times d} \to \mathbb{R}$. Then,
\begin{enumerate}
    \item For any $Z \in \mathbb{R}^{d \times k}$, $\langle \nabla \cL (U), Z \rangle = \langle (\nabla f) (UU^T), U Z^T + ZU^T \rangle$.
    \item For any $Z \in \mathbb{R}^{d \times d}$,
    \begin{align}
        &\vect (Z)^T [ \nabla^2 \cL (U)] \vect (Z) \nonumber\\
        &= \vect (UZ^T + ZU^T)^T [(\nabla^2 f) (UU^T)] \vect (UZ^T + ZU^T) + 2\langle (\nabla f) (UU^T), ZZ^T \rangle
    \end{align}
\end{enumerate}
\end{lemma}
\begin{proof}
This result is straightforward to prove by direct computation.
\end{proof}

\subsection{Empirical mean square error}

\begin{lemma} \label[lemma]{lemma:L-grad-hessian-finite}
Consider the loss $f (U) = \| UU^T - U_\star U_\star \|_{\mathcal{H}}^2$ where $\| \cdot \|_{\mathcal{H}}$ is defined in \cref{eq:lossexpanded}. Then,
\begin{enumerate}
    \item The gradient of $f$ satisfies, $\langle \nabla f (U), Z \rangle = 2 \langle UU^T - U_\star U_\star^T, U Z^T + Z U^T \rangle_{\mathcal{H}}$.
    \item For any $Z \in \mathbb{R}^{d \times d}$,
    \begin{align}
        \vect (Z)^T [ \nabla^2 f (U)] \vect (Z) = 4 \langle UU^T - U_\star U_\star^T, ZZ^T \rangle_{\mathcal{H}} + 2\| U Z^T + Z U^T \|_{\mathcal{H}}^2.
    \end{align}
\end{enumerate}
\end{lemma}
\begin{proof}
These results are proved in \cite[eq. (18)]{rong}.
\end{proof}

\subsection{Regularization $\cR_\beta$} \label{sec:regularizer-auxiliary-calc}

\begin{lemma} \label[lemma]{lemma:R-grad-hessian}
Consider the function $\cR_\beta : \mathbb{R}^{d \times m} \to \mathbb{R}$ defined as $\cR_\beta (U) = \sum_{i=1}^m L_2^\beta (U e_i)$. Define,
\begin{align}
    D (U) &= \textsf{diag} \left( \left\{ \frac{(\| U e_i \|_2^2 + 2 \beta)}{(\| U e_i\|_2^2 + \beta)^{3/2}} : i \in [m] \right\} \right), \text{ and}, \label{eq:DU-def}\\
    G (U) &= \textsf{diag} \left( \left\{ \frac{(\| U e_i \|_2^2 + 4 \beta)}{(\| U e_i\|_2^2 + \beta)^{5/2}} : i \in [m] \right\} \right). \label{eq:GU-def}
\end{align}
For any $Z \in \mathbb{R}^{d \times m}$,
\begin{enumerate}
    \item The gradient of $\cR_\beta$ is, $\nabla \cR_\beta (U) = U D(U)$.
    \item The Hessian of $\cR_\beta$ satisfies,
    \begin{align}
         \vect(Z)^T [\nabla^2 \cR_\beta (U)] \vect (Z) = \langle D(U), Z^T Z \rangle - \sum_{i=1}^m G(U)_{ii} \langle U e_i, Z e_i \rangle^2
    \end{align}
\end{enumerate}
\end{lemma}
\begin{proof}
Note that the gradient of $\Lsm (v)$ as a function of $v$ is, 
\begin{align}
    \nabla L_2^\beta (v) = \frac{2v}{\sqrt{\| v \|_2^2 + \beta}} - \frac{1}{2}\frac{\| v \|_2^2}{(\| v \|_2^2 + \beta)^{3/2}} \cdot (2v) = \frac{(\| v \|_2^2 + 2\beta) v}{(\| v \|_2^2 + \beta)^{3/2}}.
\end{align}
Therefore,
\begin{align}
     \langle Z, \nabla \cR_\beta (U) \rangle &= \sum_{i=1}^d \frac{(\| U e_i \|_2^2 + 2 \beta)}{(\| U e_i\|_2^2 + \beta)^{3/2}} \langle U e_i, Z e_i \rangle \\
     &= \sum_{i=1}^d \frac{(\| U e_i \|_2^2 + 2 \beta)}{(\| U e_i\|_2^2 + \beta)^{3/2}} \tr \left( Z e_i e_i^T U^T \right) \\
    &= \tr \left( D(U) \cdot U^T Z  \right),
\end{align}
where $D (U) = \textsf{diag} \left( \left\{ \frac{(\| U e_i \|_2^2 + 2 \beta)}{(\| U e_i\|_2^2 + \beta)^{3/2}} : i \in [m] \right\} \right)$ is as defined in \cref{eq:DU-def}.

On the other hand, the Hessian of $\Lsm (v)$ as a function of $v$ is,
\begin{align}
    \nabla^2 \Lsm (v) &= \frac{(\| v \|_2^2 + 2 \beta)}{(\| v\|_2^2 + \beta)^{3/2}} I - \frac{3 (\| v \|_2^2 + 2\beta)}{( \| v \|_2^2 + \beta)^{5/2}} vv^T + \frac{2}{(\| v \|_2^2 + \beta)^{3/2}} vv^T \\
    &= \frac{(\| v \|_2^2 + 2 \beta)}{(\| v\|_2^2 + \beta)^{3/2}} I - \frac{\| v \|_2^2 + 4 \beta}{( \| v \|_2^2 + \beta)^{5/2}} vv^T.
\end{align}
The Hessian of $\cR_\beta$ is block diagonal with the $i^{th}$ block equal to $\nabla^2 \Lsm (U e_i)$. Therefore,
\begin{align}
    \vect(Z)^T [\nabla^2 \cR_\beta (U)] \vect (Z) &= \sum_{i = 1}^m \frac{(\| U e_i \|_2^2 + 2 \beta)}{(\| U e_i \|_2^2 + \beta)^{3/2}} \| Z e_i \|_2^2 - \frac{\| Ue_i \|_2^2 + 4 \beta}{( \| U e_i \|_2^2 + \beta)^{5/2}} \langle U e_i, Z e_i \rangle^2 \\
    &= \sum_{i = 1}^m D(U)_{ii} \| Z e_i \|_2^2 - G(U)_{ii} \langle U e_i, Z e_i \rangle^2 \\
    &= \tr ( D(U) \cdot ZZ^T ) - \sum_{i = 1}^m G(U)_{ii} \langle U e_i, Z e_i \rangle^2
\end{align}
\end{proof}

Finally we introduce a lemma bounding the entries of $D(U)$ and $G(U)$.

\begin{lemma} \label{lemma:DUbound}
Suppose for some $i \in [k]$, $\| U e_i \|_2 \ge 2 \sqrt{\beta}$. Then the corresponding diagonal entries of $D(U)$ and $G(U)$ satisfy,
\begin{align}
    (D(U))_{ii} &\ge \frac{1}{\| U e_i \|_2} \label{eq:DUlb} \\
    (G(U))_{ii} &\ge \frac{1}{\| U e_i \|_2^3}. \label{eq:GUlb}
\end{align}
On the other hand, for any $U$ such that $\| U \|_{\text{op}} \le 3$,
\begin{align}
    (D(U))_{ii} \le \frac{\| U e_i \|_2^2 + 2 \beta}{(\| U e_i \|_2^2 + \beta)^{3/2}}
\end{align}
\end{lemma}
\begin{proof}
The proof for \cref{eq:DUlb} follows by observing that $D(U)_{ii} = \frac{\| U e_i \|_2^2 + 2 x}{(\| U e_i \|_2^2 + x)^{3/2}}$ where $x = \beta$. Differentiating we observe that the derivative of function in $x$ is,
\begin{align}
    \frac{2 (\| U e_i \|_2^2 + x) - 3/2 ( \| U e_i \|_2^2 + 2x)}{(\| U e_i \|_2^2 + x)^{5/2} } = \frac{ \frac{1}{2} \| U e_i \|_2^2 - x}{(\| U e_i \|_2^2 + x)^{5/2} }
\end{align}
which is increasing as long as $x \le \frac{1}{2} \| U e_i \|_2^2$. Note that when $\| U e_i \|_2 \ge 2 \sqrt{\beta} \implies x = \beta \le \frac{1}{4} \| U e_i \|_2^2$, the function is increasing in $x$ and therefore the minimum is achieved with $\beta = 0$. This results in the lower bound,
\begin{align}
    (D(U))_{ii} \ge \frac{1}{\| U e_i \|_2}.
\end{align}
Likeiwse, $G(U) = \frac{\| U e_i \|_2^2 + 4 x}{(\| U e_i \|_2^2 + x)^{5/2}}$ where $x = \beta$, and differentiating in $x$, we get,
\begin{align}
    \frac{4 (\| U e_i \|_2^2 + x) - 5/2 ( \| U e_i \|_2^2 + 2x)}{(\| U e_i \|_2^2 + x)^{5/2} } = \frac{ \frac{3}{2} \| U e_i \|_2^2 - x}{(\| U e_i \|_2^2 + x)^{5/2} }
\end{align}
which is increasing as long as $x \le \frac{3}{2} \| U e_i \|_2^2$. Yet again, the function is in the increasing regime in $x = \beta$ under the constraint $2 \sqrt{\beta} \le \| U e_i \|_2$; the minimum is achieved at $\beta = 0$, which results in \cref{eq:GUlb}.
\end{proof}

\section{Implications for shallow neural networks with quadratic activation functions}\label{sec:quad}

The extension of results from the matrix sensing model to the training of quadratic neural networks was previously carried out in \cite[Section 5]{li_17} and originally in \citep{nn-quadratic-ms}, which we explain in more detail below. Indeed, when the measurement matrices are of the form $A_i = x_i x_i^T$ for some vector $x_i \in \mathbb{R}^d$, the functional representation of the output can be written as $\langle A_i, UU^T \rangle = \sum_{i=1}^k \sigma(\langle x_i, U e_i\rangle)$, where $\sigma (\cdot) = (\cdot)^2$ takes the form of a $1$-hidden layer shallow network with quadratic activations, with the output layer frozen as all $1$'s. The columns of $U$ correspond to the weight vectors associated with individual neurons of the network. Likewise, sparsity in the column domain corresponds to learning networks with only a few non-zero neurons. In this section we provide a high level sketch of how results for matrix sensing can be extended for the training of neural networks with quadratic activations. We avoid going through the formal details for the sake of simplicity and ease of exposition.

Why can the results proved for matrix sensing not directly be applied here? It turns out that when the $x_i$ are i.i.d. Gaussian vectors, even if $n \to \infty$, rank $1$ measurements do not satisfy the restricted isometry property. In particular, $\frac{1}{n} \sum_{i=1}^n \langle A_i , X \rangle^2 \not\approx \| X \|_F^2$. However, these measurements satisfy a different form of low rank approximation, as established in \cite[Lemma 5.1]{li_17} and as we informally state below. In particular, when $x_i \sim \mathcal{N} (0,I)$ are Gaussian, for any matrix $X \in \mathbb{R}^{d \times d}$,
\begin{align}
    \frac{1}{n} \sum_{i=1}^n \langle A_i , X \rangle^2 \approx 2 \| X \|_F^2 + (\text{Tr} (X))^2
\end{align}
where the approximation becomes exact as $n \to \infty$.
In particular, with the choice of $X = UU^T - U_\star U_\star^T$, the mean squared error of the learner training a quadratic neural network takes the form,
\begin{align}\label{eq:LNN}
    \cL_{\text{NN}} (U) = \frac{1}{n} \sum_{i=1}^n \left( \langle A_i , U U^T \rangle - \langle A_i, U_\star U_\star^T \rangle \right)^2 &\approx 2 \| UU^T - U_\star U_\star^T \|_F^2 + \left( \| U \|_F^2 - \| U_\star \|_F^2 \right)^2 \nonumber
\end{align}
Notice that the RHS is, up to scaling factors, the mean squared error for matrix sensing, with an additional loss $(\| U \|_F^2 - \| U_\star \|_F^2)^2$ added to it. This loss can easily be estimated since it only relies on estimating a scalar, $\| U_\star \|_F^2$, which is approximated by $\frac{1}{n} \sum_{i = 1}^n y_i = \frac{1}{n} \sum_{i=1}^n \langle x_i x_i^T , U_\star U_\star^T \rangle + \varepsilon_i \approx \| U_\star \|_F^2$. In the sequel, we assume that $\| U_\star \|_F$ was known exactly and consider the loss,
\begin{align}
    f_{\text{NN}} (U) = \cL_{\text{NN}} (U) - \left( \| U \|_F^2 - \| U_\star \|_F^2 \right)^2 + \cR_\beta (U)
\end{align}
which subtracts the ``correction term'' $(\| U \|_F^2 - \| U_\star \|_F^2)^2$ from the mean squared error $\cL_{\text{NN}} (U)$ and also adds back the group Lasso regularizer on $U$. Overall $f_{\text{NN}} (U)$ approximately equals $2 \| UU^T - U_\star U_\star^T \|_F^2 + \cR_\beta (U)$ and therefore, running perturbed gradient descent on this loss and reusing the analysis for matrix sensing shows that the algorithm eventually converges to a solution $UU^T \approx U_\star U_\star^T$ and such that $U$ has approximately $r$ columns.

\end{document}